\documentclass[10pt,onecolumn,table]{article}
\usepackage{graphicx,amssymb,amsmath,amsthm,hyperref}
\usepackage{geometry}
\usepackage[authoryear,round]{natbib}
\usepackage{paralist}
\usepackage{pifont}
\usepackage{booktabs}
\usepackage{subcaption}
\usepackage[most]{tcolorbox}
\usepackage{enumitem,tabularx,array,colortbl,multirow}
\tcbuselibrary{skins,breakable,theorems}
\tcbuselibrary{hooks}
\usepackage{wrapfig2}
\usepackage{fancyhdr}

\definecolor{assumptioncolor}{RGB}{70,130,180}    
\definecolor{theoremcolor}{RGB}{46,125,50}        
\definecolor{corollarycolor}{RGB}{156,39,176}     
\definecolor{assumpbg}{RGB}{242,248,254}          
\definecolor{theombg}{RGB}{245,251,246}           
\definecolor{corollbg}{RGB}{251,245,254}          
\definecolor{darktext}{RGB}{30,30,30}             

\newtcbtheorem[auto counter]{styledassumption}{Assumption}{
    enhanced,
    breakable,
    colback=assumpbg,
    colframe=assumptioncolor!80!black,
    colbacktitle=assumpbg,
    coltitle=darktext,
    fonttitle=\bfseries\large,
    attach boxed title to top left={yshift=-2mm, xshift=3mm},
    boxed title style={
        sharp corners,
        boxrule=1.2pt,
        colframe=assumptioncolor!80!black,
        colback=assumpbg,
        left=2mm, right=2mm, top=1.5mm, bottom=1.5mm
    },
    top=4mm, bottom=4mm, left=4mm, right=4mm,
    arc=1.5mm,
    before upper=\color{black}
}{ass}

\newtcbtheorem[auto counter]{styledtheorem}{Theorem}{
    enhanced,
    breakable,
    colback=theombg,
    colframe=theoremcolor!80!black,
    colbacktitle=theombg,
    coltitle=darktext,
    fonttitle=\bfseries\large,
    attach boxed title to top left={yshift=-2mm, xshift=3mm},
    boxed title style={
        sharp corners,
        boxrule=1.2pt,
        colframe=theoremcolor!80!black,
        colback=theombg,
        left=2mm, right=2mm, top=1.5mm, bottom=1.5mm
    },
    top=4mm, bottom=4mm, left=4mm, right=4mm,
    arc=1.5mm,
    before upper=\color{black}
}{thm}

\newtcbtheorem[auto counter]{styledcorollary}{Corollary}{
    enhanced,
    breakable,
    colback=corollbg,
    colframe=corollarycolor!80!black,
    colbacktitle=corollbg,
    coltitle=darktext,
    fonttitle=\bfseries\large,
    attach boxed title to top left={yshift=-2mm, xshift=3mm},
    boxed title style={
        sharp corners,
        boxrule=1.2pt,
        colframe=corollarycolor!80!black,
        colback=corollbg,
        left=2mm, right=2mm, top=1.5mm, bottom=1.5mm
    },
    top=4mm, bottom=4mm, left=4mm, right=4mm,
    arc=1.5mm,
    before upper=\color{black}
}{cor}

\title{The Human-AI Substitution Principle:\\ When will you be replaced by AI in your organization?}

\author{Bonny Banerjee$^{1,}$\thanks{Corresponding author: bonnybanerjee@yahoo.com} ~~ Shreya Singh$^2$ ~~ Simanti Bandyopadhyay$^3$\\
\\
$^1$Independent Scholar, USA\\ 
$^2$Chennai Mathematical Institute, Chennai, Tamil Nadu, India\\
$^3$School of Management and Entrepreneurship, Shiv Nadar University, Delhi-NCR, India\\
}

\date{} 

\begin{document}

\maketitle

\begin{abstract}
Artificial Intelligence (AI) is rapidly transforming organizations, raising a fundamental organizational and economic question: when should a task be assigned to an AI agent rather than a human employee? We present an analytical model for studying Human--AI Task Allocation (HAT) in hierarchical organizations. A central feature of the HAT model is that it formally encodes the economic asymmetry between human skill acquisition and AI capability scaling. The HAT model allows us to derive how risk-adjusted costs, skills, organizational depth, deployment scale, strategic adaptation, and risk jointly determine when, where, why, and under what structural conditions human--AI replacement occurs. A key result is the Human--AI Substitution Principle, which gives an exact pairwise condition, in terms of risk-adjusted effective costs, under which transferring task allocation from a human agent to an AI agent strictly reduces risk-adjusted cost. The maintained human--AI cost asymmetry then determines how this condition varies with skill, capability, deployment scale, and hierarchical position. Building on this result, we show that, under stated conditions, AI adoption can produce abrupt shifts in the cost-minimizing task allocation, hybrid human--AI organizations (including, across heterogeneous tasks, cases where heterogeneity in risk-adjusted costs sustains both human and AI roles without a minimum-human-fraction constraint), and flatter managerial hierarchies with wider spans of control. The HAT model identifies structural conditions under which middle-management roles exhibit elevated vulnerability to automation, and shows that the vulnerability of highly skilled workers depends on a skill threshold shaped by organizational depth, baseline costs, and risk differentials. More broadly, the paper connects automation economics, organizational design, AI governance, and workforce planning in a tractable framework for studying AI-driven task allocation within organizations.
\end{abstract}


\section{Introduction}
\label{sec:Introduction}

Artificial intelligence (AI) is changing how organizations perform a growing range of activities. AI systems can now undertake parts of analysis, prediction, writing, coding, information processing and decision support. Their use raises an important economic question for firms: when should a task be assigned to a human employee, and when should it be assigned to an AI agent? The answer matters not only for the composition of the workforce but also for the way work is distributed across an organization.

A system's ability to perform an activity does not, by itself, make it economically preferable to a human employee. An organization must consider the full cost of deploying the system, the capability required for the task, the scale over which the system can be used, its position within the organizational structure, and the consequences of errors or failure. These considerations can point in different directions. AI may have a lower direct operating cost but expose the organization to reliability, compliance or reputational risks. A task may be technically automatable but remain costly to transfer because it is embedded in a hierarchy of supervision and coordination. Conversely, a system developed at a substantial fixed cost may become less expensive per deployment when it is used repeatedly.

The decision is therefore an organizational allocation problem, not simply a comparison between human wages and the price of an AI service. It depends on the relative costs of human skill and AI capability, the distribution of agents across hierarchical levels, the scale of AI deployment, and the risk attached to each assignment. The same technology may consequently lead to different choices across organizations, tasks and levels of responsibility.

This paper develops a Human--AI Task Allocation (HAT) framework to study this problem. It considers a hierarchical organization in which human and AI agents may be available at different levels. For a given task, the organization allocates work across the feasible agents so as to minimize total risk-adjusted cost. The central question is: under what conditions does an AI agent become the cost-minimizing choice, and how do organizational structure, deployment scale, risk and adaptation shape that choice?

\subsection{The Economic Issues}

The first issue is that human work and AI deployment have different cost structures. The cost of a human agent reflects compensation and the cost of skill. In the model, the cost of line workers varies with skill and organizational depth, while the costs associated with managerial positions depend on hierarchical location. Human expertise is treated as a scarce input whose acquisition and use carry a cost.

The cost of an AI agent is represented differently. It includes a fixed training or development cost, an operating cost incurred when the system is deployed, and a capability-related cost. The fixed component is amortized across the number of deployments. As the number of deployments increases, the fixed training cost attributable to each deployment falls; however, the model does not treat AI use as costless. Marginal operating and capability costs remain part of the comparison. The deployment count is taken as given rather than chosen endogenously by the organization.

The framework maintains a specific human--AI cost asymmetry (Assumption~\ref{ass:asymmetry}). Holding the respective baseline costs fixed, the marginal capability-related cost of AI agents is assumed not to exceed the marginal skill-related cost of human agents at the line-worker level, and the cost of AI coordination across hierarchical levels is assumed to escalate no faster than the human counterpart. These conditions do not enter the pairwise substitution condition of the Human--AI Substitution Principle, which holds for any cost values; they govern how that condition varies with skill, capability, deployment scale and hierarchical position, and so give the comparative results their economic content. They are maintained assumptions of the model, not universal empirical laws. They do not imply that every AI agent is cheaper than every human agent: baseline costs, deployment scale, capability requirements and risk can all affect the final ranking.

A second issue is the role of organizational hierarchy. The model represents an organization with managerial layers above a line-worker layer. The span of control at each managerial level determines the number of positions at the levels below it. Organizational depth and width are therefore linked. Costs can vary with hierarchical position, so the relative cost of human and AI agents may differ across levels. The allocation decision at one level cannot be understood entirely in isolation from the structure in which that level is embedded.

The hierarchy results are conditional on the model's cost specification and on the conditions stated in the relevant theorems. In particular, the result that AI availability can weakly reduce optimal organizational depth relies on a layer-redundancy condition: at every depth greater than the human-optimal depth, one managerial layer can be removed without increasing the minimum AI-enabled risk-adjusted cost. When depth falls strictly and comparable execution capacity is maintained, the organization requires a weakly larger average span of control, measured as the geometric mean of the spans. The result is not a claim that AI necessarily flattens every firm. It identifies a condition under which a shallower structure is cost-minimizing within the model.

A third issue is that direct cost is not the only relevant cost. The framework incorporates risk into the effective cost of both human and AI agents. Human risk can include operational errors, absenteeism and turnover. AI risk is separated into reliability risk, compliance risk and reputational risk. These components are weighted to reflect the organization's exposure and combined with a risk-sensitivity parameter. An AI agent may therefore have a nominal cost advantage and still be more expensive after risk is taken into account. Equally, the cost comparison may favor AI when the risk penalty is sufficiently small relative to its direct cost advantage.

A fourth issue concerns the form of the allocation itself. The organization chooses allocation weights across the feasible human and AI agents, subject to the requirement that the weights sum to one. In the unconstrained benchmark, the objective is linear in these weights. A cost-minimizing solution therefore exists at an extreme point of the feasible set: when one agent has a uniquely lowest risk-adjusted cost, the task is assigned to that agent. If agents tie, more than one allocation can be optimal. Additional feasibility requirements, such as a minimum human share, can also produce a constrained allocation in which both human and AI agents remain involved.

This distinction is important for interpreting the paper's discussion of hybrid organizations. A mixed allocation in this framework is not, by itself, evidence of a production technology in which human and AI inputs are technological complements. In the unconstrained linear benchmark, a strict cost ranking favors a single lowest-cost agent. For a single task, hybrid allocation requires indifference among agents or the additional feasibility conditions considered in the analysis. In constrained or heterogeneous multi-task settings, Theorem~\ref{thm:Constrained Interior Optima} notes that heterogeneity in risk-adjusted costs can also sustain hybrid human--AI organizational structures without an exogenous minimum-human-fraction constraint (see also the remark following Theorem~\ref{thm:AI-Dominant Extreme-Point Allocation}). Risk affects the ranking of agents; constraints can affect which combinations of agents are feasible.

The model also examines how substitution incentives differ across organizational levels and human skill levels. Since the cost of human agents varies with hierarchical position, the incentive to substitute need not be uniform throughout the organization. An intermediate level can exhibit the greatest substitution vulnerability only when the specified relationship between cost incentives and substitution feasibility satisfies the required single-crossing condition. Similarly, the skill threshold in the paper applies to the specified line-worker cost comparison for given AI and risk parameters. It does not establish that highly skilled workers are generally more vulnerable, or generally more protected, than other workers.

Finally, the framework considers changes in relative costs and strategic adaptation. One set of results studies a specified path in which AI risk-adjusted costs are non-increasing while human risk-adjusted costs are held fixed. In the unconstrained problem, and with negligible switching costs, once an AI-only allocation becomes optimal along this path, there continues to exist an optimal AI-only allocation at every subsequent time. A separate result establishes stability while the ordering of relevant risk-adjusted costs remains unchanged, with a change in allocation when a threshold is crossed. The strategic extension considers a human agent and an AI agent competing for a task at a given organizational level. Human adaptation effort may represent upskilling, task redesign or process improvement; AI effort may represent capability improvement or risk reduction. The model establishes an equilibrium for this effort-and-allocation game under its stated conditions.

\subsection{Scope and Interpretation}

The HAT framework is a normative, partial-equilibrium model of task allocation within a hierarchical organization. It identifies the allocation that minimizes modeled risk-adjusted cost for a given task and a specified set of feasible agents, cost functions, risks and organizational constraints. Its purpose is to clarify the conditions under which the organization would prefer one allocation to another within that environment. The results are conditional theoretical statements, not forecasts of what all organizations will do.

Several boundaries follow from this scope. The paper does not model the aggregate creation or destruction of tasks, the formation of new occupations, occupational mobility, or the general-equilibrium effects of AI adoption on wages, employment and unemployment. It does not estimate the causal effect of AI adoption on actual firms or establish the net employment effect of AI for the economy. The allocation problem begins with a task and a set of feasible human and AI agents; it does not explain how the task itself changes or how the firm's scale of production responds.

The model also takes important technological and institutional features as given. The human--AI cost asymmetry is maintained rather than derived from an endogenous theory of technological progress. The number of AI deployments is treated as exogenous, so the model does not determine the organization's optimal investment in AI or its optimal deployment scale. Risk components, their weights and organizational risk sensitivity enter as parameters; the analysis does not solve a separate investment problem in monitoring, compliance or risk governance.

The representation of hierarchy is deliberately focused. Spans of control determine the number of agents at each level, and hierarchical position affects costs. The model does not provide a complete account of how managers acquire, process and transmit knowledge, how exceptional problems are referred upward, or how formal decision rights and accountability are allocated between humans and AI. Managerial capability is represented through the level-dependent cost structure rather than modeled separately for each manager. Likewise, the strategic extension is a two-agent effort-and-allocation game at a given organizational level, not a general model of long-run education, career choice, worker mobility or AI innovation.

These boundaries do not diminish the role of the framework; they define the question it is designed to answer. HAT isolates the internal assignment decision through which technological capability is translated into an organizational choice. Broader questions about production complementarities, endogenous innovation, workforce transitions and labor-market adjustment require additional mechanisms and are left for future research.

\subsection{Contributions}

The paper makes three connected contributions to the analysis of AI adoption within firms.

First, it develops a unified framework for assigning a task across human and AI agents embedded in a hierarchy. The framework brings together distinct human and AI cost functions, the amortization of AI training costs over deployments, hierarchical position and agent-specific risk. This makes the relevant comparison an all-in, risk-adjusted cost comparison rather than a comparison of wages with the operating price of an AI system. The maintained cost-asymmetry assumption gives the model a clear economic structure while keeping the conditions behind the results explicit.

Second, the paper derives allocation and organizational results from this framework. The Human--AI Substitution Principle (Theorem~\ref{thm:Human--AI Substitution Principle}) identifies the risk-adjusted cost condition under which an AI agent is preferred to a human agent. The associated threshold and extreme-point results characterize how the optimal allocation changes when relative costs cross and why the unconstrained linear benchmark generally selects a lowest-cost agent. The analysis also derives a condition for an AI-dominant allocation, a skill threshold for the specified line-worker comparison, and results on how cost-based substitution incentives vary across hierarchical levels.

Third, the paper connects task allocation to organizational structure, risk constraints and strategic adaptation. It characterizes human cost escalation with organizational depth and establishes a conditional result under which AI availability can reduce optimal depth and, when execution capacity is maintained, widen average span of control. It also analyzes the role of reliability, compliance and reputational risk, as well as the effect of additional feasibility constraints on the set of optimal allocations. Finally, it studies how allocation responds to a specified path of AI costs and how human and AI agents' effort choices interact in a two-stage strategic setting.

Taken together, these contributions provide a tractable account of how relative costs, risk, deployment scale and hierarchical position shape the assignment of a given task within an organization. The paper's contribution is deliberately bounded: it does not claim a general theory of technological change or labor-market adjustment. It formalizes a particular organizational decision and establishes the implications that follow from the model's assumptions.

Our paper is closest to recent models of AI in knowledge-based hierarchies \citep{ide2025,xu2025genai} and to the production-chain theory of \citet{demirer2026}. Those papers study how AI autonomy, deployment mode or the sequencing of production steps reshapes organizations. We take the task as given and ask which feasible human or AI agent minimizes risk-adjusted cost at each level of a hierarchy, obtaining explicit conditions for substitution (Theorems~\ref{thm:Human--AI Substitution Principle}--\ref{thm:Threshold Substitution Rule}), a layer-redundancy condition under which AI availability weakly reduces depth and, when execution capacity is maintained, weakly widens the average span of control (Theorem~\ref{thm:Optimal Depth under AI Availability} and Corollary~\ref{cor:Flattening of Organizations}), and a strategic-adaptation extension (Theorem~\ref{thm:Strategic Substitution Equilibrium}).

\subsection{Organization of the Paper}

The remainder of the paper is organized as follows. Section~\ref{sec:Prior Work} reviews the related literature and positions the HAT framework within the relevant economic debates. Section~\ref{sec:Human-AI Task Allocation Model} presents the organizational structure, human and AI cost functions, risk adjustment and feasible allocation problem. Section~\ref{sec:Properties of this model} develops the theoretical results and provides the corresponding proofs in the Appendix. Section~\ref{sec:Discussions} discusses the economic interpretation, organizational implications, empirical predictions and limitations of the framework. Section~\ref{sec:Conclusions} concludes. Figure~\ref{fig:HATmodelfindings} summarizes the work presented in this paper, including the key analytical results.

\begin{figure}[h]
\hspace{-5mm}
\includegraphics[width=1.05\textwidth,height=0.57\textwidth]{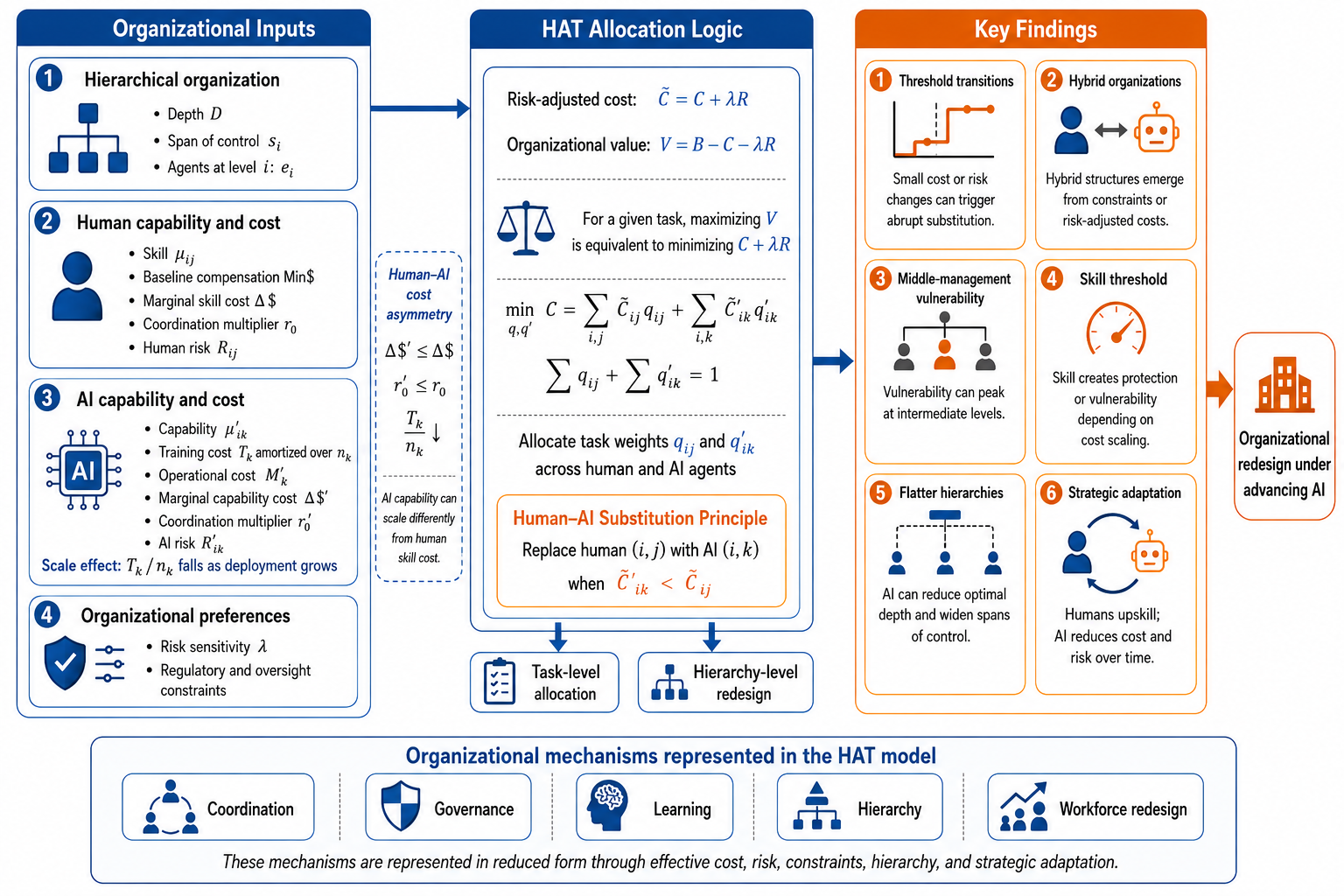}
\caption{The proposed HAT model and our key findings.}
\label{fig:HATmodelfindings}
\end{figure}


\section{Prior Work}
\label{sec:Prior Work}

This section reviews the economic work most closely related to the HAT framework. The discussion is broad because the model brings together several questions that are often studied separately: how technology changes the allocation of tasks, how firms organize work across hierarchical levels, how AI costs differ from the costs of human expertise, and how risk affects the choice between human and AI agents. The purpose is not to claim that these literatures have left AI and organizations unexplored. It is to show what each literature explains, what it leaves outside its main question, and where the HAT model fits. The review follows the model's actual structure and results. It does not treat the paper as a general theory of labor markets, technological innovation or human--AI complementarity.

\subsection{Task-Based Economics of Automation and Technological Substitution}

The task-based approach is a natural starting point for studying the economic effects of automation. It treats production as a set of tasks rather than as a fixed set of occupations. Workers with different skills and machines have different advantages across tasks. A new technology changes the relative cost or productivity of performing some of those tasks, and firms adjust the mix of labor and capital accordingly. \citet{autor2003} show that computerization changed the demand for workers according to the routine and non-routine content of their jobs. \citet{acemoglu2011} place skills, tasks and technologies in a common framework. Their approach explains why workers with similar levels of education can be affected differently when the task content of production changes.

The evidence on occupational polarization gives this framework an important empirical setting. \citet{goos2009} document changes in the distribution of employment across occupations in Europe. \citet{autor2013} examine the growth of low-skill service work alongside the decline of routine middle-skill employment in the United States. These changes show that technological progress does not affect all workers in the same way. The result depends on the tasks performed within an occupation and on the demand for other activities that remain difficult to automate. The relationship between skill and technological exposure is therefore more complex than a simple division between high-skilled and low-skilled work.

A central contribution of the task-based literature is to distinguish the displacement of labor from the creation of new tasks. \citet{autor2015} emphasizes that automation can remove some activities from human work while increasing the demand for other activities. \citet{acemoglu2018} formalize the displacement effect of automation and its implications for growth, factor shares and employment. Their later framework \citep{acemoglu2019} distinguishes automation from the creation of new tasks in which labor has a comparative advantage. These forces can work in opposite directions. Automation may reduce labor demand in tasks transferred to machines, while new tasks may raise demand for labor elsewhere. The net effect on employment and wages cannot be inferred from the productivity of a machine alone.

Subsequent work has developed both the theoretical and empirical implications of this approach. \citet{acemoglu2020robots} examine the labor-market effects of industrial robots in local U.S. labor markets. Their analysis of AI raises a further concern: the direction of technical change matters, because innovation can be directed toward automating existing tasks or creating tasks in which human labor remains important \citep{acemoglu2020wrongai}. \citet{autor2022} reviews the evidence and stresses the uncertainty surrounding the labor-market effects of technological change. \citet{restrepo2024} surveys the theory and evidence on automation, wages and employment. \citet{acemoglu2025tasks} provide a further treatment of comparative advantage, technology and labor demand in the task framework. Taken together, this work shows that the effects of technological change depend on the tasks affected, the relative advantages of workers and machines, and the creation of new activities.

The task-based approach also provides a useful way to think about AI exposure. \citet{frey2017future} estimate the susceptibility of occupations to computerization. \citet{webb2020} studies the relationship between AI-related technologies and occupational tasks. \citet{felten2021} develop measures of AI exposure across occupations, industries and geographic areas. \citet{eloundou2024} assess the potential exposure of work activities to large language models. These studies help identify where a technology may become relevant. They do not, by themselves, establish whether a firm will adopt the technology, whether it will assign a particular task to an AI system, or whether a worker will lose a job. Exposure, adoption, task reassignment and employment change are separate stages in the economic process.

Firm-level evidence further shows why the distinction matters. \citet{acemoglu2022} use online vacancies to examine how AI exposure is related to the skill requirements of job postings. Their evidence concerns changes in hiring demand, not a complete account of how firms allocate each task internally. \citet{babina2024} study AI investment and firm growth, finding that AI-investing firms can expand in sales, employment and product innovation. This is a reminder that AI investment need not lead immediately or uniformly to a reduction in employment. It may also support expansion or change the mix of activities within a firm. The HAT model isolates one part of this broader process: the assignment of a given organizational task to a feasible human or AI agent under specified cost and risk conditions.

The distinction between tasks and jobs is also important for the present analysis. A job usually combines several tasks, and an occupation combines jobs with different task profiles. The HAT framework does not model the construction of jobs from component tasks or the reallocation of labor across occupations. It takes a task as given and examines the allocation of that task among human and AI agents located at different levels of a hierarchy. The model therefore addresses a within-organization decision that lies between technological exposure and the aggregate labor-market outcomes studied in the task-based literature. It complements that literature; it does not replace its analysis of new-task creation, worker mobility, wages or economy-wide employment.

This difference in the unit of analysis is central to the interpretation of the HAT results. Theorem~\ref{thm:Human--AI Substitution Principle} states the risk-adjusted cost condition under which an AI agent is preferred to a human agent in the model. Theorem~\ref{thm:Global Optimal Allocation} establishes the extreme-point property of the unconstrained linear allocation problem, and Theorem~\ref{thm:Threshold Substitution Rule} describes the associated threshold. These are results about the allocation of a given task among the feasible agents in the model. They do not establish the aggregate displacement or reinstatement effects that arise when firms create new tasks, reorganize occupations or adjust their scale of production.

\subsection{The Economics of the Firm: Hierarchy, Information and the Organization of Work}

The task-based literature explains how technology changes the allocation of work between labor and capital. Organizational economics asks how firms arrange that work internally. Firms must divide activities, transmit information, coordinate decisions and monitor performance. These choices affect the cost of production and the capacity of the organization to respond to problems. \citet{williamson1967} studies hierarchical control and the limits to firm size. \citet{simon1947}, and \citet{march1958}, place bounded rationality and administrative coordination at the center of organizational analysis. Their work explains why the design of an organization matters when no single decision-maker can process all relevant information or supervise every activity directly.

\citet{radner1992} reviews the economics of hierarchy and the problem of managing organizations with distributed information. \citet{radner1993} develops a formal account of decentralized information processing. \citet{bolton1994} model the firm as a communication network in which the routing of information affects the allocation of administrative activities. These contributions differ in their details, but share an important insight: organizational structure is part of the production problem. The costs of communication, information processing and control help determine how tasks should be distributed across people and levels. A change in the technology used to process information can therefore affect the organization of production as well as the productivity of an individual input.

A major strand of this literature studies knowledge-based hierarchies. \citet{garicano2000} develops a model in which workers solve common problems and refer more difficult problems to agents with greater knowledge. The depth of the hierarchy and the span of control depend on the costs of acquiring knowledge and communicating problems upward. \citet{garicano2006organization} extend this approach to the knowledge economy and examine how the organization of knowledge interacts with skill differences and inequality. Their review \citep{garicano2015} explains how knowledge hierarchies can be used to understand firm organization and broader economic outcomes. These models provide a direct reference point for HAT because both study human and other agents placed at different levels of a hierarchy.

The HAT model uses a more reduced-form representation of hierarchy than the knowledge-based approach. The number of positions at each level is determined by the spans of control above that level. The model then assigns human and AI agents to positions and allows their costs to vary with hierarchical position. It does not model the knowledge held by each manager, the distribution of problem types, or the referral of exceptional problems from workers to experts. It therefore does not reproduce the full knowledge-acquisition mechanism in \citet{garicano2000}. Its focus is narrower: how the modeled cost differences between human and AI agents affect task allocation and the cost-minimizing depth of the organization under the stated assumptions.

The empirical literature on production hierarchies helps motivate this focus. \citet{caliendo2015} examine the anatomy of French production hierarchies and show how firm size, productivity and internal structure are related. \citet{caliendo2020} study productivity and organization in Portuguese firms. Their work highlights the importance of management layers and the distribution of employees across levels. \citet{bloom2014} show that information technology and communication technology can have distinct effects on firm organization, including autonomy and managerial spans. These findings support the view that technology and organizational structure are linked. They do not imply that every technology has the same effect or that AI must reduce the number of managerial layers in every firm.

The cost of managerial talent and the size of firms provide another part of the background. \citet{lucas1978} links firm size to managerial ability. \citet{rosen1982} studies authority, control and the distribution of earnings, while \citet{gabaix2008} explain executive compensation through the assignment of managerial talent to firms of different sizes. This literature helps explain why costs and responsibilities may differ across organizational levels. HAT represents the difference in human costs across levels through a reduced-form cost function. It does not derive the hierarchy-specific wage structure from an equilibrium model of managerial talent, firm size or matching. That distinction matters because the theorems are conditional on the cost function used in the model.

The model's results on hierarchy should be read against this background. Theorem~\ref{thm:Depth--Cost Growth and Dominance} characterizes the increase in human costs with organizational depth under the specified cost structure. Theorem~\ref{thm:Optimal Depth under AI Availability} compares the minimum risk-adjusted cost of a human-only organization with the minimum cost when AI is available. The result that optimal depth is weakly lower with AI depends on the theorem's assumptions, including the layer-redundancy condition. Corollary~\ref{cor:Flattening of Organizations} relates a reduction in depth to a wider average span when execution capacity is maintained. Theorem~\ref{thm:Level-Dependent Substitution Likelihood} studies how the cost-based incentive to substitute varies across hierarchical levels. Corollary~\ref{cor:Middle-Management Vulnerability} gives a conditional intermediate-level result when the specified single-crossing condition holds. None of these results implies that all firms will flatten, that managers will always be replaced, or that the most vulnerable level is the same in every organization.

The broader organization literature also warns against treating hierarchy as only a wage ladder. A hierarchy can allocate information, decision rights and responsibility. This point is developed in work on organizational information processing \citep{galbraith1974,tushman1978} and in the economics of communication and authority \citep{radner1992,radner1993,bolton1994}. HAT captures selected cost consequences of hierarchical position. It does not model every function performed by managers, nor does it derive the optimal distribution of knowledge and authority. These are important limits on how its hierarchy results should be interpreted.

\subsection{Authority, Coordination and Organizational Adaptation}

A further part of organizational economics studies who has the right to make decisions and how that right changes when information is distributed across a firm. \citet{aghion1997} distinguish formal authority from real authority. Formal authority refers to the right to decide, while real authority depends on who has the information and practical ability to shape the decision. \citet{dessein2002} examines the relationship between authority and communication. \citet{alonso2008} study when coordination requires centralization, and \citet{dessein2006} analyze how organizations balance local adaptation with coordination across units. These models show that organizational design involves more than assigning a task to the lowest-cost worker. It also involves the location of information and decision rights.

Information-processing approaches reach a related conclusion. \citet{galbraith1974} argues that organizational design must reflect the information-processing demands of the tasks a firm performs. \citet{tushman1978} develop information processing as a central concept in organizational design. \citet{simon1947}, and \citet{march1958}, emphasize that bounded rationality and limited attention make the division of decision-making work important. These arguments are relevant to AI because the technology may change the cost of collecting, classifying and processing information. But the consequences for authority depend on who can evaluate outputs, recognize exceptions and bear responsibility for decisions.

The literature also considers how organizations adapt when their information environment changes. \citet{dessein2006} analyze adaptive organizations and the trade-off between local information and coordination. Related institutional work emphasizes that organizations do not choose practices solely on the basis of measured efficiency. \citet{dimaggio1983} describe how coercive, normative and mimetic pressures shape organizational adoption. \citet{scott2008} reviews the role of institutions in organizational life. These perspectives help explain why organizations may adopt, delay or reject technologies for reasons that include regulation, professional standards and legitimacy. They are useful context for HAT's risk terms, but the model does not endogenize institutional pressure or explain diffusion through imitation.

The HAT model should not be read as a theory of authority transfer. It does not decide whether an AI agent has formal authority, whether a human manager retains final responsibility, or whether the organization centralizes or decentralizes decisions after adoption. Instead, it represents human and AI agents as occupying positions in a hierarchy and compares their effective costs for a given task. The level and span variables describe the structure of the organization; they do not provide a separate model of the allocation of decision rights. This distinction separates HAT from work that focuses directly on authority, decentralization and adaptation.

\subsection{Managerial Costs, Firm Size and Evidence on Organizational Change}

The cost structure of a hierarchy is related to the economics of firm size and managerial talent. \citet{lucas1978} studies how differences in managerial ability are connected to the size distribution of firms. \citet{rosen1982} examines authority and the distribution of earnings, while \citet{gabaix2008} develop an assignment model in which the size of firms affects the market value of managerial talent. These studies provide a rationale for the idea that positions with greater responsibility can carry higher costs. They do not establish the exact functional form used in HAT, which is a modeling choice designed to capture rising human costs across hierarchical levels.

Research on information technology offers evidence that the introduction of a new technology can be associated with organizational redesign. \citet{brynjolfsson2000} argue that the economic value of information technology depends on investments in complementary organizational changes, not just on the purchase of computers or software. \citet{bresnahan2002} provide firm-level evidence of complementarities among information technology, workplace organization and skilled labor. \citet{bloom2014} distinguish the organizational effects of information technology from those of communication technology. Their findings caution against treating all technologies as having the same effect on delegation and spans of control.

The firm-level hierarchy studies of \citet{caliendo2015,caliendo2020} add evidence on the relationship between organizational layers, firm size and productivity. Taken together, this literature shows that a firm's internal structure is related to its technology and performance. It provides empirical motivation for studying how AI availability may affect the cost-minimizing hierarchy. However, the results from earlier information technologies cannot simply be transferred to AI. Generative AI differs in the tasks it can perform, the uncertainty of its outputs, and the way it may be used as an autonomous agent or as an input into human work.

Recent public and practitioner discussion has also focused on whether AI will reduce middle management. The \citet{revelio2023} report, for example, discusses flattening and changes in management employment. Such reports can help identify questions for research, but they should not be treated as causal evidence that AI itself has eliminated middle-management roles. Changes in management layers may also reflect firm growth, restructuring, macroeconomic conditions or other technologies. HAT's result on middle-management vulnerability is theoretical and conditional: Corollary~\ref{cor:Middle-Management Vulnerability} identifies an intermediate-level peak only under the specified relationship between cost incentives and substitution feasibility. It is not an empirical ranking of occupations or a universal prediction about management employment.

\subsection{The Economics of AI: Prediction, Model Development and Deployment Costs}

The economics of AI is closely connected to the economics of prediction. \citet{agrawal2019} argue that machine learning reduces the cost of prediction and changes the value of other inputs used in decision-making, including human judgment. \citet{brynjolfsson2018} examine what machine-learning systems can learn and what this means for occupations and economic activity. \citet{cockburn2018} discuss AI as a general-purpose technology with implications for innovation. \citet{furman2019} provide a broader review of AI and the economy, including its possible effects on innovation, productivity and labor markets. These contributions explain why AI can affect a wide range of cognitive and information-intensive activities. They also show why AI should not be understood only as a machine that replaces a complete occupation: it can change the cost of one component of a broader decision or production process.

The economics of digital technologies provides a second foundation. \citet{goldfarb2019} review how digital technologies affect markets by changing the costs of search, storage, replication, communication and information processing. A common feature of many digital goods is that they involve substantial initial development costs and can then be reused across many applications. This can reduce the average cost of an additional use. The basic point is relevant to AI systems, but it should not be overstated. AI models require computing resources to train and operate. Their cost depends on the model, the task, the number of users, the required level of accuracy and the infrastructure used for deployment.

The cost of developing frontier models can be substantial. \citet{cottier2024rising} document the rising costs of training frontier AI models. \citet{hoffmann2022} study the relationship between model size, training data and compute in compute-optimal language-model training. At the same time, the cost of using a trained model has changed over time and has not fallen uniformly across tasks. \citet{epoch2025inference} report that inference prices have declined at different rates across systems and uses. These sources support a distinction between the fixed or sunk expenditure associated with model development and the recurring cost of inference, integration, maintenance and oversight. They do not support the claim that AI has zero marginal cost.

HAT represents this distinction through a fixed training cost and a marginal operational cost. The fixed cost is amortized over the number of deployments, while each deployment carries a marginal operating cost and a capability-related cost specified in the model. As the number of deployments rises, the fixed training-cost component per deployment falls. This is a cost-accounting channel: the system may become less expensive per deployment even if its technical capability does not change. The number of deployments is not chosen endogenously in the present model. The results therefore concern the effect of amortization under the specified deployment environment, not the diffusion of AI across an industry or an adoption cascade among firms.

This distinction is important when interpreting the AI cost asymmetry assumption. HAT assumes that, holding the respective baseline costs fixed, the marginal capability-related cost of AI agents does not grow faster than the marginal skill-related cost of human agents at the line-worker level, and it specifies a cost relationship for AI across organizational levels. It also includes amortization of the AI training cost over deployments. These are maintained assumptions. The model does not derive them from an independent theory of AI innovation or human learning. Nor does it claim that AI will be cheaper for every task. In some settings, data preparation, computing, integration, reliability, compliance and oversight may make AI relatively costly. The substitution results apply to the cost environment described by the assumptions, not to all possible AI technologies or organizations.

The evidence on firm AI investment also cautions against equating adoption with substitution. \citet{babina2024} find that AI-investing firms experience changes in firm growth and product innovation. These outcomes can affect the scale and composition of firm activity, and can therefore change labor demand through channels other than the replacement of existing tasks. The HAT model isolates one channel: the choice between feasible human and AI agents for a given task. It does not model product innovation, firm entry, changes in output scale or the full productivity effects of AI investment. This makes it useful to distinguish the paper's cost-allocation results from broader claims about firm performance.

\subsection{Evidence on AI Exposure, Productivity and Human--AI Interaction}

A growing empirical literature studies how AI affects the tasks performed by workers. Exposure measures estimate whether a technology can be applied to activities within an occupation or industry. \citet{felten2021} provide measures of occupational, industry and geographic exposure to AI. \citet{eloundou2024} estimate the potential exposure of work activities to large language models. These studies help describe the possible reach of AI. Yet an exposure measure does not show whether a firm adopts AI, how it uses the technology or whether a human task is actually reassigned. This is the same distinction that appears in the task-based literature between the technical possibility of automation and the economic decision to adopt it.

Experimental and firm-level studies provide evidence on the effects of AI after it is introduced into work. \citet{noy2023} find that access to generative AI can improve speed and quality in professional writing tasks. \citet{peng2023impact} study the effect of GitHub Copilot on software-development productivity. \citet{brynjolfsson2025} examine generative AI in customer support and report productivity effects that differ across workers. \citet{dellacqua2026} study knowledge workers and find that AI's effects vary across tasks and depend on where the task lies relative to the technology's capabilities. These studies show that the productivity consequences of AI are not uniform across tasks or workers.

The findings also make clear that productivity improvement is not the same as substitution. An AI system may help a worker complete a task more quickly, improve the quality of an output, or make it possible for a worker with less experience to perform at a higher level. In other cases, the system may perform an activity that previously required human input. Whether a productivity gain leads to a reduction in employment depends on output demand, firm scale, the set of tasks affected and the organization's decision about how to allocate work. The HAT framework addresses only the allocation decision for a given task among the agents available in the model. It does not estimate the productivity effects documented in these studies or determine the resulting change in total employment.

The distinction between automation and augmentation is also central to the literature on human--AI collaboration. \citet{raisch2021} describe an automation--augmentation paradox: AI can replace some human activities while supporting or improving others. \citet{agrawal2019} explain how lower-cost prediction can increase the value of complementary human judgment. \citet{athey2020} focus on the allocation of decision authority between human and artificial intelligence. \citet{kleinberg2018} study how machine predictions can be combined with human decisions. These papers show that the organizational value of AI depends on how its output enters the decision process, not only on the accuracy or cost of the algorithm.

Research on human--AI decision support provides further detail. \citet{bastani2026} study how reinforcement-learning-generated guidance affects human sequential decision-making. Their results indicate that users do not simply follow algorithmic suggestions; they combine them with their own experience. \citet{dietvorst2015} document algorithm aversion after users observe an algorithm make errors. \citet{dietvorst2018} show that allowing people to modify algorithmic outputs can increase their willingness to use algorithms. \citet{logg2019} identify settings in which people show algorithm appreciation. \citet{lee2004} explain why appropriate reliance on automation requires trust that is calibrated to system reliability and the decision context.

These studies are relevant to the HAT model because they show why AI use may involve risks and organizational costs beyond its direct operating expenditure. They do not, however, establish that all human--AI arrangements are productive complements in the formal production-function sense. HAT's linear allocation problem does not specify a technology in which human and AI inputs jointly generate output with a complementarity parameter. Its constrained hybrid allocations describe cases in which both types of agents can remain in the feasible allocation. This is a statement about the allocation problem and its constraints, not a general result about joint production or productivity gains from collaboration.

\subsection{Risk, Trust, Accountability and Institutional Conditions}

A firm's decision to use AI depends not only on its technical performance and direct cost, but also on the consequences of failure. An incorrect output may create a small loss in a routine information task and a much larger loss in a medical, financial or legal decision. The relevant cost can include the probability of failure, the size of the loss, the visibility of the error and the extent to which the organization is held responsible. Research on automation and trust has long emphasized that users should rely on automated systems in ways that reflect their actual reliability \citep{lee2004}. Work on algorithm aversion and algorithm appreciation shows that public and employee responses to algorithmic decisions may also depend on whether errors are observed and whether users can influence the outcome \citep{dietvorst2015,dietvorst2018,logg2019}.

The literature on interpretable machine learning and responsible AI highlights further concerns. \citet{doshi2017} call for a more rigorous science of interpretability. \citet{floridi2018} set out principles for the responsible development and use of AI. \citet{matthias2004} discusses the responsibility gap that can arise when learning systems act in ways that are difficult to anticipate. \citet{santonidesio2021} identify different responsibility gaps associated with AI and discuss how they can be addressed. \citet{vaccaro2019} study contestability in algorithmic systems, including the ability of affected people to challenge or question automated decisions. These contributions show that reliability, accountability, transparency and the ability to contest decisions are important conditions for the use of AI.

Institutional and professional settings can affect these costs. \citet{abbott1988} explains how professions maintain jurisdiction over areas of expert work through credentials, training and accepted authority. Product-market rules and regulatory institutions can shape the conditions under which firms adopt new technologies; \citet{oecd2021} provides a comparative review of product-market regulation. \citet{dimaggio1983} and \citet{scott2008} explain how coercive, normative and mimetic pressures influence organizational practices. \citet{dutton1994} examine the relationship between organizational image and member identification. These studies provide useful context for understanding why firms may face legal, professional or reputational limits when they use AI. They are not direct estimates of the risk parameters in HAT, and the model does not formally represent professional licensing, institutional isomorphism or organizational identity.

HAT translates selected consequences of failure into risk-adjusted effective costs. Human agents face an idiosyncratic risk term. AI risk is separated into reliability, compliance and reputational components, which enter the effective cost through the model's risk-sensitivity parameter and weights. The resulting cost comparison distinguishes nominal expenditure from the cost measure used in the allocation decision. Theorem~\ref{thm:Human--AI Substitution Principle} states the substitution condition in terms of risk-adjusted effective cost, while Theorem~\ref{thm:Risk-Adjusted Substitution Principle} sets out the role of the AI risk components in that comparison. The framework is deliberately reduced-form: the risk components and sensitivity weights are treated as model inputs. It does not derive them from a theory of regulation, liability, insurance, stakeholder beliefs or investment in monitoring.

This qualification is important for interpreting Theorem~\ref{thm:Constrained Interior Optima} and the paper's discussion of hybrid organizations. The unconstrained allocation problem is linear over a simplex. If one agent has a strictly lower risk-adjusted cost than every other agent, the optimum is an extreme-point allocation to that agent. A mixed allocation is possible when the relevant costs are tied or when the additional feasibility constraints in the model make an interior combination feasible and optimal within the constrained problem. Risk heterogeneity affects the cost ranking and therefore the identity of the preferred agent; it does not, by itself, guarantee that humans and AI will be used together in the unconstrained single-task problem. In constrained or heterogeneous multi-task settings, the remark in the statement of Theorem~\ref{thm:Constrained Interior Optima} notes that risk-adjusted cost heterogeneity can sustain hybrid human--AI structures without an exogenous minimum-human-fraction constraint. The paper also does not claim that a hybrid allocation necessarily produces a technological complementarity in output.

\subsection{Human Capital, Strategic Adaptation and the Effort-and-Allocation Game}

Technological change affects not only the tasks that workers perform, but also the return to different skills. The human capital literature studies how education and experience relate to productivity and earnings. \citet{mincer1974} provides a foundational framework linking schooling and experience to wages. \citet{lemieux2006} examines the relationship between postsecondary education and wage inequality. \citet{acemoglu2011} show how the effects of technological change depend on the skills required by tasks and on the comparative advantage of different worker groups. This literature provides a background for the HAT model's distinction between human skill and AI capability costs.

The relationship between skill and AI substitution is not one-directional. A higher level of human skill may raise the cost of employing a worker, but it may also be associated with tasks that require capabilities that are difficult for AI to provide. Measures of AI exposure have found that some highly skilled occupations are exposed to AI, but exposure does not show that these workers will be replaced \citep{webb2020,felten2021}. The HAT model makes a narrower comparison. Theorem~\ref{thm:Skill-Adjusted Substitution Condition} and Corollary~\ref{cor:Skill-Dependent Human Protection and Vulnerability Zone} derive a skill threshold for the specified line-worker comparison, given a particular AI agent and the model's cost and risk parameters. The result depends on the relative slopes and baseline levels of human and AI costs, organizational depth and risk differences. It does not establish that high-skilled workers are generally more vulnerable, or that low-skilled workers are generally protected.

The HAT model also considers strategic responses to substitution pressure. Workers may choose effort that represents upskilling, task redesign, process improvement or other forms of adaptation within the model. The AI side may choose effort associated with improving capability or reducing risk. The resulting choices affect an agent's prospects of being selected in the subsequent allocation stage, which the model represents by a smooth contest-success measure. This structure is related to the contest literature, in which agents choose costly effort to improve their position or probability of winning. \citet{lazear1981} develop a theory of rank-order tournaments, while \citet{dixit1987} analyzes strategic behavior in contests. These studies provide a formal basis for studying effort choices when agents compete for a prize or position.

The strategic extension in HAT is specific. It considers a human agent and an AI agent competing for a task at a given organizational level. Theorem~\ref{thm:Strategic Substitution Equilibrium} establishes a Nash equilibrium for the effort-and-allocation game under the stated conditions. The existence argument draws on standard fixed-point and equilibrium results \citep{debreu1952,fan1952,glicksberg1952}. The theorem does not provide a general model of education choice, career change, worker mobility or technological innovation. Nor does it explain how human capital accumulates over a working life. It studies strategic effort choices within the two-stage game specified in the paper.

The paper also examines allocation changes along a specified path of AI costs. Theorem~\ref{thm:Monotonic Automation Path} assumes that AI risk-adjusted costs are non-increasing while human risk-adjusted costs remain fixed, in the unconstrained allocation problem with negligible switching costs. Under those conditions, once an AI-only extreme-point allocation is optimal, there continues to exist an optimal AI-only extreme-point allocation at every subsequent time. Corollary~\ref{cor:Irreversibility of Automation} provides a non-reversion result under the same conditions. Theorem~\ref{thm:Perturbation Stability} establishes that small changes in costs do not change the optimal allocation as long as the relevant cost ordering remains unchanged; Corollary~\ref{cor:Discontinuous Workforce Transitions} describes the associated threshold change. These results describe what follows from the assumed cost path and linear allocation structure. They do not explain why AI costs decline, and they are not a forecast of the timing or speed of actual adoption.

\subsection{Closest Formal Models of AI and Organizational Structure}

Recent formal work makes it especially important to state the contribution of HAT with care. \citet{ide2025}, published in the Journal of Political Economy, introduce AI into a knowledge economy in which people work in hierarchical organizations. They distinguish autonomous AI agents from AI copilots and study how AI autonomy and computational resources affect organizational production, occupational choice and labor income. This paper is close to HAT because it places AI within an economic model of knowledge-based organization. The difference lies in the central mechanism. \citeauthor{ide2025} focus on AI autonomy, computational resources and the allocation of knowledge within a hierarchy. HAT focuses on the risk-adjusted costs of assigning a given task to feasible human and AI agents at different organizational levels, under a maintained cost-asymmetry assumption and a specified deployment-cost structure.

\citet[revised 2026]{xu2025genai} develop a theoretical model of generative AI in knowledge-based hierarchies. They distinguish between automation and augmentation and consider whether AI is deployed at the worker or expert level. They also study the role of reliability and human validation. Their analysis shows that organizational outcomes can depend on both the mode and location of AI deployment. This work is particularly close to HAT because it studies AI adoption within a hierarchy and considers human oversight. HAT asks a different question: given a task and a set of feasible agents at different levels, which agent or allocation minimizes modeled risk-adjusted cost? It does not model the same knowledge problem, the same deployment modes or the full human-validation mechanism in \citeauthor{xu2025genai}.

\citet{demirer2026} study production as a sequence of interdependent steps. Their model allows each step to be performed manually, augmented with AI or automated within contiguous AI-executed chains. Firms also choose how steps are grouped into tasks and jobs. This approach shows that the position of a step in a production sequence can matter for the value of automation. HAT takes a task as given and studies its assignment among feasible human and AI agents within a hierarchy. It does not model the sequence of production steps, the grouping of steps into tasks, or the construction of jobs. The two frameworks therefore address different parts of the organization of work.

\citet{athey2020} study the allocation of decision authority between human agents and AI. \citet{agrawal2019} examine the economic consequences of lower-cost prediction, and \citet{kleinberg2018} analyze how machine predictions can be combined with human decisions. These papers help explain why AI may serve as decision support rather than as an independent task performer. HAT does not model the full decision process or the value of human judgment as a separate input. It compares the effective costs of assigning a task to agents in the modeled organization. Its contribution is therefore complementary to, rather than a replacement for, models of prediction, decision quality and authority.

\citet{raisch2021} describe the tension between automation and augmentation in organizational use of AI. Their discussion is useful for understanding why a firm may use AI to support some activities and substitute for others. The empirical studies of \citet{noy2023}, \citet{peng2023impact}, \citet{brynjolfsson2025} and \citet{dellacqua2026} likewise show that AI effects vary across tasks and workers. HAT does not estimate those productivity effects or derive a production function for human--AI complementarities. Its hybrid-allocation result is a statement about the constrained allocation problem. It should not be described as proving that joint human--AI production is more productive than either input alone.

The model also has a clear line of development. Its organizational structure builds on the mathematical model of \citet{gibson2015mathematical}. A preliminary version of the cost framework appeared in \citet{banerjee2023model}, which studied the cost of applying electronic governance and robotic process automation in a distributed management system. The present HAT framework develops that starting point by deriving level sizes from spans of control, distinguishing human and AI cost functions, including a fixed AI training cost amortized over deployments, introducing risk-adjusted effective costs at the agent level, and adding a strategic effort-and-allocation stage. This lineage should be stated clearly. It helps readers identify what is retained from earlier work and what is added in the present framework.

\subsection{Positioning and Scope of the HAT Framework}

The literature reviewed in this section studies several related but distinct economic questions. Task-based models explain how technology changes the allocation of tasks between labor and capital and how automation and new-task creation affect wages and employment. Organizational economics explains why firms use hierarchies and how they divide information, authority and responsibility. Research on digital technologies helps describe the difference between development costs and recurring costs of use. Studies of AI prediction, exposure and productivity show that the effects of AI vary across tasks and workers. Research on trust, accountability and institutions explains why a technology's direct cost is not the only consideration in an adoption decision. Recent formal models bring AI into knowledge hierarchies, decision-making and the organization of production steps.

The HAT framework brings a specific set of these issues into one model of task assignment within a hierarchical organization. Its contribution is not that it provides the first analysis of AI in organizations or the first economic model of human--AI interaction. Its contribution is the combination of: (i) distinct human and AI cost functions under a maintained cost-asymmetry assumption; (ii) a hierarchy in which the number of positions at each level is determined by spans of control; (iii) an AI cost structure that separates fixed training expenditure from marginal operating costs and amortizes the fixed component over deployments; (iv) risk-adjusted effective costs for human and AI agents, with AI risk separated into reliability, compliance and reputational components; and (v) a strategic effort stage in which human and AI agents choose effort before task allocation.

The theorems establish the implications of this specification. Theorem~\ref{thm:Normalization, Feasibility, and Simplex Structure} describes the feasible allocation set, and Theorem~\ref{thm:Skill-Weighted Cost Decomposition} decomposes total risk-adjusted cost into managerial and line-worker components. Theorems~\ref{thm:Human--AI Substitution Principle}--\ref{thm:Threshold Substitution Rule} establish the risk-adjusted substitution rule, the extreme-point structure of the unconstrained linear problem and the associated threshold. Theorem~\ref{thm:AI-Dominant Extreme-Point Allocation} gives a condition for an AI-dominant extreme-point allocation. Theorems~\ref{thm:Depth--Cost Growth and Dominance}, \ref{thm:Optimal Depth under AI Availability} and \ref{thm:Level-Dependent Substitution Likelihood}, together with Corollaries~\ref{cor:Flattening of Organizations} and \ref{cor:Middle-Management Vulnerability}, connect hierarchical cost differences to organizational depth, span and substitution incentives across levels. Theorem~\ref{thm:Skill-Adjusted Substitution Condition} and Corollary~\ref{cor:Skill-Dependent Human Protection and Vulnerability Zone} derive a skill threshold for the specified line-worker comparison. Theorems~\ref{thm:Risk-Adjusted Substitution Principle} and \ref{thm:Constrained Interior Optima}, together with Corollary~\ref{cor:Minimum Human Fraction Constraint}, address risk-adjusted costs and additional feasibility constraints in the allocation problem. Theorem~\ref{thm:Monotonic Automation Path} and Corollary~\ref{cor:Irreversibility of Automation} study how the allocation changes under a non-increasing path of AI risk-adjusted costs and fixed human-side conditions. Theorem~\ref{thm:Perturbation Stability} and Corollary~\ref{cor:Discontinuous Workforce Transitions} characterize stability while cost rankings remain unchanged and the change associated with crossing a threshold. Theorem~\ref{thm:Strategic Substitution Equilibrium} studies strategic adaptation through the effort-and-allocation game.

These results are conditional on the model's assumptions. The cost asymmetry is maintained rather than derived from a theory of technological progress. The number of AI deployments is not chosen endogenously. The risk components and their weights are treated as parameters rather than as outcomes of investment in monitoring, regulation or liability. The unconstrained allocation objective is linear, so a strict cost ranking leads to an extreme-point solution. For a single task, mixed allocations require ties or the additional feasibility constraints considered in the paper; hybrid structures across heterogeneous tasks can also reflect differences in risk-adjusted costs (Theorem~\ref{thm:Constrained Interior Optima} and the remark following Theorem~\ref{thm:AI-Dominant Extreme-Point Allocation}). The hierarchy is represented through spans of control and level-specific costs, rather than through a full theory of knowledge acquisition, problem referral or decision rights. The strategic extension is an effort-stage game, not a model of long-run education, career decisions or AI innovation.

For the same reason, HAT does not determine aggregate wages, total employment, unemployment, the creation of new tasks, occupational mobility, productivity growth or the distribution of income. It does not model a complete production function in which human and AI inputs are technological complements. Nor does it estimate the effects of AI adoption on actual firms. The empirical work reviewed above provides motivation and possible points of comparison, but the theorems themselves are theoretical results within the specified model. Their empirical relevance remains a question for future work.

The place of HAT in the economics literature is therefore specific. It studies the assignment of a given task to human and AI agents inside a hierarchy. This decision lies between the technological exposure of tasks and the broader effects of automation on firms and labor markets. By making the assignment problem explicit, the framework shows how relative costs, risk, deployment amortization and hierarchical position interact within the modeled organization. The next section sets out the organizational structure, cost functions, risk terms and feasible allocation problem on which the results depend.

\section{A Human--AI Task Allocation (HAT) Model for Hierarchical Organizations}
\label{sec:Human-AI Task Allocation Model}

We build on the hierarchical cost model developed in \citet{gibson2015mathematical}, which establishes
theoretically consistent relationships between organizational structure and economic performance. We extend this framework in four directions: (i) we derive level sizes from a span-of-control primitive
rather than treating them as free parameters; (ii) we incorporate AI agents with an economically distinct cost structure that captures the asymmetry between human skill acquisition and AI capability scaling; (iii) we introduce a fixed training cost, amortized over the number of deployments, together with a marginal operating cost that is weakly independent of AI capability, reflecting the economics of machine learning deployment; and (iv) we embed risk-adjusted effective costs at the level of individual agent primitives, so that all substitution results are stated on risk-adjusted grounds throughout. A preliminary version of this model appeared in our earlier work \citep{banerjee2023model}.


\begin{tcolorbox}[colback=gray!5, colframe=gray!60, title=\textbf{Notation Conventions}, fonttitle=\small, boxrule=0.5pt]
\small
\begin{itemize}[leftmargin=*, nosep]
    \item \textbf{Subscripts:} $i$ indexes hierarchical level ($1 = $ top, $D+1 = $ line workers); $j$ indexes a human agent within a level; $k$ indexes an AI agent within a level; $\ell$ is a summation index over levels.
    \item \textbf{Primes ($'$):} Distinguish AI quantities from their human counterparts (e.g., $C'_{ik}$ is AI cost vs.\ $C_{ij}$ for human cost; $\Delta\$'$ vs.\ $\Delta\$$; $r'_0$ vs.\ $r_0$).
    \item \textbf{Tildes ($\tilde{\phantom{x}}$):} Denote risk-adjusted effective costs (e.g., $\tilde{C}_{ij} = C_{ij} + \lambda R_{ij}$).
    \item \textbf{Calligraphic letters:} $\mathcal{H}$ = set of available human agents; $\mathcal{A}$ = set of available AI agents; $\mathcal{C}$ = total cost; $\mathcal{Q}$ = feasible allocation simplex.
    \item \textbf{Tasks:} $\mathcal{T}$ = finite set of modeled organizational tasks, indexed by $\tau$. Each task has its own feasible agents, risk-adjusted agent costs $\tilde{C}_{ij}(\tau)$ and $\tilde{C}'_{ik}(\tau)$, and allocation simplex; for separable tasks, total cost is the sum of task-level costs. Without a task argument, a single task is considered.
\end{itemize}
\end{tcolorbox}

\subsection{Organizational Structure}
\label{sec:Organizational Structure}

Let $D$ denote the depth of the organization, defined as the number of managerial layers. Level $i = 1$
corresponds to the top of the hierarchy (e.g., the CEO), and level $i = D$ corresponds to the lowest
managerial layer. Beneath these $D$ levels is a layer of line workers indexed as level $D+1$, who perform execution tasks and do not supervise others.

Each manager at level $i \in \{1, \dots, D\}$ supervises a span of $s_i \in \{1, 2, \dots\}$ subordinates at level $i+1$. The number of agents at level $i$ is therefore determined by the product of spans above it:
\begin{equation}
\label{eq:level_size}
e_i = \prod_{\ell=1}^{i-1} s_\ell,
\quad i = 1, \dots, D+1,
\end{equation}
with $e_1 = 1$ by convention. This span-of-control derivation follows \citet{garicano2000} and
\citet{caliendo2020}, and ensures that organizational depth and width are jointly determined by the spans $\{s_i\}$ rather than specified independently. Because all level sizes flow from~\eqref{eq:level_size}, any change in a single span $s_i$ cascades through the entire hierarchy below level $i$---a structural amplification that will drive the organizational flattening results in Theorem~\ref{thm:Optimal Depth under AI Availability} and Corollary~\ref{cor:Flattening of Organizations}. Each position in the hierarchy can be filled by either a human agent or an AI agent. We denote the sets of available human and AI agents by $\mathcal{H}$ and $\mathcal{A}$, respectively. Human agents at level $i$ are indexed by $j = 1,\dots,e_i$ and AI agents by $k = 1,\dots,e'_i$, where $e'_i$ is the number of AI agents available at level $i$ (with $e'_i = e_i$ when every position can also be filled by an AI agent). When no feasible AI agent is available at level $i$, $e'_i = 0$ and the corresponding summation over AI agents is empty; the minimum risk-adjusted cost over an empty set of feasible agents is interpreted as $+\infty$. The human and AI agents so indexed are candidate alternatives for task assignment, rather than necessarily distinct, simultaneously occupied positions; the allocation weights introduced below determine which candidates participate in accomplishing the modeled task.

\paragraph{Empirical context.}
Empirical work using matched employer--employee data \citep{caliendo2020} motivates the representation of organizational layers and spans; the product-of-spans expression~\eqref{eq:level_size} is the maintained structural specification of the HAT model. The assumption that positions can be automated or augmented by AI is consistent with the automation--augmentation tension emphasized in the management literature \citep{raisch2021}.

\subsection{Human Cost Structure}

Let $\mu_{i,j} \geq 0$ denote the skill level of human agent $j$ deployed at organizational level $i$, denoted as human agent $(i,j)$ for brevity. Human expertise is costly to acquire: it requires scarce talent, education, accumulated experience, and sustained effort over long periods of time \citep{autor2003, acemoglu2011}. Consequently, human compensation increases with both skill level
and organizational depth. The wage and benefit cost per unit time of a human line worker $(D+1, j)$ is
\begin{equation}
\label{eq:line_worker_cost}
C_{D+1,j} = \textup{Min\$} + \frac{1}{2}\,\Delta\$
\cdot D \cdot \mu_{D+1,j},
\end{equation}
where $\textup{Min\$} > 0$ is the baseline compensation floor, $\Delta\$ > 0$ is the marginal cost of skill, and $D$ captures the empirical regularity that deeper organizations demand higher capability at every level and anchor their entire wage structure to greater complexity \citep{caliendo2020}.

The cost per unit time of a human manager at level $i$ ($i = 1, \dots, D$) is
\begin{equation}
\label{eq:manager_cost}
C_{i} = C_{D+1} \cdot (1 + r)^{D-i+1},
\end{equation}
where $r = r_0 \cdot D$ with $r_0 > 0$, and $C_{D+1}$ denotes the representative line-worker cost. The representative line-worker cost $C_{D+1}$ is a reference cost parameter, namely the cost~\eqref{eq:line_worker_cost} of a line worker at a given reference skill level; managerial costs are built from this reference cost rather than from the cost of any particular agent. The wage ratio between adjacent levels thus increases with organizational depth, capturing the escalating cost of higher-level managerial roles \citep{gabaix2008, rosen1982}. The exponential escalation in~\eqref{eq:manager_cost} is the primary channel through which organizational depth amplifies total cost; the comparative statics of depth reduction in Theorem~\ref{thm:Optimal Depth under AI Availability} exploit this structure directly. Note that managerial skill, \(\mu_{i,j}\) for $i\le D$, is not modeled explicitly; instead, managerial capability requirements are captured in reduced form through the level-dependent cost \(C_i\). Accordingly, $C_{ij} = C_i$ for every human agent $j$ at a managerial level $i \le D$, while $C_{D+1,j}$ is given by~\eqref{eq:line_worker_cost}.

\paragraph{Risk-adjusted human effective cost.}
Human agents carry operational risks including errors in judgment, absenteeism, and turnover. Let $R_{ij} \geq 0$ denote the risk associated with human agent $(i,j)$, and let $\lambda \geq 0$ be an organizational risk-sensitivity parameter that converts risk into the same cost units as $C_{ij}$. We define the \emph{risk-adjusted effective cost} of human
agent $(i,j)$ as
\begin{equation}
\label{eq:human_effective_cost}
\tilde{C}_{ij} = C_{ij} + \lambda R_{ij}.
\end{equation}
By embedding risk at the level of individual agent costs, every substitution decision in the model is
automatically made on risk-adjusted grounds.

\paragraph{Empirical context.}
The baseline floor $\textup{Min\$}$ is consistent with \citet{caliendo2020} and \citet{gabaix2008}; the linear skill--compensation link reflects the Mincer earnings equation \citep{mincer1974}; and the managerial escalation factor $(1+r_0 D)^{D-i+1}$ captures the superstar economics of managerial talent \citep{rosen1982, gabaix2008, caliendo2020}. The risk term \(\lambda R_{ij}\) captures idiosyncratic human operational risks such as error, absenteeism, turnover, and compliance failures.

\subsection{AI Cost Structure}

We now extend the model to incorporate AI agents. Unlike human agents, AI systems acquire capability
through computational training processes rather than through individual talent development. This creates a fundamentally different relationship between capability and cost.

Let $\mu'_{i,k} \geq 0$ denote the capability level of AI agent $k$ deployed at organizational level $i$. Just as human skill demands increase with organizational depth, AI capability requirements also
vary with the hierarchical level at which the agent is deployed: more complex coordination and
decision-making tasks at higher levels require more capable AI systems. For line workers ($i = D+1$) we
write $\mu'_{D+1,k}$.

Deploying an AI agent involves two economically distinct cost components:
\begin{itemize}
\item\textbf{Fixed training cost.}
A fixed training expenditure is incurred once to develop and train the system to capability level $\mu'_{D+1,k}$. This covers computational resources, data acquisition, model development, and validation. To keep all cost components on a common per-unit-time basis, $T_k > 0$ denotes this expenditure converted to its equivalent cost per unit time over the accounting horizon used for all cost components (for instance, by spreading it evenly over that horizon). When amortized over $n_k \geq 1$ deployments, the per-deployment contribution is $T_k / n_k$.\\
\item\textbf{Marginal operational cost.}
Once trained, an AI system incurs a marginal operational cost $M'_k \geq 0$ per unit time of
deployment (inference cost, maintenance, API fees, etc.). Critically, $M'_k$ is \emph{weakly independent of capability level} $\mu'_{D+1,k}$: unlike humans, whose compensation rises with skill, a more capable AI model does not necessarily cost more to run per unit time once trained. This is the fundamental economic asymmetry that drives substitution dynamics in the model.
\end{itemize}

The total effective cost per unit time of an AI line worker $(D+1,k)$ before risk adjustment is
therefore
\begin{equation}
\label{eq:AI_lineworker_cost}
C'_{D+1,k} = \frac{T_k}{n_k} + M'_k
+ \frac{1}{2}\,\Delta\$' \cdot D \cdot \mu'_{D+1,k},
\end{equation}
where $\textup{Min\$}'_k = T_k/n_k + M'_k \geq 0$ is the baseline cost of AI agent $k$ and $\Delta\$' \geq 0$ is the
marginal rate at which operational costs increase with AI capability. For the aggregate skill decomposition below and Theorem~\ref{thm:Skill-Weighted Cost Decomposition}, we take the baseline to be common across the AI agents at the line-worker level, $\textup{Min\$}'_k = \textup{Min\$}'$ for all $k$, so that these agents differ in capability $\mu'_{D+1,k}$ and risk $R'_{D+1,k}$ rather than in baseline cost (with heterogeneous baselines, $\textup{Min\$}' \sum_k q'_{D+1,k}$ in that decomposition is replaced by $\sum_k \textup{Min\$}'_k q'_{D+1,k}$). Elsewhere, including Assumption~\ref{ass:asymmetry}(ii) and the results about a single AI agent $k$ (Theorem~\ref{thm:Skill-Adjusted Substitution Condition} and Corollary~\ref{cor:Skill-Dependent Human Protection and Vulnerability Zone}), $\textup{Min\$}'$ stands for that agent's own baseline $\textup{Min\$}'_k$, so that statements about deployment scale concern each agent's own $n_k$ and $M'_k$ and do not require the limiting operating costs $M'_k$ to be equal across agents. Equation~\eqref{eq:AI_lineworker_cost} can then be written compactly as
\begin{equation}
\label{eq:AI_lineworker_parametric}
C'_{D+1,k} = \textup{Min\$}' + \frac{1}{2}\,\Delta\$'
\cdot D \cdot \mu'_{D+1,k}.
\end{equation}
The explicit capability-cost slope in this specification is interpreted conditional on the baseline cost parameters; it does not include changes in training expenditure that may be required to develop a different AI system. Comparing~\eqref{eq:AI_lineworker_parametric} with~\eqref{eq:line_worker_cost} makes the human--AI cost asymmetry transparent: both are affine in their respective capability measures, but the slopes $\Delta\$'$ and $\Delta\$$ may differ---a gap that Assumption~\ref{ass:asymmetry} formalizes and that Theorem~\ref{thm:Skill-Adjusted Substitution Condition} translates into substitution predictions. Figure~\ref{fig:hat-mechanism-a} illustrates how this slope asymmetry generates a threshold separating human protection from AI substitution.

The cost of an AI agent at managerial level $i$ is
\begin{equation}
\label{eq:AI_manager_cost}
C'_i = C'_{D+1} \cdot (1 + r')^{D-i+1},
\end{equation}
where $r' = r'_0 \cdot D$ with $r'_0 \geq 0$, and $C'_{D+1}$ denotes the representative AI line-worker
cost. The representative AI line-worker cost $C'_{D+1}$ is a reference cost parameter that remains defined even when no feasible AI agent is available at the line-worker level. We allow $r'_0 \leq r_0$, reflecting the possibility that AI agents face lower coordination-cost escalation than humans at higher hierarchical levels \citep{bloom2014}. As in the human case, $C'_{ik} = C'_i$ for every AI agent $k$ at a managerial level $i \le D$, while $C'_{D+1,k}$ is given by~\eqref{eq:AI_lineworker_parametric}. Capability requirements at managerial levels ($\mu'_{i,k}$ for $i \le D$) are not priced separately: as for human managers, they are absorbed into the reduced-form managerial cost $C'_i$, so the explicit capability-cost term appears only at the line-worker level.

\begin{figure}[htbp!]
  \centering
    \includegraphics[width=0.92\textwidth]{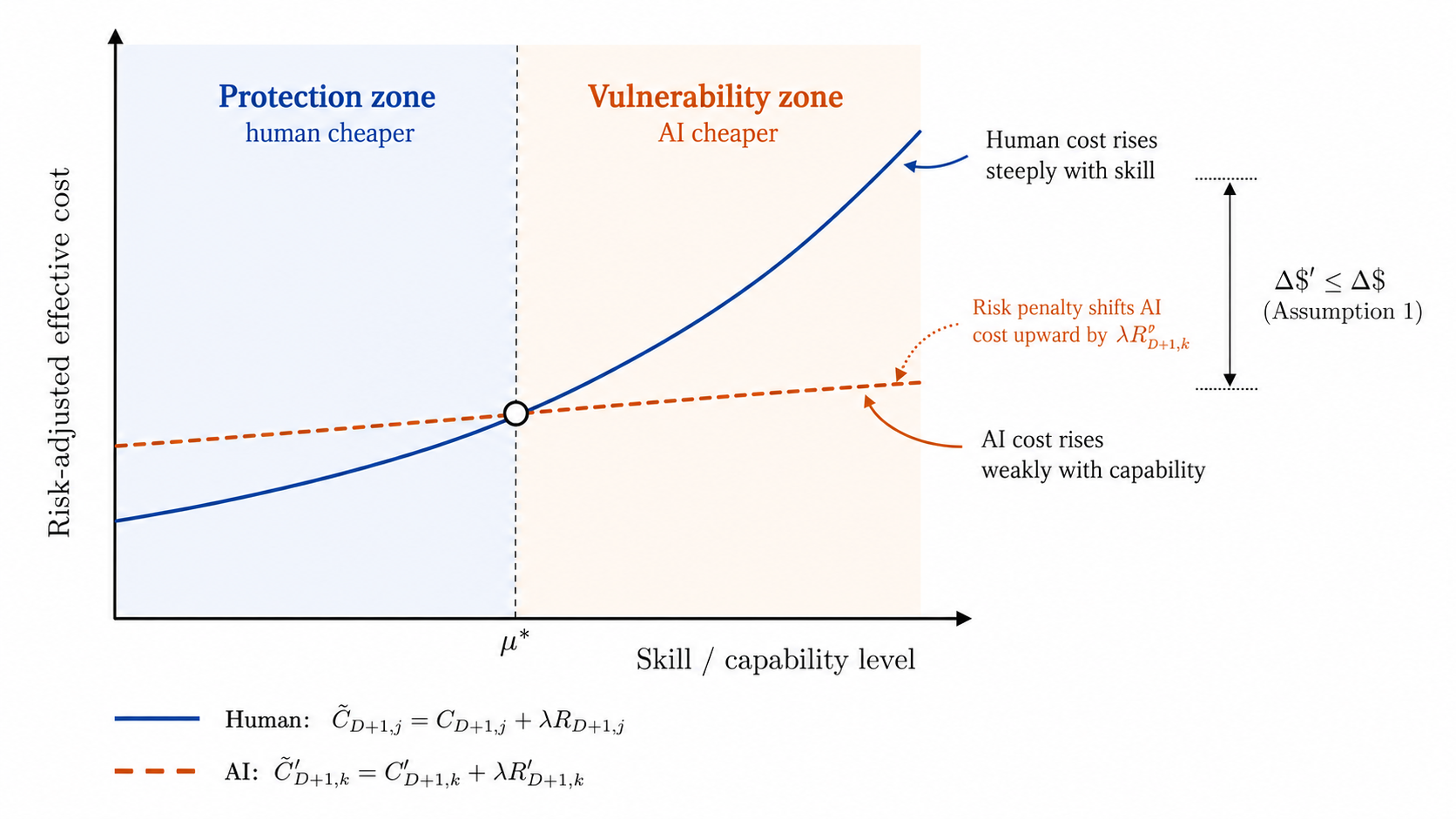}
    \caption{Risk-adjusted effective cost as a function of human skill or AI capability for a representative human line worker \((D+1,j)\) and AI agent \((D+1,k)\). The human curve represents \(\widetilde{C}_{D+1,j}=C_{D+1,j}+\lambda R_{D+1,j}\), while the AI curve represents \(\widetilde{C}'_{D+1,k}=C'_{D+1,k}+\lambda R'_{D+1,k}\). Human cost rises steeply with skill because expertise requires scarce talent, training, and accumulated experience, whereas AI cost rises more moderately with capability because a trained model can be deployed at relatively stable marginal cost. The risk penalty \(\lambda R'_{D+1,k}\) shifts the AI cost curve upward. The curves illustrate how differences in human skill costs and AI capability costs can generate substitution thresholds. The threshold \(\mu^*\) characterized in Corollary~\ref{cor:Skill-Dependent Human Protection and Vulnerability Zone} is obtained by varying human skill while holding the AI agent's capability, baseline cost, and both agents' risk terms fixed. For \(\mu_{D+1,j}<\mu^*\), the human agent's risk-adjusted cost is below that of the AI agent \((D+1,k)\); for \(\mu_{D+1,j}>\mu^*\), it is above.}
    \label{fig:hat-mechanism-a}
\end{figure}

\paragraph{Risk-adjusted AI effective cost.}
AI systems carry distinct operational risks. Following the AI governance literature \citep{agrawal2019,
dellacqua2026}, we decompose AI risk into three components: reliability risk $R^{\mathrm{(rel)}}_{ik}$ (errors, instability, distributional shift), compliance risk $R^{\mathrm{(comp)}}_{ik}$ (regulatory and legal exposure), and reputational risk $R^{\mathrm{(rep)}}_{ik}$ (public backlash, loss of trust). The aggregate risk of AI agent $(i,k)$ is
\begin{equation}
\label{eq:AI_risk}
R'_{ik} = \omega_1 R^{\mathrm{(rel)}}_{ik}
+ \omega_2 R^{\mathrm{(comp)}}_{ik}
+ \omega_3 R^{\mathrm{(rep)}}_{ik},
\end{equation}
where $\omega_1, \omega_2, \omega_3 \geq 0$ are organizational weights reflecting the risk profile
of the task domain (see Figure~\ref{fig:hat-mechanism-c}). The relative magnitudes of these weights vary substantially across industries: in healthcare or pharmaceutical settings, compliance weight $\omega_2$ may dominate due to FDA and regulatory exposure; in consumer-facing customer service, reputational risk weight $\omega_3$ may be paramount; while in autonomous systems or manufacturing, reliability weight $\omega_1$ takes precedence.

\begin{figure}[tb!]
  \centering
    \includegraphics[width=0.7\textwidth,height=0.15\textwidth]{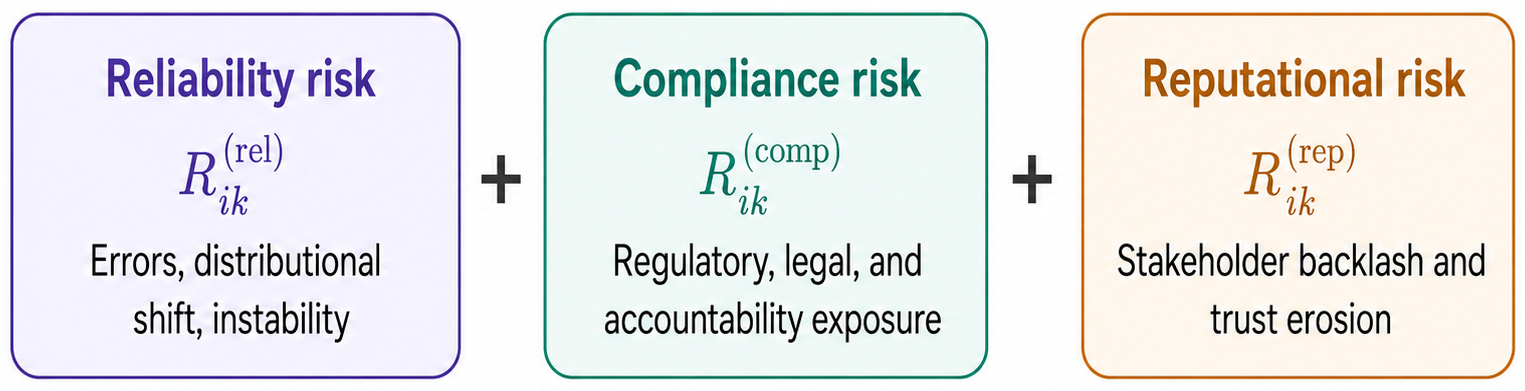}
    \caption{Three-component decomposition of AI risk. The aggregate risk for AI agent \((i,k)\) is
$R'_{ik} = \omega_1 R^{(\mathrm{rel})}_{ik} + \omega_2 R^{(\mathrm{comp})}_{ik} + \omega_3 R^{(\mathrm{rep})}_{ik}$. Each component is weighted by domain-specific organizational sensitivity parameters \((\omega_1,\omega_2,\omega_3)\), and the aggregate risk enters the risk-adjusted AI cost through $\widetilde C'_{ik}=C'_{ik}+\lambda R'_{ik}$. In regulated industries such as healthcare and finance, a large compliance weight \(\omega_2\) can make compliance risk dominate the AI risk term, thereby offsetting nominal AI cost advantages and, across heterogeneous tasks, helping sustain hybrid human--AI organizations.
}
    \label{fig:hat-mechanism-c}
\end{figure}

The \emph{risk-adjusted effective cost} of AI agent $(i,k)$ is
\begin{equation}
\label{eq:AI_effective_cost}
\tilde{C}'_{ik} = C'_{ik} + \lambda R'_{ik},
\end{equation}
using the same risk-sensitivity parameter $\lambda$ as in~\eqref{eq:human_effective_cost}. This risk-adjusted cost structure forms the basis for Theorem~\ref{thm:Constrained Interior Optima}: a sufficiently large AI risk penalty $\lambda R'_{ik}$ can eliminate the nominal cost advantage of an AI agent relative to a human alternative; across heterogeneous tasks this can sustain hybrid organizations without an exogenous minimum-human-fraction constraint.

\paragraph{Empirical context.}
The decomposition into fixed training cost $T_k$ and marginal operational cost $M'_k$ mirrors the economics of digital goods \citep{goldfarb2019}. \citet{hoffmann2022} show that model performance depends systematically on training compute, model size, and dataset size, and that compute-optimal scaling can improve downstream performance without simply increasing model size. The three-component risk decomposition follows the AI governance literature \citep{floridi2018, doshi2017}: reliability risk corresponds to distributional shift failures \citep{dellacqua2026}; compliance risk corresponds to regulatory exposure (EU AI Act, sector-specific FDA/SEC/FINRA rules); and reputational risk corresponds to documented algorithm aversion \citep{dietvorst2015}. The domain-specificity of $(\omega_1,\omega_2,\omega_3)$ is supported by \citet{lee2004}, who show that trust in automation is calibrated differently across high- and low-stakes domains.

\subsection{The Human--AI Cost Asymmetry}

The key structural difference between human and AI costs is formalized in the following assumption, which we maintain throughout the paper and which gives substantive content to all substitution results.

\begin{styledassumption}{Human--AI Cost Asymmetry}{asymmetry}
\hfill
\begin{enumerate}[label=\roman*.] 
    \item Holding the respective baseline costs fixed, the marginal capability-related cost of AI agents grows no faster than the marginal skill-related cost of human agents at the line-worker level:
    \begin{equation}
    \label{eq:asymmetry_delta}
    \Delta\$' \;\leq\; \Delta\$.
    \end{equation}
    \item As the number of deployments $n_k \to \infty$, the amortized training cost $T_k/n_k \to 0$, so that the AI baseline cost $\textup{Min\$}'$ approaches the marginal operational cost $M'_k$, which is independent of capability level $\mu'_{D+1,k}$.
    \item The AI coordination cost multiplier satisfies $r'_0 \leq r_0$, so that managerial cost escalation with depth is no steeper for AI agents than for humans.
\end{enumerate}
\end{styledassumption}

Skill and capability enter the cost functions explicitly only at the line-worker level, where $\Delta\$$ and $\Delta\$'$ are defined (equations~\eqref{eq:line_worker_cost} and~\eqref{eq:AI_lineworker_parametric}); managerial costs are captured in reduced form through $C_i$ and $C'_i$. Part~(i) of Assumption~\ref{ass:asymmetry} is therefore operative at the line-worker level, and part~(iii) governs the escalation of managerial costs. The comparison in part~(i) presupposes that human skill and AI capability are measured on comparable scales for the task under consideration.

\paragraph{Empirical grounding of Assumption~\ref{ass:asymmetry}.}
Assumption~\ref{ass:asymmetry} is a maintained modeling premise motivated by empirical patterns in human skill costs, digital deployment economics, and information-technology-enabled coordination. We ground each part in evidence while recognizing that the inequalities may vary across organizational contexts.

\textbf{Part (i): $\Delta\$' \le \Delta\$$.}
On the human side, the convexity of earnings in skill is well established: \citet{lemieux2006} shows increasingly convex returns at the top of the skill distribution; the Mincer equation \citep{mincer1974}, replicated across decades and countries, yields 8--12\% returns per year of schooling; and \citet{autor2003} document a near-doubling of the college wage premium between 1963 and 1987. On the AI side, \citet{hoffmann2022} show that downstream performance improves predictably with training compute, model size, and dataset size, and that better performance can be obtained through compute-optimal scaling rather than simply increasing model size. Inference-price data corroborate this pattern: LLM inference prices have fallen rapidly, though unevenly across tasks, while model capabilities have continued to improve \citep{epoch2025inference}. This motivates treating \(\Delta\$'\) as small relative to \(\Delta\$\), rather than assuming that it is literally zero.

\textbf{Part (ii): \(T_k/n_k \to 0\) and \(\textup{Min\$}' \to M'_k\) as \(n_k \to \infty\).}
This is a structural feature of software economics \citep{goldfarb2019}. Training costs for frontier models can reach tens of millions of dollars and have grown rapidly for the largest models \citep{cottier2024rising}, while marginal inference costs are captured separately by inference-price data \citep{epoch2025inference}. As models are deployed across thousands of enterprise use cases, the per-deployment training contribution approaches zero, leaving $M'_k$ as the binding cost floor.

\textbf{Part (iii): $r'_0 \le r_0$.}
Human coordination costs escalate with hierarchy depth due to communication overhead, information loss, and agency costs \citep{williamson1967, radner1992}. AI agents may mitigate some of these, since they can process large volumes of structured data without human bottlenecks. \citet{bloom2014} provide causal evidence that IT improvements reduce within-firm coordination costs and promote decentralization, consistent with $r'_0 < r_0$. For AI specifically, the assumption \(r'_0\le r_0\) should be interpreted as a maintained modeling premise motivated by the possibility that digital agents can process and transmit structured information with fewer human communication bottlenecks; whether the inequality holds in a given organization is an empirical question.

Taken together, Assumption~\ref{ass:asymmetry} is conservative: it requires only $\Delta\$' \le \Delta\$$, not $\Delta\$' = 0$, nesting the symmetric case as a special case. All substitution results in Section~\ref{sec:Properties of this model} hold under this weak inequality.

\subsection{Task Cost Aggregation and Joint Human--AI Allocation}

The preceding subsections defined cost and risk structures for individual human and AI agents at each hierarchical level. We now aggregate these into a joint allocation problem. The motivation is direct: an organization does not choose human and AI agents in isolation---it allocates a given task across both agent types simultaneously, subject to a feasibility constraint. The joint framework enables us to characterize optimal allocations (Theorem~\ref{thm:Global Optimal Allocation}), identify conditions under which hybrid organizations are optimal (Theorem~\ref{thm:Constrained Interior Optima}), and derive comparative statics on organizational depth (Theorem~\ref{thm:Optimal Depth under AI Availability}).

Let \(q_{ij}\ge 0\) denote the fraction of a task accomplished by a human agent \((i,j)\), and \(q'_{ik}\ge 0\) denote the fraction of the same task accomplished by an AI agent \((i,k)\). For a given task, we assume that effective costs are expressed on a common task-execution and accounting basis and that each agent's cost contribution is proportional to its assigned fraction of the task. The total risk-adjusted cost per unit time is therefore the sum of the allocation-weighted risk-adjusted costs of the participating human and AI agents. Mathematically,
\begin{equation}
\label{eq:joint_cost}
\mathcal{C} = C^{(\text{humans})} + C^{(\text{AI})}
\end{equation}
\begin{equation}
\label{eq:joint_cost_expanded}
= \sum_{i=1}^{D+1} \sum_{j=1}^{e_i} \tilde{C}_{ij} \cdot q_{ij}
+ \sum_{i=1}^{D+1} \sum_{k=1}^{e'_i} \tilde{C}'_{ik} \cdot q'_{ik}
= \sum_{i=1}^{D+1} \left(
\sum_{j=1}^{e_i} \tilde{C}_{ij} \cdot q_{ij}
+ \sum_{k=1}^{e'_i} \tilde{C}'_{ik} \cdot q'_{ik}
\right),
\end{equation}
where
\begin{equation}
\label{eq:feasibility}
q_{ij},\; q'_{ik} \geq 0, \qquad
\sum_{i}\sum_{j} q_{ij} + \sum_{i}\sum_{k} q'_{ik}
= \sum_{i}\!\left(\sum_{j} q_{ij} + \sum_{k} q'_{ik}\right)
= 1.
\end{equation}
The feasible allocation space is
\begin{equation}
\label{eq:simplex}
\mathcal{Q} = \left\{(q,q') \;:\;
q_{ij}, q'_{ik} \geq 0,\;
\sum_{i,j} q_{ij} + \sum_{i,k} q'_{ik} = 1
\right\},
\end{equation}
which is a standard simplex (Theorem~\ref{thm:Normalization, Feasibility, and Simplex Structure}). A human-only task allocation corresponds to $q'_{ik} = 0$ for all $(i,k)$; an AI-only task allocation to $q_{ij} = 0$ for all $(i,j)$; and a hybrid task allocation to at least one strictly positive human allocation weight and at least one strictly positive AI allocation weight. We describe an organization as human-only, AI-only or hybrid according to the type of task allocation it implements.

\subsection{Skill and Cost Decomposition at the Operational Level}

At the line-worker level ($i = D+1$), the cost expression admits a decomposition that separates the structural (hierarchical) component from the skill-dependent (operational) component and makes the human--AI asymmetry explicit:
\begin{align}
\label{eq:skill_decomposition}
&\sum_{j=1}^{e_{D+1}} \tilde{C}_{D+1,j}\cdot q_{D+1,j}
+ \sum_{k=1}^{e'_{D+1}} \tilde{C}'_{D+1,k}\cdot q'_{D+1,k}\nonumber\\
&=
\underbrace{\textup{Min\$} \sum_{j} q_{D+1,j}
+ \frac{1}{2}\Delta\$ \cdot D\sum_{j}(\mu_{D+1,j} \cdot q_{D+1,j})
+ \lambda\sum_{j}(R_{D+1,j} \cdot q_{D+1,j})
}_{\text{human operational cost: skill-increasing, risk-adjusted}}\nonumber\\
&\quad+
\underbrace{\textup{Min\$}' \sum_{k} q'_{D+1,k}
+ \frac{1}{2}\Delta\$' \cdot D\sum_{k}(\mu'_{D+1,k} \cdot q'_{D+1,k})
+ \lambda\sum_{k}(R'_{D+1,k} \cdot q'_{D+1,k})
}_{\text{AI operational cost: weakly capability-increasing, risk-adjusted}}\nonumber\\
&= \alpha
+ \beta \cdot D\!\sum_{j}(\mu_{D+1,j} \cdot q_{D+1,j})
+ \alpha'
+ \beta' \cdot D\!\sum_{k}(\mu'_{D+1,k} \cdot q'_{D+1,k})\nonumber\\
&\quad+ \lambda\!\left[\sum_{j}(R_{D+1,j} \cdot q_{D+1,j})
+ \sum_{k}(R'_{D+1,k} \cdot q'_{D+1,k})\right],
\end{align}
where
\begin{equation}
\alpha = \textup{Min\$}\sum_{j} q_{D+1,j},
\quad \beta = \tfrac{1}{2}\Delta\$,
\quad \alpha' = \textup{Min\$}'\sum_{k} q'_{D+1,k},
\quad \beta' = \tfrac{1}{2}\Delta\$'.
\end{equation}

The decomposition~\eqref{eq:skill_decomposition} makes four structural features explicit.

\textbf{First}, skill heterogeneity enters costs exclusively through the line-worker layer: managerial
costs at levels $i \leq D$ depend only on organizational depth and the representative base cost $C_{D+1}$, not directly on individual skill levels. Substitution dynamics driven by skill differences therefore originate at the bottom of the hierarchy and, to the extent that changes in operational skills affect the representative line-worker costs, propagate upward through cost escalation (Theorem~\ref{thm:Skill-Weighted Cost Decomposition}).

\textbf{Second}, under Assumption~\ref{ass:asymmetry}(i), \(\beta'\le \beta\), so the AI capability-cost term grows no faster than the human skill-cost term at comparable levels. If \(\beta'<\beta\), the human skill-cost term grows strictly faster.

\textbf{Third}, under Assumption~\ref{ass:asymmetry}(ii), as $n_k \to \infty$ for every $k$, the baseline term $\alpha'$ (which, with agent-specific baselines, is $\sum_k \textup{Min\$}'_k q'_{D+1,k}$) tends, for fixed allocation weights, to $\sum_k M'_k q'_{D+1,k}$, a capability-independent floor, while the human cost keeps its skill-dependent term $\beta \cdot D\sum_{j}(\mu_{D+1,j} \cdot q_{D+1,j})$.

\textbf{Fourth}, risk enters additively for both humans and AI agents, but with structurally different
compositions: human risk $R_{D+1,j}$ reflects idiosyncratic behavioral risks, while AI risk $R'_{D+1,k}$ is governed by the domain-specific weights $(\omega_1, \omega_2, \omega_3)$ in
equation~\eqref{eq:AI_risk}. In regulated environments where $\omega_2$ is large, the AI risk term $\lambda R'_{D+1,k}$ may offset AI's nominal cost advantage and thereby change which agent is preferred for a given task; across heterogeneous tasks this can sustain hybrid human--AI organizations without an exogenous minimum-human-fraction constraint (Theorem~\ref{thm:Constrained Interior Optima}).

Together, these four features establish the analytical machinery for the main results: the pairwise risk-adjusted comparison determines when substitution reduces cost (Theorem~\ref{thm:Human--AI Substitution Principle}) and the slope gap $\beta' \leq \beta$ shapes how that comparison varies with skill and capability (Theorem~\ref{thm:Skill-Adjusted Substitution Condition}), risk-adjusted cost heterogeneity and additional feasibility constraints can produce constrained hybrid allocations (Theorem~\ref{thm:Constrained Interior Optima}), and the span-of-control structure~\eqref{eq:level_size} transmits these agent-level dynamics into organizational-level predictions about depth and shape (Theorem~\ref{thm:Optimal Depth under AI Availability}).

\subsection{Summary of Model Parameters}

Table~\ref{tab:parameters} summarizes the parameters of the HAT model together with their interpretations and corresponding observable organizational constructs.


\begin{table}[htbp!]
\centering
\caption{Summary of HAT model parameters and organizational interpretation.}
\label{tab:parameters}
\footnotesize 
\setlength{\tabcolsep}{4pt}
\renewcommand{\arraystretch}{1.3}
\begin{tabular}{@{}>{\raggedright\arraybackslash}p{1.7cm}>{\raggedright\arraybackslash}p{1.4cm}>{\raggedright\arraybackslash}p{4.2cm}>{\raggedright\arraybackslash}p{5.8cm}@{}}
\toprule
\textbf{Parameter} &
\textbf{Domain} &
\textbf{Model interpretation} &
\textbf{Observable organizational construct} \\
\midrule

\multicolumn{4}{@{}l}{\cellcolor{blue!15}\textbf{\textit{Organizational Structure}}} \\
\rowcolor{blue!3}
$D$ &
$\mathbb{Z}_{+}$ &
Organizational depth &
Number of managerial layers; reporting hierarchy \\

\rowcolor{blue!6}
$s_i$ &
$\mathbb{Z}_{\ge 1}$ &
Span of control at level $i$ &
Average number of direct reports per manager \\

\rowcolor{blue!3}
$e_i$ &
$\mathbb{Z}_{+}$ &
Number of agents at level $i$ &
Organizational headcount by hierarchical level \\

\rowcolor{blue!6}
$e'_i$ &
$\mathbb{Z}_{\ge 0}$ &
Number of AI agents available at level $i$ &
Deployable AI agents by hierarchical level \\

\midrule[0.5pt]

\multicolumn{4}{@{}l}{\cellcolor{green!15}\textbf{\textit{Human Agent Parameters}}} \\
\rowcolor{green!3}
$\mu_{i,j}$ &
$\mathbb{R}_{+}$ &
Human skill level &
Employee experience, expertise, certifications, productivity, performance evaluations \\

\rowcolor{green!6}
$\textup{Min\$}$ &
$\mathbb{R}_{+}$ &
Baseline human compensation &
Salary floor, wages, employee benefits, hiring cost \\

\rowcolor{green!3}
$\Delta\$$ &
$\mathbb{R}_{+}$ &
Marginal cost of human skill &
Incremental compensation associated with higher expertise, training, or experience \\

\rowcolor{green!6}
$r_0$ &
$\mathbb{R}_{+}$ &
Human coordination cost multiplier &
Communication overhead, managerial supervision, coordination burden \\

\midrule[0.5pt]

\multicolumn{4}{@{}l}{\cellcolor{purple!15}\textbf{\textit{AI Agent Parameters}}} \\
\rowcolor{purple!3}
$\mu'_{i,k}$ &
$\mathbb{R}_{+}$ &
AI capability level &
Model accuracy, benchmark performance, reasoning capability, task success rate \\

\rowcolor{purple!6}
$T_k$ &
$\mathbb{R}_{+}$ &
Fixed AI training cost (per-unit-time equivalent) &
Model development, data collection, training, fine-tuning, deployment preparation \\

\rowcolor{purple!3}
$n_k$ &
$\mathbb{Z}_{+}$ &
Deployment scale &
Number of users, business units, or organizational deployments over which AI development cost is amortized \\

\rowcolor{purple!6}
$M'_k$ &
$\mathbb{R}_{\ge0}$ &
Marginal AI operational cost &
Inference cost, cloud compute, licensing, maintenance cost \\

\rowcolor{purple!3}
$\textup{Min\$}'_k$ &
$\mathbb{R}_{\ge0}$ &
Baseline AI cost of agent $k$ ($T_k/n_k+M'_k$) &
Amortized development plus operational cost of deploying an AI agent \\

\rowcolor{purple!6}
$\Delta\$'$ &
$\mathbb{R}_{\ge0}$ &
Marginal AI capability cost &
Incremental computational or engineering cost associated with higher AI capability, holding baseline cost fixed \\

\rowcolor{purple!3}
$r'_0$ &
$\mathbb{R}_{\ge0}$ &
AI coordination cost multiplier &
Integration effort, workflow orchestration, monitoring, human oversight requirements \\

\midrule[0.5pt]

\multicolumn{4}{@{}l}{\cellcolor{red!15}\textbf{\textit{Risk Parameters}}} \\
\rowcolor{red!3}
$\lambda$ &
$\mathbb{R}_{\ge0}$ &
Organizational risk sensitivity &
Firm's tolerance for operational, regulatory, or strategic risk \\

\rowcolor{red!6}
$R_{ij}$ &
$\mathbb{R}_{\ge0}$ &
Human risk &
Human error, turnover, absenteeism, compliance violations, operational failures \\

\rowcolor{red!3}
$R'_{ik}$ &
$\mathbb{R}_{\ge0}$ &
AI risk &
Reliability failures, hallucinations, cybersecurity, compliance, reputational risk \\

\rowcolor{red!6}
$\omega_1,\omega_2,\omega_3$ &
$\mathbb{R}_{\ge0}$ &
AI risk weights &
Relative organizational importance of reliability, regulatory compliance, and reputational concerns \\

\midrule[0.5pt]

\multicolumn{4}{@{}l}{\cellcolor{orange!15}\textbf{\textit{Decision Variables}}} \\
\rowcolor{orange!3}
$q_{ij}$ &
$[0,1]$ &
Human task allocation weight &
Fraction of a given task assigned to a human agent \\

\rowcolor{orange!6}
$q'_{ik}$ &
$[0,1]$ &
AI task allocation weight &
Fraction of the same task assigned to an AI agent \\

\midrule[0.5pt]

\multicolumn{4}{@{}l}{\cellcolor{teal!15}\textbf{\textit{Cost Functions}}} \\
\rowcolor{teal!3}
$\tilde{C}_{ij}$ &
$\mathbb{R}_{+}$ &
Risk-adjusted human cost &
Overall organizational cost of assigning a task to a human employee \\

\rowcolor{teal!6}
$\tilde{C}'_{ik}$ &
$\mathbb{R}_{+}$ &
Risk-adjusted AI cost &
Overall organizational cost of assigning a task to an AI agent \\

\bottomrule
\end{tabular}
\end{table}

The effective cost and risk terms used in the HAT framework should be interpreted as \emph{reduced-form organizational quantities} rather than narrow financial measures. Specifically, the effective cost of assigning a task to a human or AI agent may incorporate compensation, supervision, coordination, communication, implementation, and operational costs, while the associated risk may capture reliability, compliance, cybersecurity, governance, and reputational considerations. This reduced-form representation allows the HAT model to remain analytically tractable while encompassing a broad class of organizational mechanisms within a unified optimization framework.

\section{Properties of the HAT Model}
\label{sec:Properties of this model}

We now derive the structural properties of the HAT model, characterizing the geometry of the feasible allocation space, the behavior of risk-adjusted costs, and the implications for human--AI substitution. The analysis proceeds progressively: we establish baseline feasibility and cost structure (\S\ref{sec:Baseline Properties Without AI}), examine how AI availability reshapes organizational design (\S\ref{sec:Impact of AI on Organizational Structure}), identify middle-management vulnerability as a distinctive prediction (\S\ref{sec:Middle-Management Vulnerability}), analyze the role of risk in limiting substitution (\S\ref{sec:Risk, Coordination, and the Limits of Substitution}), and characterize dynamics and robustness (\S\ref{sec:Comparative Statics and Robustness}).

\subsection{Baseline Properties Without AI}
\label{sec:Baseline Properties Without AI}

We begin by establishing the geometric structure of feasible allocations and the fundamental properties of cost decomposition and hierarchical escalation. These baseline results hold regardless of AI availability and provide the foundation for all subsequent substitution analysis.

\subsubsection{Feasibility and Allocation Structure}

The feasible allocation space forms a simplex, enabling sharp characterization of optimal solutions.

\begin{styledtheorem}{Normalization, Feasibility, and Simplex Structure}{Normalization, Feasibility, and Simplex Structure}
\begin{equation}
q_{ij} \ge 0, \quad q'_{ik} \ge 0, \quad
\sum_{i=1}^{D+1}\sum_{j=1}^{e_i} q_{ij}
+ \sum_{i=1}^{D+1}\sum_{k=1}^{e'_i} q'_{ik} = 1.
\end{equation}
The allocation space
\begin{equation}
\mathcal{Q} = \left\{(q,q') : q_{ij} \ge 0,\ q'_{ik} \ge 0,\
\sum_{i,j} q_{ij} + \sum_{i,k} q'_{ik} = 1\right\}
\end{equation}
forms a simplex. Moreover, any valid allocation is a convex combination of extreme points, each corresponding to assigning the entire task to a single agent at a specific level.
\end{styledtheorem}

\noindent
\textbf{Interpretation.} The simplex structure implies that optimal allocations can be characterized by extreme points, where the entire task is assigned to a single agent. This geometric property underpins all subsequent optimization results and explains why concentrated task assignments emerge naturally in optimal designs.

\subsubsection{Cost Structure and Decomposition}

We next decompose total risk-adjusted cost to clarify the distinct roles of managerial structure, line-worker skill, and operational risk.

\begin{styledtheorem}{Skill-Weighted Cost Decomposition}{Skill-Weighted Cost Decomposition}
Under the common AI baseline-cost specification $\textup{Min\$}'_k = \textup{Min\$}'$ for all AI agents $k$ at level $D+1$, the total risk-adjusted cost $\mathcal{C}$ decomposes as
\begin{align*}
\mathcal{C} &=
\sum_{i=1}^{D} \sum_{j=1}^{e_i} \tilde{C}_{ij}\, q_{ij}
+
\sum_{i=1}^{D} \sum_{k=1}^{e'_i} \tilde{C}'_{ik}\, q'_{ik}
\\
&\quad
+ \alpha + \beta \cdot D
  \sum_{j=1}^{e_{D+1}} \bigl(\mu_{D+1,j}\, q_{D+1,j}\bigr)
+ \lambda \sum_{j=1}^{e_{D+1}} \bigl(R_{D+1,j}\, q_{D+1,j}\bigr)
\\
&\quad
+ \alpha' + \beta' \cdot D
  \sum_{k=1}^{e'_{D+1}} \bigl(\mu'_{D+1,k}\, q'_{D+1,k}\bigr)
+ \lambda \sum_{k=1}^{e'_{D+1}} \bigl(R'_{D+1,k}\,
  q'_{D+1,k}\bigr),
\end{align*}
where
\[
\alpha = \textup{Min\$}\sum_{j=1}^{e_{D+1}} q_{D+1,j},
\quad \beta = \tfrac{1}{2}\Delta\$,
\quad \alpha' = \textup{Min\$}'\sum_{k=1}^{e'_{D+1}} q'_{D+1,k},
\quad \beta' = \tfrac{1}{2}\Delta\$'.
\]
Moreover, explicit skill terms enter only through line workers at level \(D+1\):
{\footnotesize
\begin{align*}
\mathcal{C} &=
\underbrace{
\sum_{i=1}^{D}\sum_{j=1}^{e_i} \tilde{C}_{ij}\, q_{ij}
+
\sum_{i=1}^{D}\sum_{k=1}^{e'_i} \tilde{C}'_{ik}\, q'_{ik}
}_{\text{managerial (hierarchical) component}}\\
&+
\underbrace{
\alpha + \beta \cdot D\!
\sum_{j=1}^{e_{D+1}}\!(\mu_{D+1,j}\, q_{D+1,j})
+ \alpha' + \beta' \cdot D\!
\sum_{k=1}^{e'_{D+1}}\!(\mu'_{D+1,k}\, q'_{D+1,k})
+ \lambda\Bigl[\sum_j R_{D+1,j}\, q_{D+1,j}
             + \sum_k R'_{D+1,k}\, q'_{D+1,k}\Bigr]
}_{\text{skill- and risk-dependent (line worker) component}}.
\end{align*}}
\end{styledtheorem}

\noindent
\textbf{Interpretation.} Total risk-adjusted cost separates into managerial costs and line-worker costs. Skill heterogeneity enters explicitly only through the line-worker layer \(D+1\), while risk enters through the risk-adjusted effective costs \(\widetilde C_{ij}\) and \(\widetilde C'_{ik}\). Managerial costs at levels \(i\le D\) do not depend directly on individual skill variables; instead, managerial capability requirements are captured in reduced form through level-dependent costs and risk-adjusted costs. This implies that substitution dynamics driven by skill differences originate at the bottom of the hierarchy and, to the extent that changes in operational skills affect the representative line-worker costs, propagate upward through cost escalation.

\subsubsection{Hierarchy, Depth, and Cost Growth}

We now characterize how organizational depth drives human cost escalation, creating strong incentives for depth reduction when AI substitution becomes feasible.

\begin{styledtheorem}{Depth--Cost Growth and Dominance}{Depth--Cost Growth and Dominance}
Human organizational costs satisfy
\begin{equation}
C_i = C_{D+1}(1 + r_0 D)^{D-i+1}.
\end{equation}
For $r_0 > 0$ and any fixed managerial level $i$, $C_i$ is strictly increasing in $D$ (for $D \ge i$) whenever the line-worker cost $C_{D+1}$ is held fixed or is non-decreasing in $D$, and the factor $(1 + r_0 D)^D$ satisfies
\begin{equation}
\lim_{D \to \infty}
\frac{(1 + r_0 D)^D}{a^D} = \infty
\quad \text{for all } a > 0.
\end{equation}
Moreover, higher managerial levels have higher per-agent costs relative to the line-worker reference cost: for $i < D+1$,
\begin{equation}
\frac{C_i}{C_{D+1}} = (1 + r_0 D)^{D-i+1}.
\end{equation}
If there exists $\epsilon > 0$ such that $\sum_{j=1}^{e_i} q_{ij} \ge \epsilon$ for a managerial level $i \le D$, then
\begin{equation}
\sum_{j=1}^{e_i} C_{ij}\, q_{ij}
\;\ge\; \epsilon\, C_{D+1}(1 + r_0 D)^{D-i+1}.
\end{equation}
\end{styledtheorem}

\noindent
\textbf{Interpretation.} Human hierarchical costs grow super-exponentially with depth, making deep structures inherently expensive. Even small allocations to upper levels create disproportionate cost burdens that act as bottlenecks. The dominance in the title of the theorem refers to the per-agent cost ratio $C_i/C_{D+1}$; the contribution of level $i$ to total cost also depends on the allocation weight that level receives, as the last inequality of the theorem shows. This structural property creates strong incentives to reduce depth or substitute AI agents whose managerial cost multiplier $r'_0 \le r_0$ (Assumption~\ref{ass:asymmetry}(iii)) escalates more slowly---a mechanism formalized in Theorem~\ref{thm:Optimal Depth under AI Availability}.

Having established baseline feasibility, cost decomposition, and hierarchical escalation, we now turn to how AI availability fundamentally reshapes organizational structure.

\subsection{Impact of AI on Organizational Structure}
\label{sec:Impact of AI on Organizational Structure}

Under the stated conditions, AI availability can transform optimal organizational design by enabling flatter hierarchies and concentrated task assignments. We establish the fundamental substitution principle, characterize optimal depth reduction, and derive the conditions under which extreme-point allocations emerge.

\subsubsection{Optimal Depth Reduction under AI Availability}

We first show that, under the stated layer-redundancy condition, AI availability can weakly reduce optimal organizational depth by substituting costly hierarchical coordination.

\begin{styledtheorem}{Optimal Depth under AI Availability}{Optimal Depth under AI Availability}
Let \(\mathcal{C}(D)\) denote the minimum risk-adjusted cost of accomplishing a task in an organization of depth \(D\) using only human agents, and let \(\mathcal{C}^{(AI)}(D)\) denote the corresponding minimum risk-adjusted cost when AI agents are available. Suppose:
\begin{enumerate}
\item Human costs grow super-exponentially with depth:
\[
C_i(D)=C_{D+1}(1+r_0D)^{D-i+1},
\qquad r_0>0,
\]
so that deeper human hierarchies impose increasing nominal managerial costs.

\item AI agents are feasible alternatives for at least some managerial roles and have weakly lower risk-adjusted effective costs than the corresponding human agents for those roles. The inequality \(r'_0\le r_0\) (Assumption~\ref{ass:asymmetry}(iii)) makes this plausible, since AI managerial cost escalation is then no steeper than human escalation, but it does not by itself imply it.

\item Layer-redundancy condition: for every depth \(D>D^*\), at least one managerial layer can be removed without increasing the minimum AI-enabled risk-adjusted cost, that is,
\[
\mathcal{C}^{(AI)}(D-1)\le \mathcal{C}^{(AI)}(D).
\]
\end{enumerate}
Assume that \(\mathcal{C}(D)\) and \(\mathcal{C}^{(AI)}(D)\) attain their minima over \(D\). Let
\[
D^*=\min\arg\min_D \mathcal{C}(D),
\qquad
D^{(AI)*}=\min\arg\min_D \mathcal{C}^{(AI)}(D).
\]
Then
\[
D^{(AI)*}\le D^*.
\]
Moreover, if \(D^*>1\) and the layer-redundancy inequality is strict at the human-optimal depth, \(\mathcal{C}^{(AI)}(D^*-1)<\mathcal{C}^{(AI)}(D^*)\), then
\[
D^{(AI)*}<D^*.
\]
\end{styledtheorem}

\noindent
\textbf{Interpretation.} Under the conditions of the theorem, AI availability weakly reduces optimal organizational depth by lowering the risk-adjusted cost of some higher-level coordination or managerial roles and by making some managerial layers redundant. Because AI coordination costs escalate no faster than human coordination costs \((r'_0\le r_0)\), the nominal cost advantage of AI can become stronger at higher hierarchical levels, provided that the AI line-worker cost does not exceed the human line-worker cost and that risk differentials do not offset the advantage. When this advantage is large enough to make an intermediate layer unnecessary, the optimal organization becomes strictly shallower. Thus, AI provides a structural mechanism for organizational flattening, not merely by lowering task costs, but by reducing the value of maintaining costly coordination layers. The implications for average span of control and managerial layer reduction are analyzed in Corollary~\ref{cor:Flattening of Organizations}.

\noindent\textbf{Remark.}
The conclusion that AI availability can reduce organizational depth assumes that the hierarchy is allowed to be re-optimized after AI substitution. That is, reporting relationships and organizational layers are not treated as fixed exogenous constraints. Under this assumption, any managerial layer whose coordination function becomes redundant under AI-enabled substitution can be removed without increasing the minimum AI-enabled risk-adjusted cost. The layer-redundancy condition itself is a maintained structural assumption on the AI-enabled cost function \(\mathcal{C}^{(AI)}(D)\); it is not implied by the human--AI cost asymmetry of Assumption~\ref{ass:asymmetry} alone.

\subsubsection{The Human--AI Substitution Principle}

The core mechanism governing task allocation is a local, marginal comparison of risk-adjusted effective costs.

\begin{styledtheorem}{Human--AI Substitution Principle}{Human--AI Substitution Principle}
The total risk-adjusted cost is
\begin{equation}
\mathcal{C} = \sum_{i=1}^{D+1}\sum_{j=1}^{e_i}
\tilde{C}_{ij}\, q_{ij}
+ \sum_{i=1}^{D+1}\sum_{k=1}^{e'_i}
\tilde{C}'_{ik}\, q'_{ik}.
\end{equation}
Consider a perturbation $q_{ij} \to q_{ij} - \delta$, $q'_{ik} \to q'_{ik} + \delta$ for $0 < \delta \le q_{ij}$. Then
\[
\Delta\mathcal{C} = \delta\,(\tilde{C}'_{ik} - \tilde{C}_{ij}).
\]
Hence, cost decreases if and only if $\tilde{C}'_{ik} < \tilde{C}_{ij}$, i.e.,
\[
C'_{ik} + \lambda R'_{ik} < C_{ij} + \lambda R_{ij}.
\]
Even when $C'_{ik} < C_{ij}$, a sufficiently large risk differential $\lambda(R'_{ik} - R_{ij}) > 0$ can prevent cost-reducing substitution.
\end{styledtheorem}

\noindent
\textbf{Interpretation.} Transferring task allocation from a human agent to an AI agent strictly reduces total cost if and only if the AI agent has lower \emph{risk-adjusted} effective cost. A nominal cost advantage ($C'_{ik} < C_{ij}$) can be fully offset by a risk penalty ($\lambda R'_{ik} \gg \lambda R_{ij}$); this holds for any cost values and does not depend on Assumption~\ref{ass:asymmetry}. This local, exact condition is independent of the rest of the organization and serves as the foundational decision rule underlying all subsequent results.

\subsubsection{Global Optimal Allocation and Extreme-Point Structure}

We extend the local substitution principle to characterize global optimal allocations.

\begin{styledtheorem}{Global Optimal Allocation / Extreme Point Optimality}{Global Optimal Allocation}
Let
\[
\mathcal{C}(q, q') =
\sum_{i=1}^{D+1} \sum_{j=1}^{e_i} \tilde{C}_{ij}\, q_{ij}
+
\sum_{i=1}^{D+1} \sum_{k=1}^{e'_i} \tilde{C}'_{ik}\, q'_{ik}
\]
be defined over the simplex
\[
\mathcal{Q} =
\bigl\{(q, q') : q_{ij} \geq 0,\;
q'_{ik} \geq 0,\;
\textstyle\sum_{i,j} q_{ij} + \sum_{i,k} q'_{ik} = 1\bigr\}.
\]
Let
\[
m =
\min\Bigl\{
\min_{i,j}\tilde{C}_{ij},
\;
\min_{i,k}\tilde{C}'_{ik}
\Bigr\}
\]
denote the lowest risk-adjusted cost among all human and AI agents. Then the minimum value of
\(\mathcal{C}(q,q')\) over \(\mathcal{Q}\) is \(m\), and there exists a global minimizer at an extreme point of \(\mathcal{Q}\). That is, there exists an optimal allocation \((q^*,q'^*)\) such that
\[
q_{i^*j^*}^*=1
\quad \text{or} \quad
q_{i^*k^*}^{\prime *}=1,
\]
with all other allocation weights equal to zero, where the selected agent satisfies
\[
\tilde{C}_{i^*j^*}=m
\quad \text{or} \quad
\tilde{C}'_{i^*k^*}=m.
\]

If the minimum-risk-adjusted-cost agent is unique, then the global minimizer is unique and assigns the entire task to that agent. If multiple human or AI agents attain the same minimum value \(m\), then any allocation supported only on those tied minimum-cost agents is also globally optimal. Equivalently, the set of global minimizers is the convex hull of the extreme points corresponding to agents whose risk-adjusted cost equals \(m\).
\end{styledtheorem}

\noindent
\textbf{Interpretation.}
In the unconstrained HAT allocation problem, the organization can always achieve an optimum by assigning the task to an agent with minimum risk-adjusted cost. If that agent is unique, the optimal allocation is a pure human or pure AI assignment; if several agents tie, any mixture among those tied agents is equally optimal. Thus, mixing arises only from cost ties or additional organizational constraints, not from the linear unconstrained objective itself. The threshold rule that follows characterizes when transferring task weight from a specified human agent to a specified AI agent reduces cost.

\subsubsection{Threshold Substitution and Flattening Dynamics}

Substitution decisions are governed by a sharp threshold condition, and, under the layer-redundancy condition, AI availability can produce flatter organizational structures.

\begin{styledtheorem}{Threshold Substitution Rule}{Threshold Substitution Rule}
Fix a human agent $(i,j)$ and an AI agent $(i,k)$ at the same level $i$, and consider transfers of task weight between them, holding all other weights fixed. Define $\Delta_{ikj} = \tilde{C}'_{ik} - \tilde{C}_{ij}$. Then:\\[2pt]
If $\Delta_{ikj} < 0$: full substitution, i.e., transferring all of the weight $q_{ij}$ from $(i,j)$ to $(i,k)$ minimizes cost.\\
If $\Delta_{ikj} > 0$: no substitution, i.e., no transfer from $(i,j)$ to $(i,k)$ reduces cost.\\
If $\Delta_{ikj} = 0$: pairwise indifference, i.e., every such transfer leaves cost unchanged.
\end{styledtheorem}

\noindent
\textbf{Interpretation.} In the unconstrained problem, substitution is a \emph{threshold phenomenon}: for a given pair of agents the allocation switches once the risk-adjusted cost differential changes sign, with no interior mixture except at a tie. This gives a clear decision boundary for the modeled task. Here full substitution and no substitution refer to moving all or none of the weight between the two specified agents $(i,j)$ and $(i,k)$; whether the AI agent is also the globally optimal choice when other agents are available is determined by Theorem~\ref{thm:Global Optimal Allocation}. Whether observed AI replacement is sudden or gradual depends on how thresholds, switching costs and the timing of adoption vary across tasks, roles and firms, which the model does not determine.

\noindent\textbf{Remark.}
The threshold structure in Theorem~\ref{thm:Threshold Substitution Rule} follows from the linearity of the risk-adjusted objective function together with the unconstrained simplex formulation of the allocation problem. Under nonlinear cost functions, interaction effects, regularization terms, or additional organizational constraints, optimal allocations may occur at interior points of the feasible region, leading to gradual rather than sharp substitution.

\begin{styledcorollary}{Flattening of Organizations}{Flattening of Organizations}
Under the conditions of Theorem~\ref{thm:Optimal Depth under AI Availability}, AI availability weakly reduces optimal hierarchical depth:
\[
D^{(AI)*}\le D^*.
\]
If the inequality is strict and the AI-enabled organization, with spans $s^{(AI)}_\ell$, at least maintains the original execution capacity $E=\prod_{\ell=1}^{D^*}s_\ell$, that is, $\prod_{\ell=1}^{D^{(AI)*}}s^{(AI)}_\ell\ge E$, then the reduced-depth organization requires a weakly larger geometric-mean span of control (strictly larger if $E>1$). Thus, AI adoption can produce flatter hierarchies with wider average managerial spans.
\end{styledcorollary}

\noindent
\textbf{Interpretation.}
AI availability can flatten organizational structure by weakly reducing optimal depth under the conditions of Theorem~\ref{thm:Optimal Depth under AI Availability}. When depth is strictly reduced but comparable execution capacity is maintained, equation~\eqref{eq:level_size} implies that the average span of control cannot decrease. Here execution capacity is the number of line-worker positions $e_{D+1}=\prod_{\ell=1}^{D}s_\ell$, comparable capacity means that this number is at least maintained after adoption, and the average span is the geometric mean of the spans, $e_{D+1}^{1/D}$, as in the proof. For a given $e_{D+1}>1$ this mean is strictly larger in a shallower hierarchy; if $e_{D+1}=1$ it need not increase. Thus, AI adoption can shift organizations toward fewer managerial layers and broader average managerial spans, especially when AI-assisted coordination allows remaining managers to supervise larger functional domains.

The preceding results show how AI affects organizational depth and allocation structure. We now examine why substitution pressure need not be uniform across the hierarchy and can instead peak at intermediate levels.

\subsection{Middle-Management Vulnerability}
\label{sec:Middle-Management Vulnerability}

The HAT model identifies conditions under which substitution likelihood can peak at intermediate hierarchical levels rather than at the top or bottom. This section analyzes middle-management vulnerability as a conditional mechanism arising from the interaction between upward cost-based substitution pressure and top-level substitution-feasibility constraints.

\subsubsection{Level-Dependent Substitution Patterns}

We first establish that substitution incentives vary systematically across hierarchical levels.

\begin{styledtheorem}{Level-Dependent Substitution Pressure}{Level-Dependent Substitution Likelihood}
Let $C_i = C_{D+1}(1 + r_0 D)^{D-i+1}$, $r_0 > 0$. Then $C_i$ is strictly decreasing in $i$, so, holding AI cost and the risk terms $R_{ij}$ and $R'_{ik}$ fixed, the cost-based incentive for substitution (cost reduction from replacing a human by an AI agent) increases as $i \to 1$ (higher levels in the hierarchy). Under Assumption~\ref{ass:asymmetry}(iii), $r'_0 \le r_0$, so AI managerial cost escalation is no steeper than human escalation; if in addition $C'_{D+1} \le C_{D+1}$, the nominal gap $C_i - C'_i$ does not narrow toward the top, reinforcing this gradient.
\end{styledtheorem}

\noindent
\textbf{Interpretation.} From a cost perspective, higher levels are more attractive substitution targets because human cost $C_i$ grows rapidly toward the top. However, this substitution incentive is moderated by two competing forces:

\begin{itemize}
    \item \textbf{Cost Effect (Favors Substitution).} Since $C_i$ is largest for small $i$ (higher levels), the nominal cost gap $C_i - C'_i$ does not narrow toward the top (and widens if either inequality is strict) when the AI line-worker cost does not exceed the human line-worker cost, $C'_{D+1} \le C_{D+1}$, and $r'_0 \le r_0$. The risk-adjusted gain $\tilde{C}_{ij} - \tilde{C}'_{ik}$ is then typically larger at upper levels, making higher-level human agents more attractive candidates for replacement.

    \item \textbf{Risk Effect (Limits Substitution).} Higher-level roles involve strategic decision-making and non-routine judgment, attracting elevated AI risk penalties---large $\omega_1 R^{\mathrm{(rel)}}_{ik}$ (reliability) and $\omega_2 R^{\mathrm{(comp)}}_{ik}$ (compliance) from~\eqref{eq:AI_risk}---that raise $\tilde{C}'_{ik}$ and restrict AI substitution.
\end{itemize}

The observed pattern reflects the trade-off between rising cost-based substitution pressure and declining substitution feasibility near the top of the hierarchy. To make this trade-off precise, let \(S(i)>0\) denote the cost-based substitution incentive at level \(i\), which tends to rise toward the top of the hierarchy because the human cost \(C_i\) does, and let \(F(i)>0\) denote the substitution-feasibility factor, which captures accountability, compliance, reputational and strategic-role constraints and is typically lowest near the top. Their product \(L(i)=S(i)F(i)\) is the substitution-exposure index used in Corollary~\ref{cor:Middle-Management Vulnerability}, where it is called substitution likelihood; it is a constructed index of relative exposure by level, not a probability. With \(\alpha_i=S(i)/S(i+1)\) (decay of the cost incentive one level down) and \(\beta_i=F(i+1)/F(i)\) (gain in feasibility one level down), moving one level down multiplies the index by \(L(i+1)/L(i)=\gamma_i=\beta_i/\alpha_i\). The single-crossing relation in the corollary means that \(\gamma_i\) crosses 1 exactly once, from above to below: \(\gamma_i>1\) for \(i<i^*\) and \(\gamma_i<1\) for \(i\ge i^*\), for some \(i^*\in\{2,\ldots,D\}\). This is a maintained condition on the two components \(S\) and \(F\); it is not derived from the cost equations of the model. Figure~\ref{fig:hat-mechanism-b} illustrates how \(S(i)\) and \(F(i)\) can satisfy this single-crossing pattern, which generates an intermediate-level peak in substitution likelihood.

\begin{figure}[htbp!]
  \centering
    \includegraphics[width=0.92\textwidth]{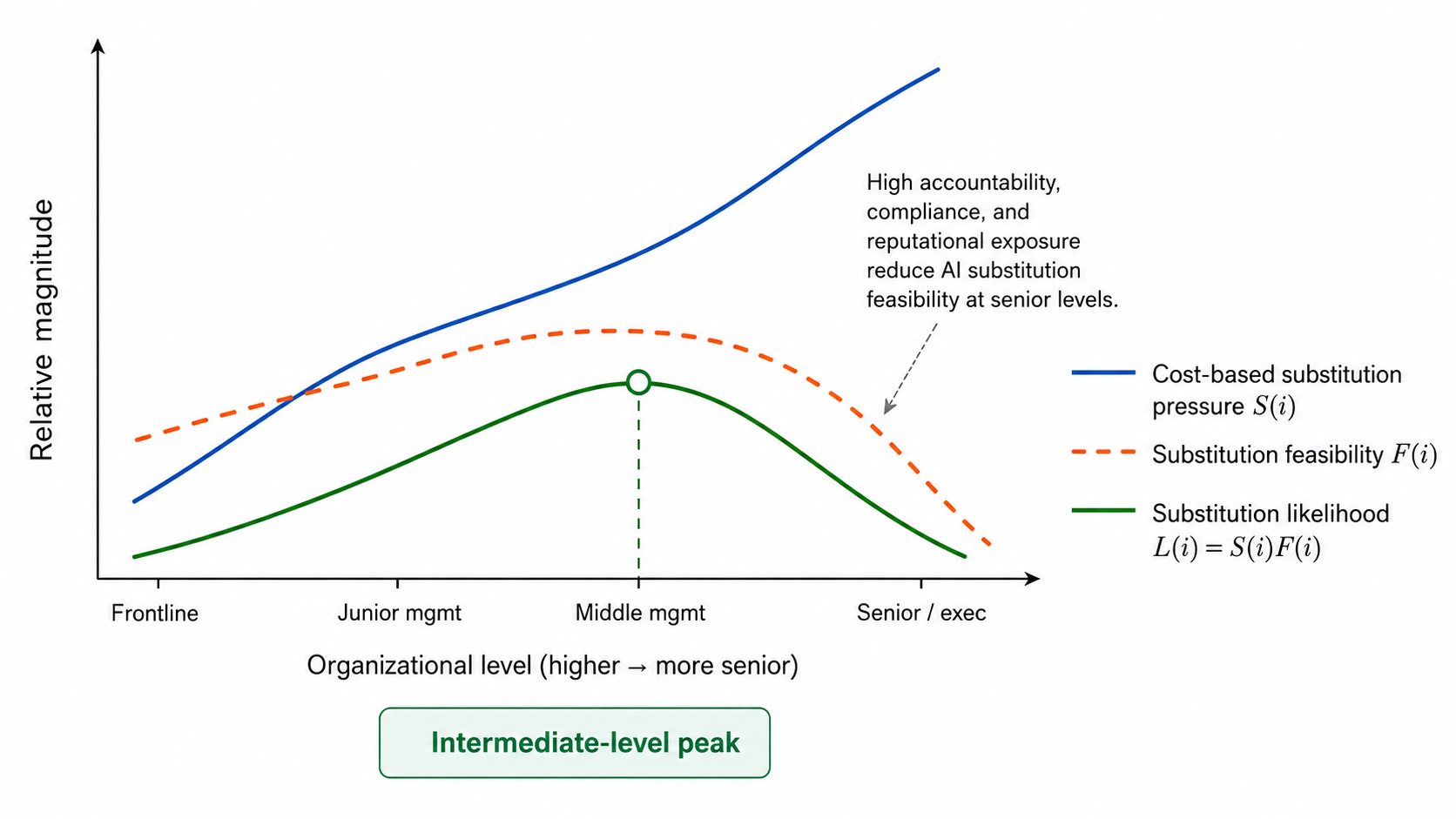}
    \caption{Illustration of the single-crossing mechanism behind middle-management vulnerability. The figure plots cost-based substitution pressure \(S(i)\), substitution feasibility \(F(i)\), and their product \(L(i)=S(i)F(i)\) across hierarchical levels. In the depicted case, \(L(i)\) reaches an intermediate-level peak, corresponding to the condition characterized in Corollary~\ref{cor:Middle-Management Vulnerability}.}
    \label{fig:hat-mechanism-b}
\end{figure}

\begin{styledcorollary}{Middle-Management Vulnerability}{Middle-Management Vulnerability}
Suppose \(D\ge 2\). Under the HAT cost specification, if the upward increase in cost-based substitution pressure is moderated by substitution-feasibility constraints that are strongest near the top of the hierarchy, then the likelihood of AI substitution need not be maximized at either the executive level or the line-worker level. Instead, under a single-crossing relation between cost-incentive decay and feasibility gains, there exists an intermediate level
\[
i^*\in\{2,\ldots,D\}
\]
at which substitution likelihood is maximized.
\end{styledcorollary}

\noindent
\textbf{Interpretation.}
This corollary does not imply that middle-management roles are always the most vulnerable to AI substitution. Rather, it identifies the structural condition under which middle-management vulnerability emerges in the HAT framework. The cost-based incentive to substitute AI tends to become stronger as one moves upward in the hierarchy because human organizational costs escalate with hierarchical level. However, actual substitution feasibility can be lowest near the top because senior roles carry greater accountability, compliance exposure, reputational consequences, and strategic decision authority. Middle-management vulnerability arises when moving downward from the top initially produces a sufficiently large gain in substitution feasibility to offset the decline in cost-based substitution pressure, but beyond some intermediate level the additional feasibility gain is no longer large enough to compensate for the weaker cost-based incentive. Under this single-crossing pattern, substitution likelihood peaks at an intermediate level \(i^*\), rather than at the executive or line-worker level. Thus, the HAT model characterizes middle-management vulnerability as a conditional organizational mechanism: intermediate roles are most exposed when they combine substantial cost-based automation incentives with relatively fewer top-level feasibility constraints.

\subsubsection{Skill-Dependent Vulnerability and Protection Zones}

Vulnerability also varies with skill heterogeneity at the line-worker level, creating protection and vulnerability zones. Figure~\ref{fig:hat-mechanism-a} provides a visual representation of this threshold logic by plotting the risk-adjusted human and AI cost curves as functions of skill or capability.

\begin{styledtheorem}{Skill-Adjusted Substitution Condition}{Skill-Adjusted Substitution Condition}
Using the line-worker cost structures~\eqref{eq:line_worker_cost} and~\eqref{eq:AI_lineworker_parametric}, transferring task allocation from a human line worker $(D+1,j)$ to an AI agent $(D+1,k)$ strictly reduces total cost if and only if
\[
\textup{Min\$}' + \tfrac{1}{2}\Delta\$' \cdot D
\cdot \mu'_{D+1,k} + \lambda R'_{D+1,k}
\;<\;
\textup{Min\$} + \tfrac{1}{2}\Delta\$ \cdot D
\cdot \mu_{D+1,j} + \lambda R_{D+1,j}.
\]
Under Assumption~\ref{ass:asymmetry}(i), $\Delta\$' \le \Delta\$$, so AI capability costs grow no faster than human skill costs.
\end{styledtheorem}

\noindent
\textbf{Interpretation.} Substitution at the line-worker level incorporates baseline costs ($\textup{Min\$}'$ vs.\ $\textup{Min\$}$), depth-scaled skill costs ($\tfrac{1}{2}\Delta\$' \cdot D \cdot \mu'_{D+1,k}$ vs.\ $\tfrac{1}{2}\Delta\$ \cdot D \cdot \mu_{D+1,j}$), and risk differentials. The $D$ factor reflects that deeper organizations demand higher capability at every level. An AI agent must outperform on the combined risk-adjusted metric, not merely on raw cost.

\begin{styledcorollary}{Skill-Dependent Human Protection and Vulnerability Zone}{Skill-Dependent Human Protection and Vulnerability Zone}
Fix an AI agent $(D+1,k)$ with risk-adjusted effective cost
\[
\tilde{C}'_{D+1,k}
= \textup{Min\$}' + \tfrac{1}{2}\Delta\$' \cdot D
\cdot \mu'_{D+1,k} + \lambda R'_{D+1,k}.
\]
Holding \(R_{D+1,j}\), \(R'_{D+1,k}\), and the AI capability level \(\mu'_{D+1,k}\) fixed, suppose the human line-worker risk-adjusted cost
\[
\tilde{C}_{D+1,j}
= \textup{Min\$} + \tfrac{1}{2}\Delta\$ \cdot D
\cdot \mu_{D+1,j} + \lambda R_{D+1,j}
\]
is strictly increasing in \(\mu_{D+1,j}\), which holds when \(\Delta\$>0\) and \(D\ge 1\). Then there exists a threshold, possibly outside the feasible skill range,
\[
\mu^* =
\frac{2\bigl[\textup{Min\$}' - \textup{Min\$}
+ \lambda(R'_{D+1,k} - R_{D+1,j})\bigr]}
{\Delta\$ \cdot D}
+
\frac{\Delta\$'}{\Delta\$}\cdot \mu'_{D+1,k}
\]
such that:
\[
\mu_{D+1,j} < \mu^*
\;\Longrightarrow\;
\tilde{C}_{D+1,j} < \tilde{C}'_{D+1,k}
\qquad(\text{protection zone}),
\]
\[
\mu_{D+1,j} > \mu^*
\;\Longrightarrow\;
\tilde{C}_{D+1,j} > \tilde{C}'_{D+1,k}
\qquad(\text{vulnerability zone}).
\]
Equivalently, human vulnerability to substitution depends on how rapidly risk-adjusted human compensation grows relative to risk-adjusted AI capability cost as skill increases.
\end{styledcorollary}

\noindent
\textbf{Interpretation.} The relationship between skill and substitution is fundamentally asymmetric. For humans, higher skill requires scarce talent and accumulated experience, so risk-adjusted cost rises with skill. For AI, per-deployment cost grows more weakly with capability under Assumption~\ref{ass:asymmetry}(i), since \(\Delta\$'\le \Delta\$\). The threshold \(\mu^*\) depends on baseline cost differences, organizational depth \(D\), AI capability \(\mu'_{D+1,k}\), and the risk differential \(\lambda(R'_{D+1,k}-R_{D+1,j})\). For \(\lambda>0\), larger AI risk penalties shift \(\mu^*\) upward, expanding the human protection zone, while larger human risk shifts \(\mu^*\) downward, expanding the vulnerability zone. Organizational depth affects the threshold through the slope of human skill cost and the scaling of the baseline and risk-differential terms, so its effect on \(\mu^*\) depends on parameter values. Workers below \(\mu^*\) remain risk-adjusted-cost-efficient relative to the fixed AI agent, while workers above \(\mu^*\) face substitution pressure. Whether highly skilled workers are protected or vulnerable is therefore task-dependent and cannot be resolved without domain-specific parameters.

The middle-management vulnerability and skill-dependent protection zones establish heterogeneous substitution patterns across hierarchical levels and skill distributions. However, substitution is further constrained by risk considerations and organizational requirements, which we analyze next.

\subsection{Risk, Coordination, and the Limits of Substitution}
\label{sec:Risk, Coordination, and the Limits of Substitution}

While the substitution principle establishes when shifting task weight from a human to an AI agent reduces cost, real-world organizations face risk constraints, regulatory requirements, and coordination demands that limit the extent of automation. This section characterizes how risk penalties and explicit constraints create hybrid human--AI structures and bound the scope of full automation.

\subsubsection{Risk-Adjusted Substitution and Risk-Based Limits}

Since risk is embedded in agent cost primitives $\tilde{C}_{ij} = C_{ij} + \lambda R_{ij}$ and $\tilde{C}'_{ik} = C'_{ik} + \lambda R'_{ik}$, all substitution decisions inherently operate on risk-adjusted grounds. We now make the structure of AI risk explicit.

\begin{styledtheorem}{Risk-Adjusted Substitution Principle}{Risk-Adjusted Substitution Principle}
Since risk is embedded in individual agent cost primitives via~\eqref{eq:human_effective_cost}
and~\eqref{eq:AI_effective_cost}, the substitution condition from Theorem~\ref{thm:Human--AI Substitution Principle} already operates on risk-adjusted grounds:
\[
\tilde{C}'_{ik} < \tilde{C}_{ij}
\;\iff\;
C'_{ik} + \lambda R'_{ik} < C_{ij} + \lambda R_{ij}.
\]
Expanding $R'_{ik}$ via~\eqref{eq:AI_risk}, the AI risk-adjusted cost is
\[
\tilde{C}'_{ik}
= C'_{ik}
+ \lambda\!\left(\omega_1 R^{\mathrm{(rel)}}_{ik}
+ \omega_2 R^{\mathrm{(comp)}}_{ik}
+ \omega_3 R^{\mathrm{(rep)}}_{ik}\right).
\]
Even when $C'_{ik} < C_{ij}$, so AI is cheaper on a nominal basis, substitution fails to reduce risk-adjusted cost whenever
\[
\lambda\bigl(\omega_1 R^{\mathrm{(rel)}}_{ik}
+ \omega_2 R^{\mathrm{(comp)}}_{ik}
+ \omega_3 R^{\mathrm{(rep)}}_{ik}
- R_{ij}\bigr)
\;\ge\; C_{ij} - C'_{ik}.
\]
\end{styledtheorem}

\noindent
\textbf{Interpretation.} Substitution requires AI to be cheaper after accounting for three risk penalties weighted by $\lambda \ge 0$: \textbf{reliability risk} ($R^{\mathrm{(rel)}}_{ik}$, weighted by $\omega_1$) from errors or distributional shift; \textbf{compliance risk} ($R^{\mathrm{(comp)}}_{ik}$, weighted by $\omega_2$) from regulatory constraints; and \textbf{reputational risk} ($R^{\mathrm{(rep)}}_{ik}$, weighted by $\omega_3$) from trust loss or brand damage (see Figure~\ref{fig:hat-mechanism-c}). In regulated industries, large $\omega_2$ can eliminate or reverse the nominal AI cost advantage. Risk considerations create a \emph{buffer region} in which transferring task allocation from the specified human agent to the specified AI agent does not reduce total risk-adjusted cost, providing a risk-based explanation for cautious AI adoption in high-stakes environments.

\subsubsection{Constrained Optimization and Hybrid Organizations}

We now incorporate explicit constraints and show how they—together with heterogeneity in risk-adjusted costs across tasks—can produce constrained hybrid allocations.

\begin{styledtheorem}{Constrained Hybrid Optima}{Constrained Interior Optima}
Let \(\mathcal{C}\) be defined over the simplex \(\mathcal{Q}\) as in Theorem~\ref{thm:Global Optimal Allocation}. Suppose additional linear constraints are imposed, such as minimum human allocation, regulatory bounds, capacity constraints, or skill requirements, so that the feasible set becomes
\[
\mathcal{Q}_c
=
\{(q,q')\in\mathcal{Q}: A(q,q')\le b,\; E(q,q')=d\},
\]
where \(\mathcal{Q}_c\) is nonempty. Then the constrained problem
\[
\min_{(q,q')\in\mathcal{Q}_c}\mathcal{C}(q,q')
\]
has a global solution, and at least one global minimizer is an extreme point of \(\mathcal{Q}_c\). However, an extreme point of \(\mathcal{Q}_c\) need not be an extreme point of the original simplex \(\mathcal{Q}\). Consequently, when the additional constraints exclude the unconstrained minimum-cost vertex or impose positive lower or upper bounds across agent types, an optimal constrained allocation
may assign positive weight to multiple agents, including both human and AI agents.

\noindent\textbf{Remark.}
Under the AI risk decomposition \eqref{eq:AI_risk}, a large compliance, reliability, or reputational risk component can eliminate a nominal AI cost advantage:
\[
\tilde{C}'_{ik}=C'_{ik}+\lambda R'_{ik}
\ge
\tilde{C}_{ij}=C_{ij}+\lambda R_{ij}
\]
even when \(C'_{ik}<C_{ij}\). In the unconstrained single-task problem, this risk adjustment can change which vertex is optimal when it affects the global minimum risk-adjusted cost; by itself it does not generally create an interior mixture unless there is a tie. In constrained or heterogeneous multi-task settings, however, risk-adjusted cost heterogeneity can sustain hybrid human--AI organizational structures without imposing an exogenous minimum-human-fraction constraint.
\end{styledtheorem}

\noindent
\textbf{Interpretation.}
The unconstrained HAT problem selects a minimum-risk-adjusted-cost agent over the original simplex. Additional feasibility constraints change the geometry of the allocation problem: the optimum is still attained at an extreme point of the constrained feasible set, but that point need not be a vertex of the original simplex. As a result, binding constraints such as minimum human presence, regulatory limits, capacity bounds, or skill requirements can make hybrid allocations optimal, with both human and AI agents receiving positive task weight. Risk adjustment plays a complementary role: high reliability, compliance, or reputational risk can eliminate AI's nominal cost advantage for some roles. Thus, constrained and risk-heterogeneous environments provide a structural mechanism for persistent partial substitution and hybrid human--AI organizations. The case of risk-adjusted cost heterogeneity across separable tasks, which requires no constraint, is made explicit in the remark following Theorem~\ref{thm:AI-Dominant Extreme-Point Allocation}.

\begin{styledcorollary}{Minimum Human Fraction Constraint}{Minimum Human Fraction Constraint}
Suppose the only additional constraint imposed on the simplex \(\mathcal{Q}\) is $\sum_{(i,j)\in\mathcal H} q_{ij}\ge \gamma$, $\gamma\in(0,1]$, where \(\mathcal H\) indexes all human agents. If every human agent has risk-adjusted cost strictly greater than at least one feasible AI agent, then every optimal allocation assigns human weight only up to the required boundary:
\[
\sum_{(i,j)\in\mathcal H} q_{ij}=\gamma.
\]
If at least one human agent and one AI agent are tied at the global minimum risk-adjusted cost, then there exists an optimal allocation satisfying the boundary condition and, when $\gamma<1$, allocations with a larger human share may also be optimal.
\end{styledcorollary}

\noindent
\textbf{Interpretation.} When every human agent is strictly more expensive in risk-adjusted cost than a feasible AI alternative, the minimum-human-fraction constraint binds: human participation is reduced to the lowest allowable level \(\gamma\). Thus, regulatory, policy, or organizational constraints set a floor on human participation and a corresponding ceiling on automation, even when AI has a strict risk-adjusted cost advantage. If the cheapest human agent and the cheapest AI agent tie at the global minimum risk-adjusted cost, a boundary allocation with human weight \(\gamma\) remains optimal, and, when \(\gamma<1\), additional human participation may also be optimal because it does not increase total risk-adjusted cost.

\subsubsection{Global Condition for Full Automation}

\textbf{Definition.}
\textit{Full automation} occurs when, for every modeled organizational task \(\tau\in\mathcal{T}\), there exists an optimal unconstrained HAT allocation that assigns the task entirely to an AI agent. A stronger form, \textit{strict full automation}, occurs when every cost-minimizing allocation assigns each task only to AI agents.

\begin{styledtheorem}{Condition for AI-Dominant Extreme-Point Allocation}{AI-Dominant Extreme-Point Allocation}
In the unconstrained HAT allocation problem for a given task \(\tau\in\mathcal{T}\), let
\[
m_H(\tau)=\min_{i,j}\widetilde C_{ij}(\tau),
\qquad
m_A(\tau)=\min_{i,k}\widetilde C'_{ik}(\tau).
\]
Then an AI-only extreme-point allocation is optimal if and only if
\[
m_A(\tau)\le m_H(\tau).
\]
Moreover, the optimum is AI-dominant, in the sense that no human agent is optimal, if and only if 
$m_A(\tau)<m_H(\tau)$. If \(m_A(\tau)=m_H(\tau)\), then an AI-only extreme-point allocation is optimal but not unique, since at least one human allocation also attains the same minimum risk-adjusted cost.
\end{styledtheorem}

\noindent
\textbf{Interpretation.}
For a given task, AI dominance occurs when the best AI agent has lower risk-adjusted effective cost than the best human agent. This condition is stronger than nominal cost advantage because AI risk penalties \(\lambda R'_{ik}\) may offset lower direct AI cost. Thus, high organizational risk sensitivity \(\lambda\) or large reliability, compliance, and reputational weights \((\omega_1,\omega_2,\omega_3)\) can prevent AI dominance even when \(C'_{ik}<C_{ij}\).

The theorem also clarifies the organization-wide limiting case. If AI agents are replicable across a set of separable organizational tasks \(\mathcal{T}\), then full automation is optimal whenever
\[
m_A(\tau)\le m_H(\tau)
\qquad
\text{for every } \tau\in\mathcal{T}.
\]
Strict full automation occurs when
\[
m_A(\tau)<m_H(\tau)
\qquad
\text{for every } \tau\in\mathcal{T}.
\]
Human work can remain optimal when at least one task has a strictly lower human risk-adjusted cost or when humans tie AI at the minimum; outside the unconstrained setting, human participation may also persist because external constraints require it. Thus, full automation is not assumed by the HAT framework; it is the limiting case in which AI is cost-minimizing task by task under risk-adjusted optimization.

\noindent\textbf{Remark (hybrid organizations across tasks).}
Suppose the organizational tasks in $\mathcal{T}$ are separable, so that total cost is the sum of task-level costs and the HAT allocation problem is solved task by task. By Theorems~\ref{thm:Global Optimal Allocation} and~\ref{thm:AI-Dominant Extreme-Point Allocation}, a task $\tau_H$ with $m_H(\tau_H)<m_A(\tau_H)$ is assigned only to human agents, and a task $\tau_A$ with $m_A(\tau_A)<m_H(\tau_A)$ is assigned only to AI agents. If both kinds of tasks are present, the cost-minimizing organization employs both human and AI agents even though each task is assigned to a single type of agent and no minimum-human-fraction constraint is imposed. This is the sense in which heterogeneity in risk-adjusted costs across tasks can sustain hybrid human--AI organizations; it follows from applying the single-task results to each task and does not imply technological complementarity between human and AI inputs.

We now examine how these structures evolve dynamically and respond to parameter perturbations.

\subsection{Comparative Statics and Robustness}
\label{sec:Comparative Statics and Robustness}

We extend the analysis to dynamic settings and examine the stability of the optimal allocation under parameter perturbations. These results characterize the trajectory of automation under technological progress and show that, in the unconstrained linear problem, the optimal allocation changes through discrete transitions at cost thresholds rather than through smooth adjustment.

\subsubsection{Dynamics of Automation}

We first analyze how declining AI risk-adjusted costs affect the cost-minimizing allocation over time.

\begin{styledtheorem}{Monotonic Automation Path}{Monotonic Automation Path}
In the unconstrained allocation problem over \(\mathcal{Q}\), with human risk-adjusted costs held fixed and switching costs negligible, let \(\tilde{C}'_{ik}(t)\) denote the risk-adjusted cost of AI agents at time \(t\), and suppose that \(\tilde{C}'_{ik}(t)\) is non-increasing in \(t\) for every \((i,k)\). Then, once an AI-only extreme-point allocation of \(\mathcal{Q}\) is optimal, there continues to exist an optimal AI-only extreme-point allocation at every subsequent time. Once the cheapest AI agent is strictly less costly than every human agent, every subsequent optimal allocation assigns the task exclusively to AI agents. While the cheapest AI agent and the cheapest human agent have equal cost, several allocations are optimal, and a selected optimal allocation need not evolve monotonically.
\end{styledtheorem}

\noindent
\textbf{Interpretation.} When human risk-adjusted costs are held fixed, and in the unconstrained allocation problem with negligible switching costs, declining AI risk-adjusted costs can only strengthen the position of AI: once an AI-only allocation is optimal, there continues to exist an optimal AI-only allocation at every subsequent time, and once the cheapest AI agent is strictly less costly than every human agent, no human agent receives weight in any optimal allocation. Such declines may arise from falling deployment costs, improved reliability, or reduced compliance risk as governance frameworks mature. The result should therefore be interpreted as a conditional monotonicity result: under stable human-side costs and in the absence of binding constraints, improvements in AI cost or risk preserve the optimality of AI-only allocations and, once AI is strictly cheaper, make them the only optimal ones. When the cheapest human and AI agents have equal cost, several allocations are optimal at once, and a particular selection among them need not evolve monotonically. The allocation objective contains no switching-cost term, so the result describes a benchmark in which such costs are negligible; it does not analyze adjustment frictions.

\begin{styledcorollary}{Irreversibility of Automation}{Irreversibility of Automation}
Under the conditions of Theorem~\ref{thm:Monotonic Automation Path}, in the unconstrained problem with human risk-adjusted costs fixed and negligible switching costs, suppose
\[
\widetilde C'_{ik}(t_0)<\widetilde C_{ij}.
\]
Then for all \(t\ge t_0\), the optimal allocation does not revert to the human agent \((i,j)\).
\end{styledcorollary}

\noindent
\textbf{Interpretation.} Once an AI agent becomes strictly more risk-adjusted-cost-efficient than a human agent, and human costs remain fixed and switching costs are negligible, the cost-minimizing allocation of the unconstrained problem does not revert to that human agent. Reversion would require a change in one of the maintained conditions, such as human upskilling, increased AI risk, new regulatory constraints, or non-negligible switching costs. Thus, irreversibility is not absolute; it holds within the model for the unconstrained problem under fixed human-side conditions and negligible switching costs.

\subsubsection{Perturbation Stability and Discontinuous Transitions}

We now study robustness to cost parameter perturbations, revealing that the optimal allocation exhibits stability punctuated by discrete shifts.

\begin{styledtheorem}{Perturbation Stability}{Perturbation Stability}
In the unconstrained allocation problem over \(\mathcal{Q}\), suppose the minimum risk-adjusted cost is attained by a unique agent. Sufficiently small perturbations in risk-adjusted costs do not change the optimal allocation. Along a continuous path of perturbations, the optimal allocation can change only if the cost of the current minimizer ties with that of another agent at some point on the path (a human--AI threshold $\tilde{C}'_{i'k} = \tilde{C}_{ij}$, or a human--human or AI--AI equality).
\end{styledtheorem}

\noindent
\textbf{Interpretation.} Optimal allocations depend only on the \emph{ordering} of risk-adjusted costs $\{\tilde{C}_{ij}, \tilde{C}'_{ik}\}$, not exact numerical values. The parameter space can be partitioned into regions defined by strict inequalities; within each region, the identity of the minimum-risk-adjusted-cost agent remains unchanged, so small perturbations preserving these inequalities do not alter optimal allocations (piecewise constant structure). Region boundaries are characterized by equalities between the current minimum-cost agent and another candidate: the threshold $\tilde{C}'_{i'k} = \tilde{C}_{ij}$ of the theorem when the minimizer switches between a human and an AI agent, and the analogous equalities between two human agents or between two AI agents when only the identity of the selected agent changes. When parameters cross such a boundary, the ordering changes and the optimal allocation jumps discontinuously to a different extreme point. The cost-minimizing allocation therefore remains stable under small cost changes but changes abruptly at threshold crossings. At a boundary itself the minimizer is tied, so arbitrarily small perturbations can change which of the tied agents is selected; local stability refers to a unique minimizer.

\begin{styledcorollary}{Discontinuous Task-Allocation Transitions}{Discontinuous Workforce Transitions}
In the unconstrained linear problem, suppose that risk-adjusted costs vary continuously with an external parameter and that the minimum-risk-adjusted-cost agent is unique except at isolated parameter values. Whenever the identity of the unique minimizing agent changes across such a parameter value, the cost-minimizing task allocation jumps discontinuously from one extreme point to another.
\end{styledcorollary}

\noindent
\textbf{Interpretation.} The corollary concerns the composition of the cost-minimizing task allocation, i.e., the weights $q_{ij}$ and $q'_{ik}$ in the unconstrained linear problem; it does not refer to employment or headcount, which the model does not determine. Because the modeled allocation shifts only when risk-adjusted cost thresholds are crossed, it remains unchanged over intervals of parameter values and then changes discretely. This is a property of the allocation rule, not a forecast: whether observed workforce composition changes in discrete steps depends on how thresholds are distributed across tasks, roles and firms, and on adjustment costs and timing, which the model does not determine.

\subsubsection{Strategic Behavior and Endogenous Effort}

Finally, we extend the model to allow agents to strategically adjust effort or skill levels, introducing a game-theoretic dimension.

\begin{styledtheorem}{Strategic Substitution Equilibrium}{Strategic Substitution Equilibrium}
Fix a human agent \((i,j)\) and an AI agent \((i,k)\) competing for task allocation at organizational level \(i\). The interaction has two stages: effort and allocation. In the effort stage, the organization induces or implements adaptation on both agents. The human adaptation level is \(u_j\in U_j=[0,\bar u_j]\) representing upskilling, task redesign, process improvement, or increased effort associated with the human agent. The AI adaptation level is \(u_k\in U_k=[0,\bar u_k]\) representing capability improvement, fine-tuning, monitoring, reliability enhancement, compliance controls, or other risk-mitigation investments associated with the AI agent. Here \(\bar u_j,\bar u_k>0\). The realized risk-adjusted effective costs are
\[
\widetilde C_{ij}(u_j)
=
\widetilde C^0_{ij}-h_j(u_j),
\qquad
\widetilde C'_{ik}(u_k)
=
\widetilde C^{\prime 0}_{ik}-a_k(u_k),
\]
where \(h_j\) and \(a_k\) are continuous, nondecreasing cost-reduction functions satisfying $h_j(0)=0$, $a_k(0)=0$.

Let the effort costs be \(\phi_j(u_j)\) and \(\psi_k(u_k)\), where \(\phi_j\) and \(\psi_k\) are continuous and convex, with $\phi_j(0)=0$, $\psi_k(0)=0$. Let \(\rho>0\) denote the sensitivity of the effort-stage contest to realized risk-adjusted cost differences. Define
\[
p_j(u_j,u_k)
=
\frac{\exp[-\rho \widetilde C_{ij}(u_j)]}
{\exp[-\rho \widetilde C_{ij}(u_j)]
+
\exp[-\rho \widetilde C'_{ik}(u_k)]},
\qquad
p'_k(u_j,u_k)=1-p_j(u_j,u_k).
\]
For finite \(\rho\), \(p_j\) and \(p'_k\) are reduced-form contest-success measures that determine effort-stage incentives, rather than literal probabilities generated by the deterministic allocation rule of the subsequent stage; the effort-stage game thus represents a smooth incentive mechanism associated with anticipated substitution pressure. Let \(B_j\ge 0\) and \(B'_k\ge 0\) denote the benefits of being selected for the human and AI sides, respectively. The effort-stage payoffs are
\[
\Pi_j(u_j,u_k)
=
B_j p_j(u_j,u_k)-\phi_j(u_j),
\]
\[
\Pi'_k(u_j,u_k)
=
B'_k p'_k(u_j,u_k)-\psi_k(u_k).
\]

Assume that \(\Pi_j(\cdot,u_k)\) and \(\Pi'_k(u_j,\cdot)\) are quasiconcave in their own effort choices for every fixed rival effort. Then the effort game admits a Nash equilibrium
\[
u^*=(u_j^*,u_k^*).
\]

In the allocation stage, the organization applies the HAT allocation rule to the realized equilibrium costs $\widetilde C_{ij}(u_j^*)$ and $\widetilde C'_{ik}(u_k^*)$. Thus, in our two-agent allocation subproblem,
\[
q_{ij}^*=1,\quad q_{ik}^{\prime *}=0
\quad
\text{if}
\quad
\widetilde C_{ij}(u_j^*)<\widetilde C'_{ik}(u_k^*),
\]
\[
q_{ij}^*=0,\quad q_{ik}^{\prime *}=1
\quad
\text{if}
\quad
\widetilde C'_{ik}(u_k^*)<\widetilde C_{ij}(u_j^*).
\]

If $\widetilde C_{ij}(u_j^*)=\widetilde C'_{ik}(u_k^*)$, then both agents attain the same realized minimum risk-adjusted cost, and any allocation supported on the tied agents is optimal. If \(h_j,a_k,\phi_j,\psi_k\) are differentiable, with one-sided derivatives at the boundaries, then the human equilibrium effort satisfies one of the following necessary marginal conditions:
\[
u_j^*=0
\quad\Longrightarrow\quad
B_j\rho\,p_j(u^*)p'_k(u^*)\dot h_j(0)
\le
\dot\phi_j(0),
\]
\[
u_j^*=\bar u_j
\quad\Longrightarrow\quad
B_j\rho\,p_j(u^*)p'_k(u^*)\dot h_j(\bar u_j)
\ge
\dot\phi_j(\bar u_j),
\]
or
\[
u_j^*\in(0,\bar u_j)
\quad\Longrightarrow\quad
B_j\rho\,p_j(u^*)p'_k(u^*)\dot h_j(u_j^*)
=
\dot\phi_j(u_j^*).
\]
Similarly, the AI equilibrium investment satisfies one of
\[
u_k^*=0
\quad\Longrightarrow\quad
B'_k\rho\,p_j(u^*)p'_k(u^*)\dot a_k(0)
\le
\dot\psi_k(0),
\]
\[
u_k^*=\bar u_k
\quad\Longrightarrow\quad
B'_k\rho\,p_j(u^*)p'_k(u^*)\dot a_k(\bar u_k)
\ge
\dot\psi_k(\bar u_k),
\]
or
\[
u_k^*\in(0,\bar u_k)
\quad\Longrightarrow\quad
B'_k\rho\,p_j(u^*)p'_k(u^*)\dot a_k(u_k^*)
=
\dot\psi_k(u_k^*).
\]

Finally, for any fixed effort profile \((u_j,u_k)\),
\[
\lim_{\rho\to\infty}p_j(u_j,u_k)
=
\begin{cases}
1, & \widetilde C_{ij}(u_j)<\widetilde C'_{ik}(u_k),\\
0, & \widetilde C_{ij}(u_j)>\widetilde C'_{ik}(u_k),\\
\frac12, & \widetilde C_{ij}(u_j)=\widetilde C'_{ik}(u_k).
\end{cases}
\]
Thus, as \(\rho\to\infty\), the smooth effort-stage contest converges to the deterministic risk-adjusted cost comparison underlying the Human--AI Substitution Principle.
\end{styledtheorem}

\noindent
\textbf{Interpretation.}
This theorem makes strategic adaptation explicit. Humans can respond to substitution pressure by investing in upskilling or task redesign, while AI agents can respond through capability improvement, reliability enhancement, monitoring, or risk mitigation. Equilibrium effort balances the marginal benefit of improving selection prospects against the marginal cost of effort or investment. The factor \(p_j(u^*)p'_k(u^*)\) is largest when the human--AI contest is close and small when one side already has a decisive cost advantage, so, other things being equal, adaptation incentives are strongest near the substitution boundary. As \(\rho\to\infty\), the smooth contest converges, for any fixed effort profile, to the deterministic risk-adjusted cost comparison that underlies the Human--AI Substitution Principle.

The comparative statics and robustness results show that, along a path of non-increasing AI risk-adjusted costs with human-side conditions held fixed, once an AI-only allocation of the unconstrained problem is optimal, an optimal AI-only allocation continues to exist at every subsequent time and, with negligible switching costs, the allocation does not revert to a strictly more costly human agent (Theorem~\ref{thm:Monotonic Automation Path} and Corollary~\ref{cor:Irreversibility of Automation}); that the allocation is stable while the ordering of risk-adjusted costs is unchanged and changes discretely when a threshold is crossed (Theorem~\ref{thm:Perturbation Stability} and Corollary~\ref{cor:Discontinuous Workforce Transitions}); and that, in the two-agent effort game, substitution outcomes reflect an equilibrium in which both agents adjust effort (Theorem~\ref{thm:Strategic Substitution Equilibrium}). Together with the baseline properties (\S\ref{sec:Baseline Properties Without AI}), depth reduction mechanisms (\S\ref{sec:Impact of AI on Organizational Structure}), middle-management vulnerability (\S\ref{sec:Middle-Management Vulnerability}), and risk-based constraints (\S\ref{sec:Risk, Coordination, and the Limits of Substitution}), these results provide a comprehensive characterization of the HAT model's structural properties and their implications for organizational design under AI adoption.

\section{Discussion}
\label{sec:Discussions}

Sections~\ref{sec:Human-AI Task Allocation Model} and~\ref{sec:Properties of this model} develop the Human--AI Task Allocation (HAT) framework and establish its theoretical properties. This section examines the economic significance of those results and clarifies their implications for organizational choice. The central contribution of the framework is to explain how the allocation of a given task between human and AI agents depends jointly on their relative cost structures, organizational position, deployment scale, and exposure to risk. It also examines how these assignment decisions can influence the cost-minimizing structure of the organization and how human and AI agents may respond strategically to substitution incentives.

The analysis connects three economic margins that are often considered separately. The first is the \textbf{task-allocation margin}: whether a given task should be assigned to a human or an AI agent. The second is the \textbf{organizational-design margin}: how the relative costs of human and AI agents affect the number and distribution of hierarchical positions. The third is the \textbf{adaptation margin}: how human and AI agents respond when their relative economic positions change. Risk and feasibility constraints operate across all three margins by determining whether a nominal cost advantage translates into an economically feasible allocation.

The discussion distinguishes the implications that follow directly from the model from those requiring additional assumptions or empirical investigation. The results are conditional theoretical propositions. They identify the organizational choices implied by the specified cost functions and constraints, rather than predicting the adoption decisions of all firms or the aggregate effects of AI on employment.

\subsection{Task Allocation as a Risk-Adjusted Organizational Choice}
\label{sec:Task Allocation as a Risk-Adjusted Organizational Choice}

The HAT framework treats the assignment of a task as an optimization problem. Given a task and a set of feasible human and AI agents, the organization chooses allocation weights to minimize total risk-adjusted cost. The weights are non-negative and sum to one, so that the task is allocated across the available agents.

The relevant comparison is between the risk-adjusted effective costs of individual agents. For a human agent and an AI agent, respectively, these are:
\[
\tilde{C}_{ij} = C_{ij} + \lambda R_{ij}, \qquad \tilde{C}'_{ik} = C'_{ik} + \lambda R'_{ik}.
\]

Here, $C_{ij}$ and $C'_{ik}$ represent direct effective costs, while $R_{ij}$ and $R'_{ik}$ represent the corresponding risk measures. The parameter $\lambda$ captures the organization's sensitivity to risk.

This formulation changes the economic interpretation of the adoption decision. The organization is not simply comparing the wage of an employee with the subscription or operating price of an AI system. It is comparing the total modeled cost of alternative assignments, including the risk exposure associated with each agent. The cost functions provide a reduced-form representation of relevant organizational expenditure, while the risk terms capture the potential consequences of operational failure, non-compliance and reputational damage.

The decision can also be interpreted in terms of organizational surplus. Let $B$ denote the benefit from completing a given task. If this benefit is invariant across alternative assignments, the organizational value associated with an allocation can be expressed as
\[
V = B - \mathcal{C},
\]
where $\mathcal{C}$ is the total risk-adjusted cost of the allocation and $B$ is measured on the same per-unit-time basis as $\mathcal{C}$. Consequently,
\[
\max V \quad\Longleftrightarrow\quad \min \mathcal{C}.
\]

Thus, cost minimization is equivalent to surplus maximization under the maintained assumption that task benefits do not vary with the identity of the agent.

This qualification is economically important. The model does not specify a production function in which different human--AI combinations generate different levels of output or quality. Nor does it allow the benefit from task completion to vary with the allocation. It therefore identifies the least-cost assignment for a given task rather than the most productive combination of human and AI inputs. An extension allowing assignment-dependent output quality or production complementarities could generate a different allocation problem.

The principal economic object in HAT is therefore the organizational assignment decision conditional on the task, the available agents and the relevant institutional environment. This provides a framework for understanding why technical capability alone does not determine substitution. The economically relevant question is whether deploying AI produces a sufficiently large reduction in total risk-adjusted cost relative to the available human alternative.

\subsection{Human--AI Cost Asymmetry and the Economics of Substitution}
\label{sec:Human--AI Cost Asymmetry and the Economics of Substitution}

The substitution mechanism is built on the different cost structures of human expertise and AI capability.

Human line-worker costs increase with skill and organizational depth. Managerial costs are linked to the representative line-worker cost and escalate with hierarchical position. The model therefore captures the idea that acquiring and employing human expertise involves costs associated with education, experience, accumulated knowledge and scarce managerial talent.

AI costs have a different composition. They include a fixed training or development expenditure, $T_k$, amortized across $n_k$ deployments, a marginal operating cost, $M'_k$, and a capability-related cost. The per-deployment fixed-cost component is
\[
T_k/n_k.
\]

Consequently, an increase in the number of deployments reduces the fixed training expenditure attributable to each deployment, holding other components constant. However, this does not make AI deployment costless. Inference, maintenance and capability-related expenditure remain relevant.

This distinction introduces an important economic channel into the substitution decision. AI can benefit from economies of scale in the use of an already-developed system, while human skill costs are associated with the employment and development of individual workers. The number of deployments is exogenous in the present framework, so the model captures the cost consequences of a given deployment scale without determining the organization's optimal investment or adoption scale.

The Human--AI Cost Asymmetry Assumption provides the structural basis for the substitution results. Holding the respective baseline costs fixed, the marginal capability-related cost of AI agents is assumed not to exceed the marginal skill-related cost of human agents at the line-worker level. Similarly, the coordination-cost multiplier for AI is assumed not to exceed its human counterpart. These assumptions imply that AI costs need not rise as rapidly as human costs when capability requirements or organizational depth increase.

The assumptions should not, however, be interpreted as establishing a universal cost advantage for AI. The comparison also depends on baseline costs, capability requirements, deployment scale, organizational position and risk. AI may have a lower marginal capability cost but a relatively high initial development cost. It may also incur substantial implementation, monitoring or compliance costs. The asymmetry assumption governs the relative cost trajectories; it does not predetermine which agent is cheapest in every allocation.

The Human--AI Substitution Principle (Theorem~\ref{thm:Human--AI Substitution Principle}) identifies the precise pairwise condition under which substitution is cost-reducing: shifting task weight from a human agent to an AI agent strictly reduces risk-adjusted cost if and only if
\[
\tilde{C}'_{ik} < \tilde{C}_{ij}.
\]

The condition makes clear that substitution depends on the difference between the complete effective costs of the two alternatives. A nominal AI cost advantage may be offset by higher reliability, compliance or reputational risk. Conversely, an AI system may become the preferred alternative even when its risk-adjusted advantage is smaller than its direct-cost advantage.

Theorem~\ref{thm:Global Optimal Allocation} extends this comparison to the global allocation problem. Because the objective is linear in allocation weights and the unconstrained feasible set is a simplex, an optimum is attained by assigning the task to an agent with the lowest risk-adjusted effective cost. When that agent is uniquely cheapest, the optimal allocation is concentrated entirely on that agent. When several agents share the minimum cost, allocations among those agents can also be optimal.

Theorem~\ref{thm:Threshold Substitution Rule} establishes the associated pairwise threshold property. When the risk-adjusted cost ranking between a human agent and an AI agent reverses, the cost-minimizing allocation of their combined task weight switches between them. When such a reversal changes the identity of the globally minimum-cost agent, the optimal task allocation can change discontinuously, as characterized in Corollary~\ref{cor:Discontinuous Workforce Transitions}. The resulting discontinuity follows from the linear allocation objective and the geometry of the feasible set. It is not an assumption that organizations must adopt AI abruptly in practice.

Theorem~\ref{thm:AI-Dominant Extreme-Point Allocation} extends the comparison to the condition for an AI-dominant assignment. When the minimum risk-adjusted cost among feasible AI agents is no higher than the minimum among feasible human agents, an AI-only assignment is optimal. Strict inequality makes every human agent ineligible for the set of cost-minimizing assignments; with equality, the AI-only assignment is optimal but not unique. This is a statement about the task and the agents considered in the model, not a prediction that all tasks or all organizational positions will become automated.

The economic significance of these results is that substitution is governed by a relative-cost boundary rather than by technological capability alone. Changes in development expenditure, deployment scale, capability costs or risk can shift this boundary. The same AI technology may therefore be economically attractive in one organization but not in another, even when the tasks and technical capabilities are similar.

\subsection{Organizational Hierarchy and the Location of Substitution}
\label{sec:Organizational Hierarchy and the Location of Substitution}

The HAT framework extends the task-level cost comparison to a hierarchical organization. Its distinctive feature is that organizational depth and the distribution of positions are linked through spans of control. The number of agents at each level depends on the product of the spans above that level. Consequently, a change in organizational depth or span can affect the distribution of positions throughout the hierarchy.

This relationship makes the location of substitution economically important. Replacing a line worker and replacing a manager need not have the same organizational consequences. A managerial position is associated not only with its direct cost but also with the coordination and supervisory structure in which it operates. Substitution at one hierarchical level may therefore affect the cost of maintaining other levels.

Theorem~\ref{thm:Depth--Cost Growth and Dominance} characterizes the increase in human costs with organizational depth under the specified cost structure. The managerial cost escalation means that deeper organizations carry a higher modeled cost burden. The economic mechanism arises from the depth-dependent cost multiplier: as organizational depth increases, the cost of maintaining positions at different managerial levels increases according to the model's escalation function.

Theorem~\ref{thm:Optimal Depth under AI Availability} compares the minimum risk-adjusted cost of a human-only organization with that of an organization in which AI agents are available. It establishes a conditional result that AI availability can weakly reduce optimal organizational depth. The result depends on the assumptions of the theorem, including the layer-redundancy condition. In particular, the organization must be able to redesign its hierarchy, and at every depth above the human-optimal depth one managerial layer must be removable without increasing the minimum AI-enabled risk-adjusted cost.

The distinction between weak and strict reduction matters. The availability of AI does not necessarily eliminate an organizational layer. It creates the possibility of a shallower structure when the cost advantage of AI makes the coordination function of an existing layer redundant. Without this condition, the original organizational depth may remain optimal.

Corollary~\ref{cor:Flattening of Organizations} establishes the corresponding relationship between organizational depth and span of control. If the organization becomes strictly shallower while maintaining comparable execution capacity, the remaining managerial levels must accommodate that capacity through a weakly larger average span, measured by the geometric mean of the spans. This result provides an organizational implication of AI substitution: a reduction in layers can be accompanied by wider supervisory spans.

The result should not be interpreted as an automatic reduction in firm size or employment. The model does not determine output, total demand for labor or the number of tasks performed by the organization. It establishes a change in the cost-minimizing configuration of the modeled hierarchy under the specified capacity comparison.

The framework also allows substitution incentives to differ across managerial levels. Theorem~\ref{thm:Level-Dependent Substitution Likelihood} characterizes how the cost-based incentive to substitute varies with hierarchical position. Since human managerial costs increase with position, the potential cost saving from replacing a human manager with an AI agent may be greater at some levels than at others.

However, cost incentives and substitution feasibility are distinct. Higher-level positions may carry greater responsibility, non-routine decision-making requirements and accountability exposure. These characteristics can raise the risk-adjusted cost of AI or limit the feasibility of transferring the relevant functions.

Corollary~\ref{cor:Middle-Management Vulnerability} identifies the conditions under which substitution vulnerability can peak at an intermediate organizational level. Such a pattern arises when the cost-based incentive to substitute and the feasibility of doing so satisfy the specified single-crossing condition. Intermediate positions can then combine substantial potential cost savings with fewer feasibility constraints than the most senior positions.

This result provides a theoretical explanation for why intermediate managerial roles may be exposed to substitution pressure. It does not establish that middle management is universally the most vulnerable part of the organization. The result depends on the relationship between level-specific cost differences and substitution feasibility. Establishing whether the predicted pattern occurs in practice requires evidence on both dimensions, rather than simply observing changes in middle-management employment.

The broader implication is that AI can affect the organization of work through two connected channels: the substitution of individual agents and the redesign of the hierarchy within which those agents operate. These channels should be distinguished when evaluating the organizational consequences of AI adoption.

\subsection{Skill Differences, Risk and the Conditions for Hybrid Allocation}
\label{sec:Skill Differences, Risk and the Conditions for Hybrid Allocation}

At the operational level, the model distinguishes between the skill of a human agent and the capability of an AI agent. Human costs increase with skill, while AI capability costs increase according to the separate cost function specified in Section~\ref{sec:Human-AI Task Allocation Model}. The Human--AI Cost Asymmetry Assumption allows these cost schedules to have different slopes.

Theorem~\ref{thm:Skill-Adjusted Substitution Condition} and Corollary~\ref{cor:Skill-Dependent Human Protection and Vulnerability Zone} derive a threshold for the comparison between a specified human line worker and an AI agent. Denoting the threshold by $\mu^*$, the relative cost advantage depends on the human worker's skill level, the AI agent's capability, baseline costs, organizational depth and the difference between their risk exposures.

Under the parameter conditions of the theorem, a human agent below the threshold has the lower risk-adjusted effective cost, while an AI agent has the advantage above it. A higher AI risk penalty shifts the threshold in favor of the human agent. A higher human risk term shifts it in favor of AI.

The result provides a conditional explanation of how skill differences influence substitution. It does not imply that highly skilled workers are generally more vulnerable or that lower-skilled workers are necessarily protected. A worker's position relative to the threshold depends on the specific task, the alternative AI capability, the relative cost slopes and the relevant risk parameters. The threshold is therefore an agent- and task-specific comparison, not a general ranking of occupations by skill.

Risk further modifies this relationship. Human risk is represented by an agent-specific term, while AI risk is decomposed into reliability, compliance and reputational components:
\[
R'_{ik} = \omega_1 R^{\mathrm{(rel)}}_{ik} + \omega_2 R^{\mathrm{(comp)}}_{ik} + \omega_3 R^{\mathrm{(rep)}}_{ik}.
\]

The weights reflect the relative importance of these risk dimensions in the relevant organizational environment. Theorem~\ref{thm:Risk-Adjusted Substitution Principle} clarifies how the resulting risk-adjusted costs affect substitution. A high reliability penalty may limit AI deployment in tasks where errors have serious operational consequences. A high compliance penalty may be particularly important in regulated activities, while reputational risk may matter more in activities involving direct interaction with customers or other stakeholders.

The parameter $\lambda$ determines how strongly these risks affect the allocation decision. It follows that two organizations facing similar direct costs may make different allocation choices because they attach different economic importance to reliability, compliance or reputational exposure.

The model's treatment of hybrid allocation requires a further distinction. Theorem~\ref{thm:Constrained Interior Optima} considers the allocation problem when additional linear feasibility constraints restrict the set of available choices. These constraints alter the geometry of the feasible set. At least one optimal solution is an extreme point of the constrained feasible region, but that point need not be an extreme point of the original simplex. Consequently, both human and AI agents may receive positive allocation weights.

The source of such a hybrid allocation must be identified carefully. Risk affects the relative cost ranking of agents, while feasibility constraints affect the set of admissible allocations. In the unconstrained linear benchmark, risk heterogeneity alone does not produce an interior human--AI allocation when one agent has a strictly lower risk-adjusted cost than all the others. The task is assigned to the lowest-cost agent. A mixed allocation is possible when agents tie at the minimum or when additional constraints make a combination of human and AI inputs optimal.

This distinction is central to the economic interpretation of hybrid organizations. The coexistence of humans and AI across an organization does not necessarily mean that both inputs are technologically complementary in the production of a particular task. Different agents may be preferred for different tasks, while constraints may require human involvement in tasks for which AI has a nominal or even risk-adjusted cost advantage. The current framework addresses the allocation problem, but does not establish production complementarities between human and AI inputs.

Theorem~\ref{thm:AI-Dominant Extreme-Point Allocation} provides the corresponding benchmark for full automation. For a given task, an AI-only assignment is optimal when the lowest risk-adjusted AI cost is no higher than the lowest risk-adjusted human cost. Strict full automation requires the AI alternative to be strictly cheaper. Across an organization, full automation of the modeled task set requires this condition to hold for every relevant task, subject to the absence of binding human-participation constraints.

The framework therefore identifies both the economic conditions that favor AI assignment and the circumstances in which human participation remains optimal or necessary. The central implication is that risk and feasibility are not peripheral considerations added after the cost comparison. They are integral to determining the set of economically admissible assignments.

\subsection{Changes in Relative Costs, Stability and Threshold Transitions}
\label{sec:Changes in Relative Costs, Stability and Threshold Transitions}

The comparative-static results examine how the optimal allocation responds to changes in the relative costs of human and AI agents. They provide a link between the static substitution condition and the possibility of organizational adjustment over time.

Theorem~\ref{thm:Monotonic Automation Path} studies a specified path in which AI risk-adjusted costs are non-increasing, while human risk-adjusted costs are held fixed. In the unconstrained allocation problem, and with negligible switching costs, once an AI-only allocation is optimal, there continues to exist an optimal AI-only allocation at every subsequent time, and once the cheapest AI agent is strictly less costly than every human agent, every optimal allocation assigns the task exclusively to AI agents. The result formalizes the consequences of an improvement in the relative cost position of AI.

The economic importance of this result lies in its conditional monotonicity. As AI becomes less costly in risk-adjusted terms, the range of circumstances in which an AI-only allocation is optimal expands. However, the theorem does not explain the sources of declining AI costs. Improvements in computing, model design, reliability, or governance may provide possible economic channels, but technological progress and risk-reduction investment are not endogenously determined in the present framework.

Corollary~\ref{cor:Irreversibility of Automation} establishes a non-reversion result under the same maintained conditions, including negligible switching costs. Once an AI agent becomes strictly less costly in risk-adjusted terms, the allocation does not return to the more expensive human agent as long as the human-side costs remain fixed, no binding constraint applies and the AI cost advantage is maintained.

This is a conditional irreversibility result. It does not mean that actual organizational decisions cannot be reversed. Changes in worker skills, AI reliability, regulation, accountability requirements or switching costs may alter the effective cost ranking. The result instead identifies the circumstances under which a reversal cannot occur within the specified cost path.

Theorem~\ref{thm:Perturbation Stability} examines the stability of the optimal allocation under small changes in model parameters. When one agent has a strictly lower risk-adjusted cost and that cost ranking remains unchanged, small perturbations do not change the identity of the optimal agent. The allocation is therefore locally stable within regions of the parameter space where the cost ordering is preserved.

A change in allocation occurs when a parameter change moves the organization across a boundary at which the relevant risk-adjusted costs become equal and their ordering reverses. In the unconstrained linear benchmark, the resulting switch can be discontinuous: the allocation moves from one extreme point to another rather than passing smoothly through an interior combination.

The economic interpretation is one of stability within cost-ranking regions and adjustment at threshold crossings. The result helps explain why the cost-minimizing assignment may remain unchanged over a range of parameter values even as underlying costs evolve. However, the model does not establish that organizational restructuring must occur immediately when a threshold is crossed. Implementation delays, employment contracts, organizational routines and adjustment costs may affect the timing of actual adoption.

Taken together, Theorems~\ref{thm:Monotonic Automation Path} and \ref{thm:Perturbation Stability} address two distinct questions. The first examines how allocation changes along a specified path of improving AI costs. The second identifies the conditions under which an existing allocation remains stable or changes as the relevant cost ranking reverses. Both results follow from the model's allocation structure; neither constitutes a theory of endogenous technological progress or the diffusion of AI across firms.

\subsection{Strategic Adaptation by Human and AI Agents}
\label{sec:Strategic Adaptation by Human and AI Agents}

The strategic extension recognizes that human and AI agents need not remain passive when their relative positions in the allocation problem change. Potential substitution creates incentives for both sides to improve their competitive position.

Human adaptation effort can represent upskilling, task redesign or process improvement. AI effort can represent capability improvement, reliability enhancement, monitoring or risk mitigation. These choices affect the agents' realized costs and their prospects of receiving the task in the subsequent allocation stage.

Theorem~\ref{thm:Strategic Substitution Equilibrium} establishes the existence of an equilibrium in the specified two-agent effort-and-allocation game when each side's payoff is quasiconcave in its own effort. The economic mechanism involves a trade-off between the anticipated benefit of improving one's allocation prospects and the cost of undertaking additional effort. In the marginal conditions of the theorem, the return to additional effort for either agent is proportional to $p_j p'_k = p_j(1-p_j)$, where $p_j$ is the human agent's contest-success measure. Other things being equal, incentives to adapt are therefore strongest when the human and AI risk-adjusted costs are close and weaken when either agent is far ahead.

This extension adds a strategic dimension to the HAT framework. The initial cost comparison determines which agent is more attractive under existing conditions, but the possibility of future substitution can itself influence adaptation choices. The resulting allocation is therefore shaped not only by existing differences in human and AI costs but also by the efforts undertaken in response to those differences.

The contribution of the strategic extension should nevertheless be interpreted within its specified scope. It considers a human agent and an AI agent competing for a task at a given organizational level. It does not provide a dynamic model of lifetime learning, occupational mobility, educational investment, the endogenous innovation decisions of AI developers or long-run human capital accumulation. Moreover, equilibrium existence does not by itself establish uniqueness, actual adaptation magnitudes or the speed at which organizations respond to substitution pressure.

The strategic result therefore complements the static allocation framework by identifying a channel through which substitution incentives can induce adaptation. A broader account of workforce adjustment would require a dynamic extension linking individual investment decisions to organizational change over time.

\subsection{Empirical Implications and Measurement}
\label{sec:Empirical Implications and Measurement}

The HAT framework generates several empirically relevant hypotheses concerning task allocation, hierarchy, risk and adaptation. These hypotheses are implications of the model's assumptions rather than findings established by the theoretical analysis. Their evaluation requires data that connect actual task assignments to the relevant costs, capabilities, risk exposures and organizational constraints.

\paragraph{Risk-adjusted substitution thresholds.} Theorems~\ref{thm:Human--AI Substitution Principle}--\ref{thm:Threshold Substitution Rule} and \ref{thm:Perturbation Stability} imply that task assignment should depend on the difference between the risk-adjusted costs of human and AI alternatives. An empirical investigation could examine whether the probability of assigning a task to AI increases when its risk-adjusted cost falls below the corresponding human cost. It should distinguish nominal cost advantages from advantages that remain after risk is included.

\paragraph{Organizational depth and span of control.} Theorems~\ref{thm:Depth--Cost Growth and Dominance} and \ref{thm:Optimal Depth under AI Availability} and Corollary~\ref{cor:Flattening of Organizations} suggest that AI availability may be associated with a reduction in organizational depth when the layer-redundancy condition is satisfied. If execution capacity is maintained, a shallower structure should be accompanied by a weakly larger average span of control, measured by the geometric mean of the spans. Empirical tests must separate AI-related organizational redesign from changes caused by firm growth, contraction, mergers or unrelated restructuring.

\paragraph{Variation in substitution across hierarchical levels.} Theorem~\ref{thm:Level-Dependent Substitution Likelihood} and Corollary~\ref{cor:Middle-Management Vulnerability} suggest that substitution pressure may vary systematically across organizational levels. The relevant empirical question is whether intermediate roles combine substantial cost-based substitution incentives with relatively favorable substitution feasibility. A decline in middle-management employment alone would not establish this mechanism. Evidence is also required on the tasks transferred, the associated cost savings and the constraints governing substitution.

\paragraph{Skill and risk heterogeneity.} Theorem~\ref{thm:Skill-Adjusted Substitution Condition} and Corollary~\ref{cor:Skill-Dependent Human Protection and Vulnerability Zone} imply that the relationship between human skill and AI substitution depends on the parameters determining the threshold. Theorem~\ref{thm:Risk-Adjusted Substitution Principle} further implies that the components of AI risk can change the relative cost ranking. Empirical work should therefore examine skill, AI capability and task-specific risk jointly rather than use education, occupation or wages as stand-alone indicators of substitution vulnerability.

\paragraph{Feasibility constraints and hybrid allocation.} Theorem~\ref{thm:Constrained Interior Optima} suggests that minimum-human requirements, regulatory conditions, capacity restrictions and skill requirements can affect the optimal composition of an allocation. Empirical research should distinguish a mixed allocation within a particular task from the coexistence of human and AI agents across different tasks. These are economically different organizational outcomes.

\paragraph{Deployment scale and cost amortization.} The AI cost structure implies that a larger number of deployments reduces the fixed training-cost component attributable to each deployment, holding the other cost components constant. An empirical test could examine whether deployment scale is associated with lower average development cost per organizational application. However, deployment scale may itself respond to expected profitability, so the relationship should not automatically be interpreted as a causal effect.

\paragraph{Strategic adaptation.} Theorem~\ref{thm:Strategic Substitution Equilibrium} motivates examining whether the marginal incentives for investment in human upskilling and task redesign, and in AI capability improvement and risk mitigation, are strongest, other things being equal, where the human--AI cost ranking is closely contested, rather than where one agent's cost advantage is largest. These responses should be distinguished from long-term changes in education, career choice and occupational mobility, which are outside the current model.

Table~\ref{tab:measures} summarizes possible empirical measures for the main HAT model components. These are candidate proxies rather than direct observations of all theoretical parameters.

\begin{table}[t!]
\centering
\caption{Possible empirical measures for HAT model components.}\vspace{-5mm}
\label{tab:measures}
\small
\renewcommand{\arraystretch}{1.2} 
\begin{tcolorbox}[enhanced, sharp corners, boxrule=0pt, colback=white, frame hidden]
\rowcolors{1}{gray!20}{gray!0}
\begin{tabular}{>{\raggedright\arraybackslash}p{2.6cm}>{\raggedright\arraybackslash}p{5.3cm}>{\raggedright\arraybackslash}p{5.4cm}}
\toprule
\textbf{Model component} &
\textbf{Possible empirical measure} &
\textbf{Measurement qualification} \\
\midrule
Organizational depth and spans &
Organizational charts, reporting lines, number of managerial layers, manager-to-subordinate ratios and headcount by level. &
AI-related restructuring must be separated from changes in firm size and other organizational adjustments. \\
\addlinespace
Human baseline and skill costs &
Payroll, benefits, recruitment and training costs, linked to role, experience, qualifications and task complexity. &
Compensation is a proxy for modeled cost; managerial capability and skill are not directly observed through wages alone. \\
\addlinespace
AI baseline and deployment scale &
Development and integration expenditure, vendor and inference charges, maintenance costs, and internal usage or deployment records. &
Fixed development costs must be distinguished from recurring operating costs, and the number of deployments over which fixed expenditure is amortized must be identified. \\
\addlinespace
AI capability and human skill &
Task-specific performance measures, evaluation scores, relevant qualifications, experience and task-completion assessments. &
Measures must be comparable for the task; benchmark performance alone does not establish organizational suitability. \\
\addlinespace
Human and AI risk &
Human error, absence and turnover records; AI reliability assessments, audit findings, compliance incidents and reputational indicators. &
Risks should be measured in the relevant task and institutional context rather than inferred from the agent category alone. \\
\addlinespace
Risk sensitivity and risk weights &
Organizational risk assessments, governance protocols, expert elicitation and differences in task-domain exposure. &
These are latent organizational parameters requiring defensible measurement or calibration. \\
\addlinespace
Feasibility constraints &
Regulatory requirements, approval rules, capacity limits, mandatory human-review provisions and skill restrictions. &
Legal or technical restrictions should be distinguished from voluntary managerial preferences. \\
\addlinespace
Strategic adaptation &
Training participation, task redesign, process improvement, model updates and investments in validation or monitoring. &
These measures do not, by themselves, identify long-term career or innovation responses. \\
\bottomrule
\end{tabular}
\end{tcolorbox}
\end{table}

A credible empirical application would require more than information on whether a firm has adopted AI. Ideally, the data would identify the task, the human and AI alternatives available, the associated direct costs and risks, the allocation selected and any binding organizational constraints. Firm-level adoption indicators, occupational exposure measures and changes in employee headcounts can provide useful background, but are insufficient on their own to identify the mechanisms developed in HAT.

The framework therefore offers a basis for organizing empirical research around the economic determinants of task assignment rather than treating AI adoption as a single firm-level outcome. Its empirical contribution lies in generating hypotheses about the interaction between costs, risk and organizational structure, not in claiming that these relationships have already been established by the model.

\subsection{Limitations and Directions for Future Research}
\label{sec:Limitations and Directions for Future Research}

The interpretation of the HAT results is bounded by the model's maintained assumptions and reduced-form representation of organizational behavior. These limitations identify both the scope of the present contribution and directions for extending the analysis.

\paragraph{Maintained human--AI cost asymmetry.} The relative growth of human skill costs, AI capability costs and coordination costs is imposed through the Human--AI Cost Asymmetry Assumption. The framework does not derive these differences from an endogenous theory of technological progress or human capital accumulation. Future research could examine environments in which the cost of improving AI capability rises more rapidly, human skill acquisition becomes less expensive, or AI coordination and oversight costs increase with organizational complexity.

\paragraph{Exogenous deployment scale.} The number of AI deployments is taken as given. The model captures the amortization of fixed training expenditure but does not determine the organization's optimal deployment scale or investment in AI. An extension that endogenizes these decisions could examine the interaction between task allocation, development expenditure and the scale of adoption.

\paragraph{Linear allocation objective.} The linear objective over the simplex produces an extreme-point solution in the unconstrained benchmark. The framework therefore does not represent technological complementarities, task sequencing, congestion or team effects that might make a mixed human--AI allocation productive even in the absence of an external constraint. A nonlinear production or coordination structure would be necessary to study these mechanisms.

\paragraph{Reduced-form representation of hierarchy.} The model derives the number of positions from spans of control and allows costs to vary with organizational depth. It does not explicitly model managerial knowledge acquisition, information transmission, problem referral or the allocation of formal decision rights. Integrating these mechanisms could explain more fully when AI changes the need for supervision and coordination rather than simply changing the relative cost of occupying a position.

\paragraph{Exogenous risk and governance parameters.} Reliability, compliance and reputational risk, together with their weights and organizational risk sensitivity, enter as parameters. The framework does not determine how monitoring, regulation, liability, insurance, auditing or stakeholder responses generate these parameters. Future research could endogenize governance and accountability choices, thereby connecting risk reduction to the organization's decision to automate.

\paragraph{Task-level and partial-equilibrium scope.} The analysis takes the task and the set of feasible agents as given. It does not model task creation, task destruction, occupational reorganization, changes in output scale or economy-wide labor-market adjustment. Linking the HAT allocation problem to a task-based production model could connect internal assignment choices to employment, wages, productivity and the creation of new work.

\paragraph{Limited strategic horizon.} The effort-and-allocation extension examines competition between one human and one AI agent at a given organizational level. It does not model lifetime skill acquisition, career choices, occupational mobility or the endogenous innovation decisions of AI developers. A dynamic framework could examine how human capital accumulation, AI improvement and organizational redesign interact over time.

These extensions would broaden the range of economic questions that the HAT framework can address. They should, however, be distinguished from the contribution of the present paper. The current model isolates the organizational assignment of a given task and establishes how relative human and AI costs, deployment amortization, risk and hierarchical position determine the cost-minimizing allocation under specified feasibility conditions.

The central economic insight is that AI substitution is not determined by technical capability or nominal operating cost in isolation. It depends on how the relative cost of human expertise and AI capability interacts with deployment scale, organizational structure and risk. The framework further shows how these task-level choices can influence the cost-minimizing hierarchy and generate incentives for strategic adaptation. By bringing these mechanisms into a unified allocation problem, HAT provides a tractable theoretical basis for analyzing the organizational economics of human--AI substitution without extending its conclusions beyond the assumptions on which they rest.

\section{Conclusions}
\label{sec:Conclusions}
Identifying the conditions under which tasks are reassigned from human agents to AI agents is a foundational challenge for today's organizations. This paper proposes a formal Human–AI Task Allocation (HAT) framework, moving the discourse from speculative forecasting toward a structural, testable theory of organizational design.

HAT provides a tractable analytical framework that derives the Human--AI Substitution Principle within a hierarchical organization from explicit cost and risk primitives, complementing recent models of AI in knowledge hierarchies \citep{ide2025,xu2025genai}. The principle states that shifting task weight from a human to an AI agent strictly reduces cost if and only if the AI agent has lower risk-adjusted effective cost for the task. This condition is substantively meaningful when human and AI costs scale differently with skill, capability, hierarchy, deployment scale, and risk. Building on this principle, the paper derives several organizational implications: substitution may occur discontinuously at thresholds; hybrid human–AI structures can arise across heterogeneous tasks when their risk-adjusted cost rankings differ; middle-management roles may face elevated vulnerability under identifiable single-crossing conditions; high skill can either protect or expose workers depending on the economics of expertise; and AI adoption can reduce hierarchy depth and widen spans of control.

By integrating automation economics, organizational hierarchy, human--AI interaction, risk governance, and strategic adaptation, the HAT framework provides a tractable foundation for studying AI-driven organizational transformation and a set of predictions that can be tested with organizational and labor-market data. This work is a significant step toward the formal design and analysis of AI-driven organizations, where technology, risk, and structure are linked inextricably.


\section*{Acknowledgment}
This research was done while BB was an Infosys Chair Professor (visiting) at Chennai Mathematical Institute, Chennai, Tamil Nadu 603103, India.

\section*{Authors' Contributions and Responsibility Statement}
BB directed the research and developed the Human–AI Task Allocation (HAT) model (Section~\ref{sec:Human-AI Task Allocation Model}). The theorems and corollaries were jointly derived by BB and SS. BB and SB prepared the draft of the manuscript. All authors reviewed and approved the manuscript. The authors employed calculators, language editing software, Internet search engines, and AI tools during the production of this manuscript. The authors maintain ultimate responsibility for the published work. The views, opinions, and conclusions expressed in this manuscript are solely those of the authors and do not reflect the policy, position, or endorsement of any affiliated institution, funding agency, or organization. The authors declare no competing financial or non-financial interests.

\bibliographystyle{abbrvnat} 
\bibliography{dms_limited} 

\newpage
\clearpage 
\thispagestyle{empty} 
\vspace*{\fill}
\begin{center}
    \textit{This page is intentionally left blank.}
\end{center}
\vspace*{\fill}
\clearpage

\appendix
\section{Appendix: Proofs of Theorems and Corollaries}
\makeatletter
\setcounter{tcb@cnt@styledtheorem}{0}
\setcounter{tcb@cnt@styledcorollary}{0}
\def\theHtcb@cnt@styledtheorem{A\arabic{tcb@cnt@styledtheorem}}
\def\theHtcb@cnt@styledcorollary{A\arabic{tcb@cnt@styledcorollary}}
\makeatother


\begin{styledtheorem}{Normalization, Feasibility, and Simplex Structure}{app:Normalization, Feasibility, and Simplex Structure}
\begin{equation*}
q_{ij} \ge 0, \quad q'_{ik} \ge 0, \quad
\sum_{i=1}^{D+1}\sum_{j=1}^{e_i} q_{ij}
+ \sum_{i=1}^{D+1}\sum_{k=1}^{e'_i} q'_{ik} = 1.
\end{equation*}
The allocation space
\begin{equation*}
\mathcal{Q} = \left\{(q,q') : q_{ij} \ge 0,\ q'_{ik} \ge 0,\
\sum_{i,j} q_{ij} + \sum_{i,k} q'_{ik} = 1\right\}
\end{equation*}
forms a simplex. Moreover, any valid allocation is a convex combination of extreme points, each corresponding to assigning the entire task to a single agent at a specific level.
\end{styledtheorem}

\begin{proof}
From~\eqref{eq:joint_cost_expanded}, the allocation weights are defined as $q_{ij}, q'_{ik} \ge 0$.

Furthermore, from~\eqref{eq:feasibility},
\[
\sum_{i,j} q_{ij} + \sum_{i,k} q'_{ik} = 1.
\]
Thus $\mathcal{Q}$ is the standard simplex in $\mathbb{R}^n$ defined by non-negativity and unit-sum constraints. Its extreme points are the unit vectors $\mathbf{1}_{ij}$ and $\mathbf{1}'_{ik}$, each putting weight one on a single agent $(i,j)$ or $(i,k)$ and zero on all other agents, that is, assigning the entire task to that agent. Every $(q,q')\in\mathcal{Q}$ is the convex combination $\sum_{i,j} q_{ij}\,\mathbf{1}_{ij}+\sum_{i,k} q'_{ik}\,\mathbf{1}'_{ik}$ of these extreme points, because its coefficients are non-negative and sum to one.
\end{proof}

\begin{styledtheorem}{Skill-Weighted Cost Decomposition}{app:Skill-Weighted Cost Decomposition}
Under the common AI baseline-cost specification $\textup{Min\$}'_k = \textup{Min\$}'$ for all AI agents $k$ at level $D+1$, the total risk-adjusted cost $\mathcal{C}$ decomposes as
\begin{align*}
\mathcal{C} &=
\sum_{i=1}^{D} \sum_{j=1}^{e_i} \tilde{C}_{ij}\, q_{ij}
+
\sum_{i=1}^{D} \sum_{k=1}^{e'_i} \tilde{C}'_{ik}\, q'_{ik}
\\
&\quad
+ \alpha + \beta \cdot D
  \sum_{j=1}^{e_{D+1}} \bigl(\mu_{D+1,j}\, q_{D+1,j}\bigr)
+ \lambda \sum_{j=1}^{e_{D+1}} \bigl(R_{D+1,j}\, q_{D+1,j}\bigr)
\\
&\quad
+ \alpha' + \beta' \cdot D
  \sum_{k=1}^{e'_{D+1}} \bigl(\mu'_{D+1,k}\, q'_{D+1,k}\bigr)
+ \lambda \sum_{k=1}^{e'_{D+1}} \bigl(R'_{D+1,k}\, q'_{D+1,k}\bigr),
\end{align*}
where
\[
\alpha = \textup{Min\$}\sum_{j=1}^{e_{D+1}} q_{D+1,j},
\quad
\beta = \tfrac{1}{2}\Delta\$,
\quad
\alpha' = \textup{Min\$}'\sum_{k=1}^{e'_{D+1}} q'_{D+1,k},
\quad
\beta' = \tfrac{1}{2}\Delta\$'.
\]
Moreover, explicit skill terms enter only through line workers at level \(D+1\):
{\footnotesize
\begin{align*}
\mathcal{C} &=
\underbrace{
\sum_{i=1}^{D}\sum_{j=1}^{e_i} \tilde{C}_{ij}\, q_{ij}
+
\sum_{i=1}^{D}\sum_{k=1}^{e'_i} \tilde{C}'_{ik}\, q'_{ik}
}_{\text{managerial (hierarchical) component}}\nonumber\\
+
&\underbrace{
\alpha + \beta \cdot D\!
\sum_{j=1}^{e_{D+1}}\!(\mu_{D+1,j}\, q_{D+1,j})
+ \alpha' + \beta' \cdot D\!
\sum_{k=1}^{e'_{D+1}}\!(\mu'_{D+1,k}\, q'_{D+1,k})
+ \lambda\Bigl[\sum_j R_{D+1,j} q_{D+1,j}
             + \sum_k R'_{D+1,k} q'_{D+1,k}\Bigr]
}_{\text{skill- and risk-dependent (line worker) component}}.
\end{align*}}
In particular, $\mathcal{C}$ depends explicitly on the skill variables $(\mu_{D+1,j},\, \mu'_{D+1,k})$ only through the line workers (level $D+1$); the managerial costs $C_i$ and $C'_i$ for $i \le D$ do not contain these variables directly and depend only on $D$ and on the representative line-worker costs $C_{D+1}$ and $C'_{D+1}$.
\end{styledtheorem}

\begin{proof}[Proof of cost decomposition]
From~\eqref{eq:joint_cost_expanded}, the total risk-adjusted cost is
\[
\mathcal{C} =
\sum_{i=1}^{D+1} \sum_{j=1}^{e_i} \tilde{C}_{ij}\, q_{ij}
+
\sum_{i=1}^{D+1} \sum_{k=1}^{e'_i} \tilde{C}'_{ik}\, q'_{ik}.
\]

Separating managerial levels ($i = 1,\dots,D$) from line workers ($i = D+1$):
\begin{align*}
\mathcal{C} &=
\sum_{i=1}^{D}\sum_{j=1}^{e_i} \tilde{C}_{ij}\, q_{ij}
+
\sum_{i=1}^{D}\sum_{k=1}^{e'_i} \tilde{C}'_{ik}\, q'_{ik}
\\
&\quad
+\sum_{j=1}^{e_{D+1}} \tilde{C}_{D+1,j}\, q_{D+1,j}
+\sum_{k=1}^{e'_{D+1}} \tilde{C}'_{D+1,k}\, q'_{D+1,k}.
\end{align*}

From~\eqref{eq:human_effective_cost}, $\tilde{C}_{D+1,j} = C_{D+1,j} + \lambda R_{D+1,j}$, and
from~\eqref{eq:line_worker_cost}, $C_{D+1,j} = \textup{Min\$} + \tfrac{1}{2}\Delta\$ \cdot D \cdot \mu_{D+1,j}$. Thus,
\begin{align*}
\sum_{j=1}^{e_{D+1}} \tilde{C}_{D+1,j}\, q_{D+1,j}
&= \textup{Min\$}\sum_{j} q_{D+1,j}
+ \tfrac{1}{2}\Delta\$ \cdot D\sum_{j}
  (\mu_{D+1,j}\, q_{D+1,j})
+ \lambda\sum_{j}(R_{D+1,j}\, q_{D+1,j}).
\end{align*}

Defining $\alpha = \textup{Min\$}\sum_j q_{D+1,j}$ and
$\beta = \tfrac{1}{2}\Delta\$$ gives
\[
\sum_{j=1}^{e_{D+1}} \tilde{C}_{D+1,j}\, q_{D+1,j}
= \alpha + \beta \cdot D \sum_{j}(\mu_{D+1,j}\, q_{D+1,j})
+ \lambda\sum_{j}(R_{D+1,j}\, q_{D+1,j}).
\]

From~\eqref{eq:AI_effective_cost}, $\tilde{C}'_{D+1,k} = C'_{D+1,k} + \lambda R'_{D+1,k}$, and from~\eqref{eq:AI_lineworker_parametric} under the common baseline $\textup{Min\$}'_k = \textup{Min\$}'$, $C'_{D+1,k} = \textup{Min\$}' + \tfrac{1}{2}\Delta\$' \cdot D \cdot \mu'_{D+1,k}$. Thus,
\begin{align*}
\sum_{k=1}^{e'_{D+1}} \tilde{C}'_{D+1,k}\, q'_{D+1,k}
&= \textup{Min\$}'\sum_{k} q'_{D+1,k}
+ \tfrac{1}{2}\Delta\$' \cdot D\sum_{k}
  (\mu'_{D+1,k}\, q'_{D+1,k})
+ \lambda\sum_{k}(R'_{D+1,k}\, q'_{D+1,k}).
\end{align*}

Defining $\alpha' = \textup{Min\$}'\sum_k q'_{D+1,k}$ and $\beta' = \tfrac{1}{2}\Delta\$'$ and substituting yields the stated decomposition.
\end{proof}

\begin{proof}[Proof of skill-linearity restriction to line workers]
From the cost decomposition above, the total cost is
\begin{align*}
\mathcal{C} &=
\sum_{i=1}^{D}\sum_{j=1}^{e_i} \tilde{C}_{ij}\, q_{ij}
+
\sum_{i=1}^{D}\sum_{k=1}^{e'_i} \tilde{C}'_{ik}\, q'_{ik}\nonumber\\
&+ \alpha + \beta \cdot D\sum_{j}(\mu_{D+1,j}\, q_{D+1,j})
+ \lambda\sum_j(R_{D+1,j}\, q_{D+1,j})\nonumber\\
&+ \alpha' + \beta' \cdot D\sum_{k}(\mu'_{D+1,k}\, q'_{D+1,k})
+ \lambda\sum_k(R'_{D+1,k}\, q'_{D+1,k}).
\end{align*}

Grouping terms by whether they involve skill variables $(\mu_{D+1,j},\, \mu'_{D+1,k})$ yields the stated decomposition. The managerial component involves only $\tilde{C}_{ij}$ and $\tilde{C}'_{ik}$ for $i \le D$, which from~\eqref{eq:manager_cost} and~\eqref{eq:AI_manager_cost} depend on organizational depth $D$ and on the representative line-worker costs $C_{D+1}$ and $C'_{D+1}$, but not directly on individual skill levels. All explicit dependence on skill appears exclusively through the line-worker terms $\sum_j(\mu_{D+1,j}\, q_{D+1,j})$ and
$\sum_k(\mu'_{D+1,k}\, q'_{D+1,k})$, confirming that skill enters explicitly only at level $D+1$.
\end{proof}

\begin{styledtheorem}{Depth--Cost Growth and Dominance}{app:Depth--Cost Growth and Dominance}
Human organizational costs satisfy
\[
C_i = C_{D+1}(1 + r_0 D)^{D-i+1}.
\]
For $r_0 > 0$ and any fixed managerial level $i$, $C_i$ is strictly increasing in $D$ (for $D \ge i$) whenever the line-worker cost $C_{D+1}$ is held fixed or is non-decreasing in $D$, and the factor $(1 + r_0 D)^D$ satisfies
\[
\lim_{D \to \infty}
\frac{(1 + r_0 D)^D}{a^D} = \infty
\quad \text{for all } a > 0.
\]
Moreover, higher managerial levels have higher per-agent costs relative to the line-worker reference cost: for $i < D+1$,
\[
\frac{C_i}{C_{D+1}} = (1 + r_0 D)^{D-i+1}.
\]
If there exists $\epsilon > 0$ such that $\sum_{j=1}^{e_i} q_{ij} \ge \epsilon$ for a managerial level $i \le D$, then
\[
\sum_{j=1}^{e_i} C_{ij}\, q_{ij}
\;\ge\; \epsilon\, C_{D+1}(1 + r_0 D)^{D-i+1}.
\]
\end{styledtheorem}

\begin{proof}
From~\eqref{eq:manager_cost},
\[
C_i = C_{D+1}(1 + r_0 D)^{D-i+1}.
\]
Fix a managerial level $i$ and let $D \ge i$, so that $D-i+1 \ge 1$.

\emph{Monotonicity.} If $D' > D \ge i$, then $1 + r_0 D' > 1 + r_0 D > 1$ and $D'-i+1 > D-i+1 \ge 1$, so
\[
(1 + r_0 D')^{D'-i+1} > (1 + r_0 D)^{D-i+1},
\]
because both the base and the exponent increase and the base exceeds one. Hence $C_i = C_{D+1}(1 + r_0 D)^{D-i+1}$ is strictly increasing in $D$ whenever $C_{D+1} > 0$ is held fixed or is non-decreasing in $D$ (as it is in~\eqref{eq:line_worker_cost} when line-worker skill is held fixed).

\emph{Limit.} Let $a > 0$. Since $(1 + r_0 D)/a \to \infty$, there is $D_0$ such that $(1 + r_0 D)/a \ge 2$ for all $D \ge D_0$, and then
\[
\frac{(1 + r_0 D)^D}{a^D} = \left(\frac{1 + r_0 D}{a}\right)^{D} \ge 2^D \to \infty.
\]
Moreover, $C_i/a^D = C_{D+1}\,(1 + r_0 D)^{-(i-1)}\,(1 + r_0 D)^D/a^D$. When $C_{D+1}$ is fixed or non-decreasing in $D$, it is bounded below by a positive constant, and $2^D$ grows faster than any power of $D$, so $C_i/a^D \to \infty$ for every $a > 0$. Hence $C_i$ grows super-exponentially in $D$.
\end{proof}

\begin{proof}[Proof of hierarchy dominance and bottleneck effects]
From~\eqref{eq:manager_cost},
$C_i = C_{D+1}(1 + r_0 D)^{D-i+1}$.
For a managerial level $i \le D$, $C_{ij} = C_i$ for every human agent $j$ at that level, so
\[
\sum_{j=1}^{e_i} C_{ij}\, q_{ij}
\;=\; C_i \sum_{j=1}^{e_i} q_{ij}
\;\ge\; \epsilon\, C_{D+1}(1 + r_0 D)^{D-i+1},
\]
where the inequality uses the assumption
$\sum_j q_{ij} \ge \epsilon$.
\end{proof}


\begin{styledtheorem}{Optimal Depth under AI Availability}{app:Optimal Depth under AI Availability}
Let \(\mathcal{C}(D)\) denote the minimum risk-adjusted cost of accomplishing a task in an organization of depth \(D\) using only human agents, and let \(\mathcal{C}^{(AI)}(D)\) denote the corresponding minimum risk-adjusted cost when AI agents are available. Suppose:
\begin{enumerate}
\item Human costs grow super-exponentially with depth:
\[
C_i(D)=C_{D+1}(1+r_0D)^{D-i+1},
\qquad r_0>0,
\]
so that deeper human hierarchies impose increasing nominal managerial costs.

\item AI agents are feasible alternatives for at least some managerial roles and have weakly lower risk-adjusted effective costs than the corresponding human agents for those roles. The inequality \(r'_0\le r_0\) (Assumption~\ref{ass:asymmetry}(iii)) makes this plausible, since AI managerial cost escalation is then no steeper than human escalation, but it does not by itself imply it.

\item Layer-redundancy condition: for every depth \(D>D^*\), at least one managerial layer can be removed without increasing the minimum AI-enabled risk-adjusted cost, that is,
\[
\mathcal{C}^{(AI)}(D-1)\le \mathcal{C}^{(AI)}(D).
\]
\end{enumerate}
Assume that \(\mathcal{C}(D)\) and \(\mathcal{C}^{(AI)}(D)\) attain their minima over \(D\). Let
\[
D^*=\min\arg\min_D \mathcal{C}(D),
\qquad
D^{(AI)*}=\min\arg\min_D \mathcal{C}^{(AI)}(D).
\]
Then
\[
D^{(AI)*}\le D^*.
\]
Moreover, if \(D^*>1\) and the layer-redundancy inequality is strict at the human-optimal depth, \(\mathcal{C}^{(AI)}(D^*-1)<\mathcal{C}^{(AI)}(D^*)\), then
\[
D^{(AI)*}<D^*.
\]
\end{styledtheorem}

\begin{proof}
For each fixed depth \(D\), \(\mathcal{C}(D)\) is the minimum risk-adjusted cost over human-only allocations, while \(\mathcal{C}^{(AI)}(D)\) is the minimum risk-adjusted cost over the enlarged feasible set in which AI agents are also available. Since the AI-enabled feasible set contains all human-only allocations as a special case, we have
\[
\mathcal{C}^{(AI)}(D)\le \mathcal{C}(D)
\qquad
\text{for every }D.
\]

By assumption, human hierarchical costs satisfy
\[
C_i(D)=C_{D+1}(1+r_0D)^{D-i+1},
\qquad r_0>0.
\]
Thus, as established in Theorem~\ref{thm:Depth--Cost Growth and Dominance}, human coordination costs increase rapidly with organizational depth. By the definition of the human-only optimum \(D^*\), no larger depth has a lower human-only cost:
\[
\mathcal{C}(D)\ge \mathcal{C}(D^*)
\qquad
\text{for all }D\ge D^*,
\]
with \(D^*=\min\arg\min_D\mathcal{C}(D)\).

Now consider the AI-enabled problem. By item~2 of the theorem, AI agents are feasible alternatives for some managerial roles at weakly lower risk-adjusted cost, and, because \(r'_0\le r_0\) (Assumption~\ref{ass:asymmetry}(iii)), AI managerial cost escalation is no steeper than human escalation. These cost properties provide an economic motivation for considering the possibility of redundant managerial layers. Whether such redundancy occurs is governed by the maintained layer-redundancy condition, which is the property used in the rest of the proof.

By the layer-redundancy condition, for every depth \(D>D^*\), at least one managerial layer can be removed without increasing the AI-enabled minimum cost:
\[
\mathcal{C}^{(AI)}(D-1)\le \mathcal{C}^{(AI)}(D).
\]
Applying this inequality at the depths \(D^*+1,\dots,D\) in turn gives, for every \(D>D^*\),
\[
\mathcal{C}^{(AI)}(D^*)\le \mathcal{C}^{(AI)}(D^*+1)\le \cdots \le \mathcal{C}^{(AI)}(D),
\]
so no depth above \(D^*\) has a lower AI-enabled cost than \(D^*\) itself. Let \(D'\) be any minimizer of \(\mathcal{C}^{(AI)}\); it exists by assumption. If \(D'\le D^*\), then \(D^{(AI)*}\le D'\le D^*\). If \(D'>D^*\), then \(\mathcal{C}^{(AI)}(D^*)\le \mathcal{C}^{(AI)}(D')=\min_D\mathcal{C}^{(AI)}(D)\), so \(D^*\) is also a minimizer and \(D^{(AI)*}\le D^*\). Since
\[
D^{(AI)*}=\min\arg\min_D \mathcal{C}^{(AI)}(D),
\]
in either case
\[
D^{(AI)*}\le D^*.
\]

Now suppose \(D^*>1\) and the layer-redundancy inequality is strict at the human-optimal depth,
\[
\mathcal{C}^{(AI)}(D^*-1)<\mathcal{C}^{(AI)}(D^*).
\]
Then \(D^*\) is not a minimizer of \(\mathcal{C}^{(AI)}\), so \(D^{(AI)*}\ne D^*\), and therefore
\[
D^{(AI)*}<D^*.
\]
This proves the result.
\end{proof}

\begin{styledtheorem}{Human--AI Substitution Principle}{app:Human--AI Substitution Principle}
The total risk-adjusted cost is
\begin{equation*}
\mathcal{C} = \sum_{i=1}^{D+1}\sum_{j=1}^{e_i}
\tilde{C}_{ij}\, q_{ij}
+ \sum_{i=1}^{D+1}\sum_{k=1}^{e'_i}
\tilde{C}'_{ik}\, q'_{ik}.
\end{equation*}
Consider a perturbation $q_{ij} \to q_{ij} - \delta$, $q'_{ik} \to q'_{ik} + \delta$ for $0 < \delta \le q_{ij}$. Then
\[
\Delta\mathcal{C} = \delta\,(\tilde{C}'_{ik} - \tilde{C}_{ij}).
\]
Hence, cost decreases if and only if $\tilde{C}'_{ik} < \tilde{C}_{ij}$, i.e.,
\[
C'_{ik} + \lambda R'_{ik} < C_{ij} + \lambda R_{ij}.
\]
Even when $C'_{ik} < C_{ij}$, a sufficiently large risk differential $\lambda(R'_{ik} - R_{ij}) > 0$ can prevent cost-reducing substitution.
\end{styledtheorem}

\begin{proof}
From the linear form of the risk-adjusted cost function,
\[
\mathcal{C}
=
\sum_{i,j} \widetilde C_{ij}\, q_{ij}
+
\sum_{i,k} \widetilde C'_{ik}\, q'_{ik}.
\]
Consider a feasible perturbation that shifts task weight from human agent \((i,j)\) to AI agent \((i,k)\):
\[
q_{ij}\to q_{ij}-\delta,
\qquad
q'_{ik}\to q'_{ik}+\delta,
\]
where \(0<\delta\le q_{ij}\). This perturbation preserves nonnegativity and the unit-sum constraint defining \(\mathcal Q\).

Under this perturbation,
\[
\mathcal{C}_{\text{new}}
=
\mathcal{C}
-
\delta \widetilde C_{ij}
+
\delta \widetilde C'_{ik}.
\]
Therefore,
\[
\Delta\mathcal{C}
=
\mathcal{C}_{\text{new}}-\mathcal{C}
=
\delta(\widetilde C'_{ik}-\widetilde C_{ij}).
\]
Since \(\delta>0\), the perturbation reduces total risk-adjusted cost if and only if
\[
\widetilde C'_{ik}<\widetilde C_{ij}.
\]

Using the definitions
\[
\widetilde C_{ij}=C_{ij}+\lambda R_{ij},
\qquad
\widetilde C'_{ik}=C'_{ik}+\lambda R'_{ik},
\]
this condition is equivalent to
\[
C'_{ik}+\lambda R'_{ik}
<
C_{ij}+\lambda R_{ij}.
\]

Finally, suppose AI has a nominal cost advantage, so that
\[
C'_{ik}<C_{ij}.
\]
The transfer nevertheless fails to reduce total cost whenever
\[
C'_{ik}+\lambda R'_{ik}
\ge
C_{ij}+\lambda R_{ij},
\]
or equivalently,
\[
\lambda(R'_{ik}-R_{ij})
\ge
C_{ij}-C'_{ik}.
\]
Thus, a sufficiently large AI risk penalty can offset AI's nominal cost advantage and prevent cost-reducing substitution.

The non-tautological content follows from Assumption~\ref{ass:asymmetry}: human and AI costs are not arbitrary comparable constants, but structurally distinct functions of skill, capability, deployment scale, and hierarchy depth. Hence the substitution condition compares two economically different cost-generating mechanisms on risk-adjusted grounds.
\end{proof}

\begin{styledtheorem}{Global Optimal Allocation / Extreme Point Optimality}{app:Global Optimal Allocation}
Let
\[
\mathcal{C}(q, q') =
\sum_{i=1}^{D+1} \sum_{j=1}^{e_i} \tilde{C}_{ij}\, q_{ij}
+
\sum_{i=1}^{D+1} \sum_{k=1}^{e'_i} \tilde{C}'_{ik}\, q'_{ik}
\]
be defined over the simplex
\[
\mathcal{Q} = 
\bigl\{(q, q') : q_{ij} \geq 0,\;
q'_{ik} \geq 0,\;
\textstyle\sum_{i,j} q_{ij} + \sum_{i,k} q'_{ik} = 1\bigr\}.
\]
Let
\[
m =
\min\Bigl\{
\min_{i,j}\tilde{C}_{ij},
\;
\min_{i,k}\tilde{C}'_{ik}
\Bigr\}
\]
denote the lowest risk-adjusted cost among all human and AI agents. Then the minimum value of
\(\mathcal{C}(q,q')\) over \(\mathcal{Q}\) is \(m\), and there exists a global minimizer at an extreme point of \(\mathcal{Q}\). That is, there exists an optimal allocation \((q^*,q'^*)\) such that
\[
q_{i^*j^*}^*=1
\quad \text{or} \quad
q_{i^*k^*}^{\prime *}=1,
\]
with all other allocation weights equal to zero, where the selected agent satisfies
\[
\tilde{C}_{i^*j^*}=m
\quad \text{or} \quad
\tilde{C}'_{i^*k^*}=m.
\]

If the minimum-risk-adjusted-cost agent is unique, then the global minimizer is unique and assigns the entire task to that agent. If multiple human or AI agents attain the same minimum value \(m\), then any allocation supported only on those tied minimum-cost agents is also globally optimal. Equivalently, the set of global minimizers is the convex hull of the extreme points corresponding to agents whose risk-adjusted cost equals \(m\).
\end{styledtheorem}

\begin{proof}
From Theorem~\ref{thm:Normalization, Feasibility, and Simplex Structure}, the feasible set
\(\mathcal{Q}\) is a standard simplex: it is defined by non-negativity constraints and a unit-sum
constraint. The objective function
\[
\mathcal{C}(q,q')
=
\sum_{i=1}^{D+1}\sum_{j=1}^{e_i}\tilde C_{ij}q_{ij}
+
\sum_{i=1}^{D+1}\sum_{k=1}^{e'_i}\tilde C'_{ik}q'_{ik}
\]
is linear in the allocation variables. Therefore, viewing the problem as a linear program over the nonempty bounded polytope \(\mathcal{Q}\), the Fundamental Theorem of Linear Programming implies that an optimal solution exists and that at least one optimal solution is attained at an extreme point of \(\mathcal{Q}\).

We now give the direct characterization of all minimizers. Let
\[
\mathcal{I}_H=\{(H,i,j): i=1,\ldots,D+1,\; j=1,\ldots,e_i\}
\]
denote the set of human-agent indices and
\[
\mathcal{I}_A=\{(A,i,k): i=1,\ldots,D+1,\; k=1,\ldots,e'_i\}
\]
denote the set of AI-agent indices; the labels \(H\) and \(A\) keep the two index sets disjoint. For notational convenience, write the objective compactly as
\[
\mathcal{C}(w)=\sum_{\ell\in \mathcal{I}_H\cup \mathcal{I}_A} c_\ell w_\ell,
\]
where \(w_\ell\ge 0\), \(\sum_\ell w_\ell=1\), and
\[
c_\ell =
\begin{cases}
\tilde C_{ij}, & \text{if } \ell=(H,i,j)\in\mathcal{I}_H,\\[2mm]
\tilde C'_{ik}, & \text{if } \ell=(A,i,k)\in\mathcal{I}_A.
\end{cases}
\]
Thus \(\mathcal{C}(w)\) is a convex combination of the risk-adjusted costs of all human and AI agents.

Let
\[
m=\min_{\ell\in \mathcal{I}_H\cup\mathcal{I}_A} c_\ell
=
\min\Bigl\{
\min_{i,j}\tilde C_{ij},
\;
\min_{i,k}\tilde C'_{ik}
\Bigr\}.
\]
Then, for any feasible allocation \(w\in\mathcal{Q}\),
\[
\mathcal{C}(w)-m
=
\sum_{\ell} (c_\ell-m)w_\ell.
\]
Since \(c_\ell-m\ge 0\) and \(w_\ell\ge 0\) for every \(\ell\), it follows that $\mathcal{C}(w)\ge m$. Hence no feasible allocation can achieve cost below \(m\).

Now let \(\displaystyle\ell^*\in\arg\min_\ell c_\ell\). The allocation \(w^*\) defined by
\[
w_{\ell^*}^*=1,
\qquad
w_\ell^*=0 \quad \text{for all } \ell\ne \ell^*
\]
is an extreme point of \(\mathcal{Q}\) and satisfies
\[
\mathcal{C}(w^*)=c_{\ell^*}=m.
\]
Therefore the global minimum is \(m\), and at least one global minimizer is attained at an extreme
point of \(\mathcal{Q}\). In the original allocation variables, this means that for some minimum-cost
human agent \((i^*,j^*)\) or minimum-cost AI agent \((i^*,k^*)\),
\[
q_{i^*j^*}^*=1
\quad \text{or} \quad
q_{i^*k^*}^{\prime *}=1,
\]
with all other allocation weights equal to zero.

It remains to characterize uniqueness and ties. Define the set of minimum-cost agents
\[
\mathcal{M}=\{\ell\in\mathcal{I}_H\cup\mathcal{I}_A: c_\ell=m\}.
\]
From
\[
\mathcal{C}(w)-m
=
\sum_{\ell} (c_\ell-m)w_\ell,
\]
equality \(\mathcal{C}(w)=m\) holds if and only if
\[
w_\ell=0
\qquad
\text{for all } \ell\notin\mathcal{M}.
\]
Thus an allocation is globally optimal if and only if it places all positive weight on agents whose
risk-adjusted cost equals the minimum value \(m\).

If \(\mathcal{M}\) contains a single agent, then the only feasible allocation supported on
\(\mathcal{M}\) assigns unit weight to that agent, so the global minimizer is unique and is a pure
strategy. If \(\mathcal{M}\) contains multiple agents, then any convex combination of the extreme points corresponding to agents in \(\mathcal{M}\) is globally optimal. Equivalently, the set of global minimizers is precisely the convex hull of the extreme points associated with the tied minimum-cost agents.

Therefore, optimal unconstrained allocation always admits an extreme-point solution, is uniquely pure
when the minimum-risk-adjusted-cost agent is unique, and becomes a face of the simplex when multiple
agents tie for the minimum risk-adjusted cost.
\end{proof}

\begin{styledtheorem}{Threshold Substitution Rule}{app:Threshold Substitution Rule}
Fix a human agent $(i,j)$ and an AI agent $(i,k)$ at the same level $i$, and consider transfers of task weight between them, holding all other weights fixed. Define $\Delta_{ikj} = \tilde{C}'_{ik} - \tilde{C}_{ij}$. Then:\\[2pt]
If $\Delta_{ikj} < 0$: full substitution, i.e., transferring all of the weight $q_{ij}$ from $(i,j)$ to $(i,k)$ minimizes cost.\\
If $\Delta_{ikj} > 0$: no substitution, i.e., no transfer from $(i,j)$ to $(i,k)$ reduces cost.\\
If $\Delta_{ikj} = 0$: pairwise indifference, i.e., every such transfer leaves cost unchanged.
\end{styledtheorem}

\begin{proof}
Consider human agent $(i,j)$ and AI agent $(i,k)$. Perturbing the allocation by transferring weight $0 < \delta \le q_{ij}$ from the human to the AI, holding all other weights fixed:
\[
q_{ij} \to q_{ij} - \delta, \qquad
q'_{ik} \to q'_{ik} + \delta.
\]
The resulting change in total risk-adjusted cost is
\[
\Delta\mathcal{C}
= \delta\tilde{C}'_{ik} - \delta\tilde{C}_{ij}
= \delta(\tilde{C}'_{ik} - \tilde{C}_{ij})
= \delta\,\Delta_{ikj}.
\]

\noindent\textbf{Case 1: $\Delta_{ikj} < 0$.}
$\Delta\mathcal{C} < 0$: transferring weight to AI strictly decreases cost. By linearity, repeated transfer reduces cost until $q_{ij}=0$: full substitution.

\noindent\textbf{Case 2: $\Delta_{ikj} > 0$.}
$\Delta\mathcal{C} > 0$: every such transfer strictly increases cost, so no transfer from this human to this AI agent reduces cost. Relative to the human, any positive weight on the AI agent is suboptimal, holding all other weights fixed. Hence no pairwise substitution occurs.

\noindent\textbf{Case 3: $\Delta_{ikj} = 0$.}
$\Delta\mathcal{C} = 0$: every feasible transfer between the two agents leaves total cost unchanged. If both agents attain the global minimum risk-adjusted cost, any allocation supported exclusively on these two agents is globally optimal.

Thus the threshold substitution rule holds.
\end{proof}

\noindent
\textbf{Remark.} The sharp threshold behavior follows from the linearity of the objective function and the linear simplex constraints: the marginal cost difference \(\Delta_{ikj}\) is constant along the transfer direction from the human agent to the AI agent. Thus, unless the two agents are exactly tied, the optimum lies at a boundary point rather than at an interior mixture. If future extensions introduce nonlinear cost terms, interaction effects, or regularization penalties, the marginal substitution incentive may vary with the allocation weights, and interior human--AI mixtures may arise.

\begin{styledcorollary}{Flattening of Organizations}{app:Flattening of Organizations}
Under the conditions of Theorem~\ref{thm:Optimal Depth under AI Availability}, AI availability weakly reduces optimal hierarchical depth:
\[
D^{(AI)*}\le D^*.
\]
If the inequality is strict and the AI-enabled organization, with spans $s^{(AI)}_\ell$, at least maintains the original execution capacity $E=\prod_{\ell=1}^{D^*}s_\ell$, that is, $\prod_{\ell=1}^{D^{(AI)*}}s^{(AI)}_\ell\ge E$, then the reduced-depth organization requires a weakly larger geometric-mean span of control (strictly larger if $E>1$). Thus, AI adoption can produce flatter hierarchies with wider average managerial spans.
\end{styledcorollary}

\begin{proof}
From Theorem~\ref{thm:Optimal Depth under AI Availability}, the optimal depth under AI availability satisfies $D^{(AI)*}\le D^*$. Thus AI availability weakly reduces optimal hierarchical depth, with strict flattening whenever $D^{(AI)*}<D^*$.

We now characterize the implication for spans of control when execution capacity is maintained. By~\eqref{eq:level_size}, the number of agents at the line-worker level is
\[
e_{D+1}=\prod_{\ell=1}^{D}s_\ell.
\]
Let
\[
E=\prod_{\ell=1}^{D^*}s_\ell
\]
denote the line-worker capacity of the original hierarchy. Suppose that after AI adoption the organization at least maintains its execution capacity, so that
\[
E^{(AI)}:=\prod_{\ell=1}^{D^{(AI)*}}s^{(AI)}_\ell \ge E.
\]
Define the geometric mean span before and after AI adoption as
\[
\bar{s}
=
\left(\prod_{\ell=1}^{D^*}s_\ell\right)^{1/D^*}
=
E^{1/D^*},
\qquad\textup{and}\qquad
\bar{s}^{(AI)}
=
\left(E^{(AI)}\right)^{1/D^{(AI)*}}
\ge
E^{1/D^{(AI)*}}.
\]
Since \(E\) is a number of line-worker positions, \(E\ge1\). If \(D^{(AI)*}<D^*\) and \(E>1\), then
\[
E^{1/D^{(AI)*}}>E^{1/D^*},
\]
so
\[
\bar{s}^{(AI)}>\bar{s}.
\]
If \(E=1\), then \(\bar{s}=1\le\bar{s}^{(AI)}\). Hence, when AI adoption strictly reduces optimal depth while at least maintaining execution capacity, the geometric-mean span of control is weakly larger, and strictly larger if \(E>1\). Therefore AI adoption can generate flatter hierarchies with wider average managerial spans.
\end{proof}

\noindent
\textbf{Remark.} Theorem~\ref{thm:Optimal Depth under AI Availability} compares the human-only optimum with the AI-enabled optimum after the organization is allowed to reorganize its hierarchy following AI substitution. That is, \(D^{(AI)*}\) is chosen from the feasible set of post-adoption organizational depths, rather than being constrained to preserve the original hierarchy. From Theorem~\ref{thm:Optimal Depth under AI Availability}, the optimal depth under AI availability satisfies $D^{(AI)*}\le D^*$. Thus AI availability weakly reduces optimal hierarchical depth, with strict flattening whenever $D^{(AI)*}<D^*$.


\begin{styledtheorem}{Level-Dependent Substitution Pressure}{app:Level-Dependent Substitution Likelihood}
Let $C_i = C_{D+1}(1 + r_0 D)^{D-i+1}$, $r_0 > 0$. Then $C_i$ is strictly decreasing in $i$, so, holding AI cost and the risk terms $R_{ij}$ and $R'_{ik}$ fixed, the cost-based incentive for substitution (cost reduction from replacing a human by an AI agent) increases as $i \to 1$ (higher levels in the hierarchy). Under Assumption~\ref{ass:asymmetry}(iii), $r'_0 \le r_0$, so AI managerial cost escalation is no steeper than human escalation; if in addition $C'_{D+1} \le C_{D+1}$, the nominal gap $C_i - C'_i$ does not narrow toward the top, reinforcing this gradient.
\end{styledtheorem}

\begin{proof}
From~\eqref{eq:manager_cost},
\[
C_i = C_{D+1}(1+r_0D)^{D-i+1}.
\]
Let \(i_1<i_2\). Then
\[
D-i_1+1>D-i_2+1.
\]
Since \(1+r_0D>1\) for \(r_0>0\), it follows that
\[
(1+r_0D)^{D-i_1+1}
>
(1+r_0D)^{D-i_2+1},
\]
and therefore
\[
C_{i_1}>C_{i_2}.
\]
Hence
\[
C_1>C_2>\cdots>C_{D+1},
\]
so human managerial cost increases as one moves upward in the hierarchy.

By Theorem~\ref{thm:Human--AI Substitution Principle}, shifting task weight from human agent \((i,j)\) to AI agent \((i,k)\) changes total risk-adjusted cost by
\[
\Delta\mathcal{C}
=
\delta(\widetilde C'_{ik}-\widetilde C_{ij}).
\]
Equivalently, the cost reduction from substitution is
\[
-\Delta\mathcal{C}
=
\delta(\widetilde C_{ij}-\widetilde C'_{ik}).
\]
Thus, holding AI cost and the risk terms \(R_{ij}\) and \(R'_{ik}\) fixed, the cost-based incentive for substitution increases with the human risk-adjusted cost \(\widetilde C_{ij}=C_{ij}+\lambda R_{ij}\). Since the human cost component \(C_i\) increases as \(i\to 1\) while \(R_{ij}\) is held fixed, the cost-based substitution pressure is stronger at higher hierarchical levels.

We now account for the level dependence of AI managerial costs. Let
\[
a=1+r_0D,
\qquad
b=1+r'_0D.
\]
Under Assumption~\ref{ass:asymmetry}(iii), \(r'_0\le r_0\), so
\[
b\le a.
\]
The human and AI managerial cost multipliers are therefore
\[
a^{D-i+1}
\qquad\text{and}\qquad
b^{D-i+1},
\]
respectively, with the AI multiplier no larger than the human multiplier at each level. Thus AI managerial costs escalate no faster than human managerial costs as one moves upward in the hierarchy.

Now let $n=D-i+1$, so that $n\in\{1,\dots,D\}$ at managerial levels and $n=0$ at the line-worker level, and define the nominal gap
\[
g(n)=C_i-C'_i=C_{D+1}\,a^{n}-C'_{D+1}\,b^{n}.
\]
Because $r'_0\ge 0$, we have $1\le b\le a$. Suppose that the AI line-worker cost does not exceed the human line-worker cost, $C'_{D+1}\le C_{D+1}$. Then
\[
g(n+1)-g(n)
=C_{D+1}a^{n}(a-1)-C'_{D+1}b^{n}(b-1)
\;\ge\;C'_{D+1}\bigl[a^{n}(a-1)-b^{n}(b-1)\bigr]
\;\ge\;0,
\]
because $a^{n}\ge b^{n}>0$ and $a-1\ge b-1\ge 0$. Hence the nominal gap $C_i-C'_i$ does not narrow as one moves up the hierarchy, and it widens if either $C'_{D+1}<C_{D+1}$ or $r'_0<r_0$, since the corresponding inequality in the display is then strict. Risk terms enter additively through
\[
\widetilde C_{ij}-\widetilde C'_{ik}
=
(C_{ij}-C'_{ik})+\lambda(R_{ij}-R'_{ik}).
\]
These risk terms may attenuate or reverse the nominal cost advantage in particular roles. Therefore, the theorem establishes a level-dependent cost-based substitution gradient, not an unconditional prediction that substitution must occur at higher levels.
\end{proof}


\begin{styledcorollary}{Middle-Management Vulnerability}{app:Middle-Management Vulnerability}
Suppose \(D\ge 2\). Under the HAT cost specification, if the upward increase in cost-based substitution pressure is moderated by substitution-feasibility constraints that are strongest near the top of the hierarchy, then the likelihood of AI substitution need not be maximized at either the executive level or the line-worker level. Instead, under a single-crossing relation between cost-incentive decay and feasibility gains, there exists an intermediate level
\[
i^*\in\{2,\ldots,D\}
\]
at which substitution likelihood is maximized.
\end{styledcorollary}

\begin{proof}
As per our convention (\S\ref{sec:Organizational Structure}), hierarchical levels are indexed by \(i\in\{1,\ldots,D+1\}\), where \(i=1\) denotes the top of the hierarchy and \(i=D+1\) denotes the line-worker level. Let substitution likelihood at level \(i\) be represented by
\[
L(i)=S(i)F(i),
\]
where \(S(i)>0\) denotes the cost-based substitution incentive and \(F(i)>0\) denotes the substitution-feasibility factor.

The cost-based incentive \(S(i)\) captures the economic pressure to replace humans with AI at level \(i\). Under the HAT cost specification, human hierarchical cost satisfies
\[
C_i=C_{D+1}(1+r_0D)^{D-i+1}.
\]
Thus, as one moves upward in the hierarchy, the human cost component increases geometrically. Consequently, when AI risk-adjusted cost does not increase at the same rate, the cost-based substitution incentive tends to be stronger at higher hierarchical positions.

The feasibility factor \(F(i)\) captures constraints that limit actual substitution, including accountability, compliance exposure, reputational risk, and strategic-role requirements. These constraints are typically strongest near the top of the hierarchy, so feasibility may increase as one moves downward from executive levels toward middle levels.

To characterize when the product \(L(i)=S(i)F(i)\) is maximized at an intermediate level, define the adjacent ratios
\[
\alpha_i=\frac{S(i)}{S(i+1)},
\qquad
\beta_i=\frac{F(i+1)}{F(i)},
\qquad i=1,\ldots,D.
\]
Here, \(\alpha_i\) measures the rate at which the cost-based substitution incentive decreases when moving one level downward, while \(\beta_i\) measures the rate at which substitution feasibility increases when moving one level downward. Then
\[
\frac{L(i+1)}{L(i)}
=
\frac{S(i+1)F(i+1)}{S(i)F(i)}
=
\frac{S(i+1)}{S(i)}
\frac{F(i+1)}{F(i)}
=
\frac{\beta_i}{\alpha_i}.
\]
Therefore,
\[
L(i+1)>L(i)
\quad \Longleftrightarrow \quad
\beta_i>\alpha_i,
\]
and
\[
L(i+1)<L(i)
\quad \Longleftrightarrow \quad
\beta_i<\alpha_i.
\]

Define
\[
\gamma_i=\frac{\beta_i}{\alpha_i},
\qquad i=1,\ldots,D.
\]
A single-crossing relation between cost-incentive decay and feasibility gains means that the sequence
\(\gamma_i\) crosses the threshold \(1\) exactly once: there exists an index at which
\(\gamma_i\) changes from greater than \(1\) to less than \(1\). Equivalently, assume
\[
\gamma_i>1 \quad \text{for all } i<i^*,
\qquad
\gamma_i<1 \quad \text{for all } i\ge i^*,
\]
for some \(i^*\in\{2,\ldots,D\}\). Such an intermediate index exists whenever feasibility initially grows faster than cost-based substitution pressure decays, but eventually grows more slowly:
\[
\gamma_1>1
\qquad\text{and}\qquad
\gamma_D<1,
\]
with a single crossing between these regimes.

Since
\[
\frac{L(i+1)}{L(i)}=\gamma_i,
\]
we have \(L(i+1)>L(i)\) for all \(i<i^*\) and \(L(i+1)<L(i)\) for all \(i\ge i^*\). Therefore \(L(i)\) increases up to level \(i^*\) and decreases after level \(i^*\). Hence \(L(i)\) is maximized at the intermediate hierarchical level \(i^*\). With strict inequalities, this maximizer is unique.
\end{proof}

\begin{styledtheorem}{Skill-Adjusted Substitution Condition}{app:Skill-Adjusted Substitution Condition}
Using the line-worker cost structures~\eqref{eq:line_worker_cost} and~\eqref{eq:AI_lineworker_parametric}, transferring task allocation from a human line worker $(D+1,j)$ to an AI agent $(D+1,k)$ strictly reduces total cost if and only if
\[
\textup{Min\$}' + \tfrac{1}{2}\Delta\$' \cdot D
\cdot \mu'_{D+1,k} + \lambda R'_{D+1,k}
\;<\;
\textup{Min\$} + \tfrac{1}{2}\Delta\$ \cdot D
\cdot \mu_{D+1,j} + \lambda R_{D+1,j}.
\]
Under Assumption~\ref{ass:asymmetry}(i), $\Delta\$' \le \Delta\$$, so AI capability costs grow no faster than human skill costs.
\end{styledtheorem}

\begin{proof}
From~\eqref{eq:line_worker_cost},
\eqref{eq:AI_lineworker_parametric},
\eqref{eq:human_effective_cost}, and
\eqref{eq:AI_effective_cost}, the risk-adjusted effective cost of human line worker \((D+1,j)\) is
\[
\widetilde{C}_{D+1,j}
=
\textup{Min\$}
+
\tfrac{1}{2}\Delta\$ \cdot D \cdot \mu_{D+1,j}
+
\lambda R_{D+1,j},
\]
and the risk-adjusted effective cost of AI line worker \((D+1,k)\) is
\[
\widetilde{C}'_{D+1,k}
=
\textup{Min\$}'
+
\tfrac{1}{2}\Delta\$' \cdot D \cdot \mu'_{D+1,k}
+
\lambda R'_{D+1,k}.
\]

Consider a feasible marginal substitution from the human line worker to the AI line worker:
\[
q_{D+1,j}\to q_{D+1,j}-\delta,
\qquad
q'_{D+1,k}\to q'_{D+1,k}+\delta,
\]
where \(0<\delta\le q_{D+1,j}\). This perturbation preserves nonnegativity and the unit-sum constraint.

By linearity of \(\mathcal C\),
\[
\Delta\mathcal{C}
=
\delta
\left(
\widetilde{C}'_{D+1,k}
-
\widetilde{C}_{D+1,j}
\right).
\]
Since \(\delta>0\), substitution reduces total risk-adjusted cost if and only if
\[
\widetilde{C}'_{D+1,k}
<
\widetilde{C}_{D+1,j}.
\]
Substituting the expressions above gives
\[
\textup{Min\$}' + \tfrac{1}{2}\Delta\$' \cdot D
\cdot \mu'_{D+1,k} + \lambda R'_{D+1,k}
\;<\;
\textup{Min\$} + \tfrac{1}{2}\Delta\$ \cdot D
\cdot \mu_{D+1,j} + \lambda R_{D+1,j},
\]
which is the stated condition.

Under Assumption~\ref{ass:asymmetry}(i),
\[
\Delta\$'\le \Delta\$,
\]
so the AI capability-cost term grows no faster than the human skill-cost term at the same organizational depth \(D\). The substitution condition therefore compares two structurally distinct cost functions on risk-adjusted grounds rather than merely restating a definitional inequality.
\end{proof}

\begin{styledcorollary}{Skill-Dependent Human Protection and Vulnerability Zone}{app:Skill-Dependent Human Protection and Vulnerability Zone}
Fix an AI agent $(D+1,k)$ with risk-adjusted effective cost
\[
\tilde{C}'_{D+1,k}
= \textup{Min\$}' + \tfrac{1}{2}\Delta\$' \cdot D
\cdot \mu'_{D+1,k} + \lambda R'_{D+1,k}.
\]
Holding \(R_{D+1,j}\), \(R'_{D+1,k}\), and the AI capability level \(\mu'_{D+1,k}\) fixed, suppose the human line-worker risk-adjusted cost
\[
\tilde{C}_{D+1,j}
= \textup{Min\$} + \tfrac{1}{2}\Delta\$ \cdot D
\cdot \mu_{D+1,j} + \lambda R_{D+1,j}
\]
is strictly increasing in \(\mu_{D+1,j}\), which holds when \(\Delta\$>0\) and \(D\ge 1\). Then there exists a threshold, possibly outside the feasible skill range,
\[
\mu^* =
\frac{2\bigl[\textup{Min\$}' - \textup{Min\$}
+ \lambda(R'_{D+1,k} - R_{D+1,j})\bigr]}
{\Delta\$ \cdot D}
+
\frac{\Delta\$'}{\Delta\$}\cdot \mu'_{D+1,k}
\]
such that:
\[
\mu_{D+1,j} < \mu^*
\;\Longrightarrow\;
\tilde{C}_{D+1,j} < \tilde{C}'_{D+1,k}
\qquad(\text{protection zone}),
\]
\[
\mu_{D+1,j} > \mu^*
\;\Longrightarrow\;
\tilde{C}_{D+1,j} > \tilde{C}'_{D+1,k}
\qquad(\text{vulnerability zone}).
\]
Equivalently, human vulnerability to substitution depends on how rapidly risk-adjusted human compensation grows relative to risk-adjusted AI capability cost as skill increases.
\end{styledcorollary}

\begin{proof}
Holding \(R_{D+1,j}\), \(R'_{D+1,k}\), and \(\mu'_{D+1,k}\) fixed, the human line-worker risk-adjusted cost is affine in \(\mu_{D+1,j}\):
\[
\widetilde C_{D+1,j}
=
\textup{Min\$}
+
\tfrac{1}{2}\Delta\$ \cdot D \cdot \mu_{D+1,j}
+
\lambda R_{D+1,j}.
\]
Since \(\Delta\$>0\) and \(D\ge 1\), this expression is strictly increasing in \(\mu_{D+1,j}\).

The boundary between the human and AI agent occurs when
\[
\widetilde C_{D+1,j}
=
\widetilde C'_{D+1,k}.
\]
Substituting the corresponding risk-adjusted costs gives
\[
\textup{Min\$}
+
\tfrac{1}{2}\Delta\$ \cdot D \cdot \mu_{D+1,j}
+
\lambda R_{D+1,j}
=
\textup{Min\$}'
+
\tfrac{1}{2}\Delta\$' \cdot D \cdot \mu'_{D+1,k}
+
\lambda R'_{D+1,k}.
\]
Solving for \(\mu_{D+1,j}\) yields
\[
\mu^*
=
\frac{2\bigl[\textup{Min\$}'-\textup{Min\$}
+\lambda(R'_{D+1,k}-R_{D+1,j})\bigr]}
{\Delta\$ \cdot D}
+
\frac{\Delta\$'}{\Delta\$}\cdot \mu'_{D+1,k}.
\]
This threshold may lie outside the feasible skill range, but it still defines the cost-ordering boundary implied by the affine cost functions.

Because \(\widetilde C_{D+1,j}\) is strictly increasing in \(\mu_{D+1,j}\), for any feasible \(\mu_{D+1,j}\),
\[
\mu_{D+1,j}<\mu^*
\quad\Longrightarrow\quad
\widetilde C_{D+1,j}<\widetilde C'_{D+1,k},
\]
and
\[
\mu_{D+1,j}>\mu^*
\quad\Longrightarrow\quad
\widetilde C_{D+1,j}>\widetilde C'_{D+1,k}.
\]
Thus, below the threshold the human line worker is risk-adjusted-cost-efficient relative to the fixed AI agent, while above the threshold the AI agent is risk-adjusted-cost-efficient relative to the human line worker.

Under Assumption~\ref{ass:asymmetry}(i),
\[
\Delta\$'\le \Delta\$,
\]
so the AI capability-cost term grows no faster than the human skill-cost term. Risk differentials shift the threshold: for \(\lambda>0\), a larger AI risk term \(R'_{D+1,k}\) raises \(\mu^*\), expanding the human protection zone, while a larger human risk term \(R_{D+1,j}\) lowers \(\mu^*\), expanding the vulnerability zone. Therefore, human vulnerability to substitution depends jointly on human skill costs, AI capability costs, and the relative risk adjustment of human and AI agents.
\end{proof}


\begin{styledtheorem}{Risk-Adjusted Substitution Principle}{app:Risk-Adjusted Substitution Principle}
Since risk is embedded in individual agent cost primitives via~\eqref{eq:human_effective_cost}
and~\eqref{eq:AI_effective_cost}, the substitution condition from Theorem~\ref{thm:Human--AI Substitution Principle} already operates on risk-adjusted grounds:
\[
\tilde{C}'_{ik} < \tilde{C}_{ij}
\;\iff\;
C'_{ik} + \lambda R'_{ik} < C_{ij} + \lambda R_{ij}.
\]
Expanding $R'_{ik}$ via~\eqref{eq:AI_risk}, the AI risk-adjusted cost is
\[
\tilde{C}'_{ik}
= C'_{ik}
+ \lambda\!\left(\omega_1 R^{\mathrm{(rel)}}_{ik}
+ \omega_2 R^{\mathrm{(comp)}}_{ik}
+ \omega_3 R^{\mathrm{(rep)}}_{ik}\right).
\]
Even when $C'_{ik} < C_{ij}$, so AI is cheaper on a nominal basis, substitution fails to reduce risk-adjusted cost whenever
\[
\lambda\bigl(\omega_1 R^{\mathrm{(rel)}}_{ik}
+ \omega_2 R^{\mathrm{(comp)}}_{ik}
+ \omega_3 R^{\mathrm{(rep)}}_{ik}
- R_{ij}\bigr)
\;\ge\; C_{ij} - C'_{ik}.
\]
\end{styledtheorem}

\begin{proof}
By Theorem~\ref{thm:Human--AI Substitution Principle}, substituting human agent \((i,j)\) with AI agent \((i,k)\) reduces total risk-adjusted cost if and only if
\[
\widetilde C'_{ik}<\widetilde C_{ij}.
\]
Using definitions~\eqref{eq:human_effective_cost} and~\eqref{eq:AI_effective_cost}, we have
\[
\widetilde C'_{ik}
=
C'_{ik}+\lambda R'_{ik},
\qquad
\widetilde C_{ij}
=
C_{ij}+\lambda R_{ij}.
\]
Therefore,
\[
\widetilde C'_{ik}<\widetilde C_{ij}
\quad\Longleftrightarrow\quad
C'_{ik}+\lambda R'_{ik}
<
C_{ij}+\lambda R_{ij}.
\]
Equivalently,
\[
C_{ij}-C'_{ik}
>
\lambda(R'_{ik}-R_{ij}).
\]
Thus, even if AI is cheaper on a nominal basis, so that
\[
C'_{ik}<C_{ij},
\]
substitution fails to reduce risk-adjusted cost whenever
\[
\lambda(R'_{ik}-R_{ij})
\ge
C_{ij}-C'_{ik}.
\]

Expanding the AI risk term using~\eqref{eq:AI_risk},
\[
R'_{ik}
=
\omega_1 R^{\mathrm{(rel)}}_{ik}
+
\omega_2 R^{\mathrm{(comp)}}_{ik}
+
\omega_3 R^{\mathrm{(rep)}}_{ik}.
\]
Substituting this expression into the preceding inequality gives
\[
\lambda\bigl(\omega_1 R^{\mathrm{(rel)}}_{ik}
+ \omega_2 R^{\mathrm{(comp)}}_{ik}
+ \omega_3 R^{\mathrm{(rep)}}_{ik}
- R_{ij}\bigr)
\ge
C_{ij}-C'_{ik}.
\]
Hence a sufficiently large reliability, compliance, or reputational risk penalty can offset AI's nominal cost advantage and prevent substitution on risk-adjusted grounds.
\end{proof}

\begin{styledtheorem}{Constrained Hybrid Optima}{app:Constrained Interior Optima}
Let \(\mathcal{C}\) be defined over the simplex \(\mathcal{Q}\) as in Theorem~\ref{thm:Global Optimal Allocation}. Suppose additional linear constraints are imposed, such as minimum human allocation, regulatory bounds, capacity constraints, or skill requirements, so that the feasible set becomes
\[
\mathcal{Q}_c
=
\{(q,q')\in\mathcal{Q}: A(q,q')\le b,\; E(q,q')=d\},
\]
where \(\mathcal{Q}_c\) is nonempty. Then the constrained problem
\[
\min_{(q,q')\in\mathcal{Q}_c}\mathcal{C}(q,q')
\]
has a global solution, and at least one global minimizer is an extreme point of \(\mathcal{Q}_c\). However, an extreme point of \(\mathcal{Q}_c\) need not be an extreme point of the original simplex \(\mathcal{Q}\). Consequently, when the additional constraints exclude the unconstrained minimum-cost vertex or impose positive lower or upper bounds across agent types, an optimal constrained allocation
may assign positive weight to multiple agents, including both human and AI agents.

\noindent\textbf{Remark.}
Under the AI risk decomposition \eqref{eq:AI_risk}, a large compliance, reliability, or reputational risk component can eliminate a nominal AI cost advantage:
\[
\tilde{C}'_{ik}=C'_{ik}+\lambda R'_{ik}
\ge
\tilde{C}_{ij}=C_{ij}+\lambda R_{ij}
\]
even when \(C'_{ik}<C_{ij}\). In the unconstrained single-task problem, this risk adjustment can change which vertex is optimal when it affects the global minimum risk-adjusted cost; by itself it does not generally create an interior mixture unless there is a tie. In constrained or heterogeneous multi-task settings, however, risk-adjusted cost heterogeneity can sustain hybrid human--AI organizational structures without imposing an exogenous minimum-human-fraction constraint.
\end{styledtheorem}

\begin{proof}
The original feasible set \(\mathcal{Q}\) is a simplex and hence a compact convex polytope. The additional linear constraints define a closed convex polyhedron. Since
\[
\mathcal{Q}_c
=
\{(q,q')\in\mathcal{Q}: A(q,q')\le b,\; E(q,q')=d\}
\]
is the intersection of a compact polytope with closed linear equality and inequality constraints, \(\mathcal{Q}_c\) is also a compact convex polytope, provided it is nonempty.

The objective \(\mathcal{C}(q,q')\) is linear, hence continuous. Therefore, \(\mathcal{C}\) attains a global minimum over the compact feasible set \(\mathcal{Q}_c\). By the fundamental theorem of linear optimization, there exists at least one optimal solution at an extreme point of \(\mathcal{Q}_c\).

However, the extreme points of \(\mathcal{Q}_c\) need not coincide with the extreme points of \(\mathcal{Q}\). The vertices of \(\mathcal{Q}\) correspond to allocations placing all task weight on a single human or AI agent. Additional constraints can exclude these vertices or make them infeasible. In that case, the optimal extreme points of \(\mathcal{Q}_c\) arise from intersections of the simplex with active constraint hyperplanes. Such points may have multiple strictly positive allocation weights.

For example, if a constraint requires at least a fraction \(\gamma\in(0,1)\) of the task to be assigned to human agents, then an otherwise AI-dominant solution must satisfy
\[
\sum_{i,j}q_{ij}\ge \gamma.
\]
If this constraint binds at the optimum, then the allocation necessarily assigns positive weight to at least one human agent and positive weight to at least one AI agent. Thus the constrained optimum can be hybrid, even though the unconstrained linear objective admits an extreme-point solution over the original simplex.

Finally, consider the risk-adjustment mechanism. By \eqref{eq:AI_risk},
\[
R'_{ik}
=
\omega_1 R^{(\mathrm{rel})}_{ik}
+
\omega_2 R^{(\mathrm{comp})}_{ik}
+
\omega_3 R^{(\mathrm{rep})}_{ik}.
\]
If
\[
\lambda R'_{ik}
\ge
C_{ij}-C'_{ik}+\lambda R_{ij},
\]
then
\[
\tilde C'_{ik}
=
C'_{ik}+\lambda R'_{ik}
\ge
C_{ij}+\lambda R_{ij}
=
\tilde C_{ij},
\]
even though \(C'_{ik}<C_{ij}\). Thus risk adjustment can reverse the nominal cost ordering and prevent AI substitution for that role. In the unconstrained single-task problem, such a change in relative cost rankings can change the identity of an optimal vertex when it affects the global minimum risk-adjusted cost. In constrained or heterogeneous multi-task settings, such role-specific reversals can cause some tasks or positions to remain human while others are assigned to AI, thereby providing
a risk-based mechanism for hybrid human--AI organizational structures.
\end{proof}

\begin{styledcorollary}{Minimum Human Fraction Constraint}{app:Minimum Human Fraction Constraint}
Suppose the only additional constraint imposed on the simplex \(\mathcal{Q}\) is $\sum_{(i,j)\in\mathcal H} q_{ij}\ge \gamma$, $\gamma\in(0,1]$, where \(\mathcal H\) indexes all human agents. If every human agent has risk-adjusted cost strictly greater than at least one feasible AI agent, then every optimal allocation assigns human weight only up to the required boundary:
\[
\sum_{(i,j)\in\mathcal H} q_{ij}=\gamma.
\]
If at least one human agent and one AI agent are tied at the global minimum risk-adjusted cost, then there exists an optimal allocation satisfying the boundary condition and, when $\gamma<1$, allocations with a larger human share may also be optimal.
\end{styledcorollary}

\begin{proof}
Minimize
\[
\mathcal{C}
=
\sum_{i,j}\tilde{C}_{ij}q_{ij}
+
\sum_{i,k}\tilde{C}'_{ik}q'_{ik}
\]
over the constrained feasible set
\[
\mathcal{Q}_c
=
\left\{(q,q')\in\mathcal{Q}:
\sum_{(i,j)\in\mathcal H} q_{ij}\ge \gamma
\right\}.
\]
The set \(\mathcal{Q}_c\) is a nonempty compact convex polytope whenever the constraint is feasible, and \(\mathcal{C}\) is linear. Hence a global optimum exists, and at least one optimum is attained at an extreme point of \(\mathcal{Q}_c\).

Suppose first that every human agent has risk-adjusted cost strictly greater than at least one feasible AI agent. Let \((q,q')\in\mathcal{Q}_c\) be an optimal allocation, and suppose for contradiction that
\[
\sum_{(i,j)\in\mathcal H} q_{ij}>\gamma.
\]
Then there exists at least one human agent \((i,j)\in\mathcal H\) with \(q_{ij}>0\). By assumption, there exists a feasible AI agent \((i',k)\) such that
\[
\tilde C'_{i'k}<\tilde C_{ij}.
\]
Choose
\[
0<\delta\le 
\min\left\{
q_{ij},
\sum_{(a,b)\in\mathcal H}q_{ab}-\gamma
\right\}.
\]
Define a perturbed allocation by
\[
q_{ij}\to q_{ij}-\delta,
\qquad
q'_{i'k}\to q'_{i'k}+\delta,
\]
with all other allocation weights unchanged. This perturbation preserves nonnegativity, preserves the unit-sum constraint, and keeps the human-fraction constraint feasible because the total human weight remains at least \(\gamma\).

The change in objective value is
\[
\Delta\mathcal{C}
=
\delta(\tilde C'_{i'k}-\tilde C_{ij})<0,
\]
which contradicts the optimality of \((q,q')\). Therefore every optimal allocation must satisfy
\[
\sum_{(i,j)\in\mathcal H}q_{ij}=\gamma.
\]

Now consider the tie case. Suppose a human agent $(i,j)$ and a feasible AI agent $(i',k)$ both attain the global minimum risk-adjusted cost $m$ over all human and AI agents. Every feasible allocation has $\mathcal{C}\ge m$, with equality if and only if all weight is placed on agents of cost $m$. The allocation with weight $\gamma$ on $(i,j)$ and $1-\gamma$ on $(i',k)$ is feasible, satisfies $\sum_{(i,j)\in\mathcal H}q_{ij}=\gamma$ and has $\mathcal{C}=m$, so it is optimal. Because $(i,j)$ and $(i',k)$ have the same cost, shifting weight from $(i',k)$ to $(i,j)$ leaves $\mathcal{C}$ unchanged, so, when $\gamma<1$, allocations with human weight above $\gamma$, up to all weight on $(i,j)$, are also optimal. Thus, in the tie case, there exists an optimal allocation on the boundary and, when $\gamma<1$, allocations with a larger human share may also be optimal.
\end{proof}

\begin{styledtheorem}{Condition for AI-Dominant Extreme-Point Allocation}{app:AI-Dominant Extreme-Point Allocation}
In the unconstrained HAT allocation problem for a given task \(\tau\in\mathcal{T}\), let
\[
m_H(\tau)=\min_{i,j}\widetilde C_{ij}(\tau),
\qquad
m_A(\tau)=\min_{i,k}\widetilde C'_{ik}(\tau).
\]
Then an AI-only extreme-point allocation is optimal if and only if
\[
m_A(\tau)\le m_H(\tau).
\]
Moreover, the optimum is AI-dominant, in the sense that no human agent is optimal, if and only if 
$m_A(\tau)<m_H(\tau)$. If \(m_A(\tau)=m_H(\tau)\), then an AI-only extreme-point allocation is optimal but not unique, since at least one human allocation also attains the same minimum risk-adjusted cost.
\end{styledtheorem}

\begin{proof}
Fix a task \(\tau\in\mathcal{T}\). The total risk-adjusted cost for this task is
\[
\mathcal{C}_\tau(q,q') =
\sum_{i,j} \widetilde C_{ij}(\tau)\, q_{ij}
+
\sum_{i,k} \widetilde C'_{ik}(\tau)\, q'_{ik},
\qquad (q,q')\in\mathcal{Q}.
\]
By Theorem~\ref{thm:Global Optimal Allocation}, the minimum value of
\(\mathcal{C}_\tau\) over \(\mathcal{Q}\) is
\[
m(\tau)=\min\{m_H(\tau),m_A(\tau)\},
\]
where
\[
m_H(\tau)=\min_{i,j}\widetilde C_{ij}(\tau),
\qquad
m_A(\tau)=\min_{i,k}\widetilde C'_{ik}(\tau).
\]

We first prove the condition for the existence of an optimal AI-only extreme-point
allocation. Such an allocation exists if and only if at least one AI agent attains the
global minimum \(m(\tau)\). Since the lowest AI risk-adjusted cost is \(m_A(\tau)\),
this is equivalent to
\[
m_A(\tau)=m(\tau)=\min\{m_H(\tau),m_A(\tau)\},
\]
which holds if and only if
\[
m_A(\tau)\le m_H(\tau).
\]
Therefore, an AI-only extreme-point allocation is optimal if and only if
\(m_A(\tau)\le m_H(\tau)\).

It remains to characterize strict AI dominance. If
\[
m_A(\tau)<m_H(\tau),
\]
then every human agent has risk-adjusted cost strictly greater than the global
minimum \(m_A(\tau)\). Hence no human agent is optimal, and every optimal allocation
must place positive weight only on AI agents attaining \(m_A(\tau)\). Thus the
optimum is AI-dominant.

Conversely, suppose the optimum is AI-dominant in the sense that no human agent is
optimal. If \(m_A(\tau)>m_H(\tau)\), then the minimum-cost human agent would have
strictly lower risk-adjusted cost than every AI agent, so an AI-only allocation could
not be optimal. If \(m_A(\tau)=m_H(\tau)\), then at least one human agent attains the
same minimum risk-adjusted cost as an AI agent and is therefore also optimal,
contradicting AI dominance. Hence AI dominance requires
\[
m_A(\tau)<m_H(\tau).
\]

Finally, if
\[
m_A(\tau)=m_H(\tau),
\]
then an AI-only extreme-point allocation is optimal because an AI agent attains the
global minimum. However, at least one human agent also attains the same minimum
risk-adjusted cost, so the AI-only allocation is not unique. By
Theorem~\ref{thm:Global Optimal Allocation}, any convex combination of the tied
minimum-cost human and AI extreme points is also optimal.
\end{proof}


\begin{styledtheorem}{Monotonic Automation Path}{app:Monotonic Automation Path}
In the unconstrained allocation problem over \(\mathcal{Q}\), with human risk-adjusted costs held fixed and switching costs negligible, let \(\tilde{C}'_{ik}(t)\) denote the risk-adjusted cost of AI agents at time \(t\), and suppose that \(\tilde{C}'_{ik}(t)\) is non-increasing in \(t\) for every \((i,k)\). Then, once an AI-only extreme-point allocation of \(\mathcal{Q}\) is optimal, there continues to exist an optimal AI-only extreme-point allocation at every subsequent time. Once the cheapest AI agent is strictly less costly than every human agent, every subsequent optimal allocation assigns the task exclusively to AI agents. While the cheapest AI agent and the cheapest human agent have equal cost, several allocations are optimal, and a selected optimal allocation need not evolve monotonically.
\end{styledtheorem}

\begin{proof}
At any fixed time \(t\), the risk-adjusted cost function is
\[
\mathcal{C}(t)
=
\sum_{i,j} \widetilde C_{ij}\,q_{ij}
+
\sum_{i,k} \widetilde C'_{ik}(t)\,q'_{ik},
\qquad (q,q')\in\mathcal{Q}.
\]
Human risk-adjusted costs are held fixed, while AI risk-adjusted costs satisfy
\[
\widetilde C'_{ik}(t_2)\le \widetilde C'_{ik}(t_1)
\qquad
\text{for all }(i,k)\text{ and all }t_2\ge t_1.
\]
Define
\[
m_H=\min_{i,j}\widetilde C_{ij},
\qquad
m_A(t)=\min_{i,k}\widetilde C'_{ik}(t).
\]
Since all human costs are fixed, \(m_H\) is constant. Since each AI cost
\(\widetilde C'_{ik}(t)\) is non-increasing in \(t\), the minimum AI cost
\(m_A(t)\) is also non-increasing in \(t\).

By Theorem~\ref{thm:Global Optimal Allocation}, the minimum value of
\(\mathcal{C}(t)\) over \(\mathcal{Q}\) is
\[
m(t)=\min\{m_H,m_A(t)\}.
\]
Moreover, optimal extreme points are precisely those corresponding to agents whose
risk-adjusted cost equals \(m(t)\).

An AI-only extreme point is optimal at time \(t\) if and only if
\[
m_A(t)\le m_H.
\]
Because \(m_A(t)\) is non-increasing and \(m_H\) is fixed, once this inequality
holds at some time \(t_0\), it continues to hold for all \(t\ge t_0\). Similarly,
strict AI dominance holds when
\[
m_A(t)<m_H.
\]
Once this strict inequality holds, it cannot be reversed as long as human
risk-adjusted costs remain fixed.

Therefore, as AI risk-adjusted costs decline, once an AI-only extreme point is optimal there continues to exist an optimal AI-only extreme point: a human-dominant optimum may become tied with an AI optimum and then become AI-dominant, but once \(m_A(t_0)\le m_H\) holds, \(m_A(t)\le m_H\) for all \(t\ge t_0\), so an AI-only extreme point is optimal at every \(t\ge t_0\). If \(m_A(t_0)<m_H\), then \(m_A(t)<m_H\) for all \(t\ge t_0\); every human agent then has cost strictly above \(m(t)\) and receives zero weight in every optimal allocation, so every subsequent optimal allocation assigns the task exclusively to AI agents. Ties or switches among AI agents may occur, so the particular AI-only extreme point that is optimal can change over time; this does not affect the conclusion.

The conclusion cannot be strengthened to monotonic movement of an arbitrarily selected optimal allocation. For example, suppose a human agent and an AI agent both have constant cost 10, while a second AI agent has cost declining from 15 to 13. The hypotheses of the theorem hold, \(m_A(t)=m_H=10\) throughout, and the set of optimal allocations is constant, so a selection from it may alternate between the human extreme point and the AI extreme point.
\end{proof}

\begin{styledcorollary}{Irreversibility of Automation}{app:Irreversibility of Automation}
Under the conditions of Theorem~\ref{thm:Monotonic Automation Path}, in the unconstrained problem with human risk-adjusted costs fixed and negligible switching costs, suppose
\[
\widetilde C'_{ik}(t_0)<\widetilde C_{ij}.
\]
Then for all \(t\ge t_0\), the optimal allocation does not revert to the human agent \((i,j)\).
\end{styledcorollary}

\begin{proof}
Under the conditions of Theorem~\ref{thm:Monotonic Automation Path}, human risk-adjusted costs are fixed and AI risk-adjusted costs are non-increasing over time. Hence, for all \(t\ge t_0\),
\[
\widetilde C'_{ik}(t)
\le
\widetilde C'_{ik}(t_0)
<
\widetilde C_{ij}.
\]
Thus the AI agent \((i,k)\) remains strictly lower in risk-adjusted cost than the human agent \((i,j)\) for all \(t\ge t_0\).

At each time \(t\), by Theorem~\ref{thm:Global Optimal Allocation}, an unconstrained optimal allocation places positive weight only on agents attaining the minimum risk-adjusted cost. Since \((i,j)\) has risk-adjusted cost strictly greater than that of AI agent \((i,k)\), the human agent \((i,j)\) cannot be a minimum-cost agent at any \(t\ge t_0\). Therefore every optimal allocation satisfies
\[
q_{ij}^*(t)=0
\qquad
\text{for all } t\ge t_0.
\]

Because switching costs are negligible, there is no dynamic friction that would justify retaining or returning to the higher-cost human allocation. Hence, once AI becomes strictly more risk-adjusted-cost-efficient than human agent \((i,j)\), the optimal allocation does not revert to that human agent under the maintained conditions.
\end{proof}

\begin{styledtheorem}{Perturbation Stability}{app:Perturbation Stability}
In the unconstrained allocation problem over \(\mathcal{Q}\), suppose the minimum risk-adjusted cost is attained by a unique agent. Sufficiently small perturbations in risk-adjusted costs do not change the optimal allocation. Along a continuous path of perturbations, the optimal allocation can change only if the cost of the current minimizer ties with that of another agent at some point on the path (a human--AI threshold $\tilde{C}'_{i'k} = \tilde{C}_{ij}$, or a human--human or AI--AI equality).
\end{styledtheorem}

\begin{proof}
By Theorem~\ref{thm:Global Optimal Allocation}, the optimal
allocation is determined entirely by the ordering of
$\{\tilde{C}_{ij}, \tilde{C}'_{ik}\}$. Consider a
perturbation $\tilde{C}_{ij} \to \tilde{C}_{ij}+\epsilon_{ij}$,
$\tilde{C}'_{ik} \to \tilde{C}'_{ik}+\epsilon'_{ik}$
with $|\epsilon_{ij}|$ and $|\epsilon'_{ik}|$ small.

Let $\tilde{c}_a$ denote the risk-adjusted cost of an arbitrary human or AI agent $a$ (that is, $\tilde{C}_{ij}$ or $\tilde{C}'_{ik}$) and $\epsilon_a$ its perturbation. Suppose the minimum of $\tilde{c}_a$ over all agents is attained by a unique agent $a^*$, and let $\delta=\min_{a\neq a^*}(\tilde{c}_a-\tilde{c}_{a^*})>0$. If $|\epsilon_a|<\delta/2$ for all $a$, then $\tilde{c}_{a^*}+\epsilon_{a^*}<\tilde{c}_a+\epsilon_a$ for every $a\neq a^*$, so $a^*$ remains the unique minimizer and the optimal allocation is unchanged.

A change can therefore occur only if the perturbation makes the cost of the current minimizer equal to or larger than that of some other agent $a$; along a continuous path of perturbations this means passing through an equality $\tilde{c}_{a^*}+\epsilon_{a^*}=\tilde{c}_a+\epsilon_a$. Such a threshold may involve a human and an AI agent ($\tilde{C}'_{i'k}=\tilde{C}_{ij}$), two human agents, or two AI agents; an equality between two agents that are not minimizers does not change the allocation. If several agents are already tied at the minimum, arbitrarily small perturbations can change which of them is selected, so stability refers to a unique minimizer. Hence optimal allocations are locally constant in the cost parameters and change only at such cost-equality thresholds.
\end{proof}

\begin{styledcorollary}{Discontinuous Task-Allocation Transitions}{app:Discontinuous Workforce Transitions}
In the unconstrained linear problem, suppose that risk-adjusted costs vary continuously with an external parameter and that the minimum-risk-adjusted-cost agent is unique except at isolated parameter values. Whenever the identity of the unique minimizing agent changes across such a parameter value, the cost-minimizing task allocation jumps discontinuously from one extreme point to another.
\end{styledcorollary}

\begin{proof}
Let $\tilde{C}_{ij}(\theta)$ and $\tilde{C}'_{ik}(\theta)$
vary continuously in an external parameter $\theta \in
\mathbb{R}$ (e.g., time, technology level, or market
conditions). The optimal allocation
$(q^*(\theta), q'^*(\theta)) \in \arg\min_{(q,q')\in
\mathcal{Q}}\mathcal{C}(\theta)$
is determined by the ordering of
$\{\tilde{C}_{ij}(\theta), \tilde{C}'_{ik}(\theta)\}$.

Assume that the minimum-risk-adjusted-cost agent is unique except at isolated values of $\theta$. By Perturbation Stability, the identity of the minimizing agent is constant on any interval of $\theta$ on which the cost of the current minimizer does not equal the cost of another agent. Hence $(q^*(\theta), q'^*(\theta))$ is piecewise constant in $\theta$.

When $\theta$ reaches a value $\theta^*$ at which the cost of the current minimizer equals that of another agent (a human and an AI agent, two human agents, or two AI agents) and the ranking of the two is reversed beyond $\theta^*$, the minimizing agent switches. An equality between non-minimizing agents does not alter the optimal allocation set. An equality at the minimum without a reversal of the cost ranking does not change the unique minimizing agent on either side of the equality, although alternative allocations may be optimal at the equality itself. Since the optimal allocation on either side of $\theta^*$ is the vertex of $\mathcal{Q}$ that puts all weight on the (unique) minimizing agent, and the two vertices differ, the allocation jumps discontinuously from one vertex to another.
\end{proof}


\begin{styledtheorem}{Strategic Substitution Equilibrium}{app:Strategic Substitution Equilibrium}
Fix a human agent \((i,j)\) and an AI agent \((i,k)\) competing for task allocation at organizational level \(i\). The interaction has two stages: effort and allocation. In the effort stage, the organization induces or implements adaptation on both agents. The human adaptation level is \(u_j\in U_j=[0,\bar u_j]\) representing upskilling, task redesign, process improvement, or increased effort associated with the human agent. The AI adaptation level is \(u_k\in U_k=[0,\bar u_k]\) representing capability improvement, fine-tuning, monitoring, reliability enhancement, compliance controls, or other risk-mitigation investments associated with the AI agent. Here \(\bar u_j,\bar u_k>0\). The realized risk-adjusted effective costs are
\[
\widetilde C_{ij}(u_j)
=
\widetilde C^0_{ij}-h_j(u_j),
\qquad
\widetilde C'_{ik}(u_k)
=
\widetilde C^{\prime 0}_{ik}-a_k(u_k),
\]
where \(h_j\) and \(a_k\) are continuous, nondecreasing cost-reduction functions satisfying $h_j(0)=0$, $a_k(0)=0$.

Let the effort costs be \(\phi_j(u_j)\) and \(\psi_k(u_k)\), where \(\phi_j\) and \(\psi_k\) are continuous and convex, with $\phi_j(0)=0$, $\psi_k(0)=0$. Let \(\rho>0\) denote the sensitivity of the effort-stage contest to realized risk-adjusted cost differences. Define
\[
p_j(u_j,u_k)
=
\frac{\exp[-\rho \widetilde C_{ij}(u_j)]}
{\exp[-\rho \widetilde C_{ij}(u_j)]
+
\exp[-\rho \widetilde C'_{ik}(u_k)]},
\qquad
p'_k(u_j,u_k)=1-p_j(u_j,u_k).
\]
For finite \(\rho\), \(p_j\) and \(p'_k\) are reduced-form contest-success measures that determine effort-stage incentives, rather than literal probabilities generated by the deterministic allocation rule of the subsequent stage; the effort-stage game thus represents a smooth incentive mechanism associated with anticipated substitution pressure. Let \(B_j\ge 0\) and \(B'_k\ge 0\) denote the benefits of being selected for the human and AI sides, respectively. The effort-stage payoffs are
\[
\Pi_j(u_j,u_k)
=
B_j p_j(u_j,u_k)-\phi_j(u_j),
\]
\[
\Pi'_k(u_j,u_k)
=
B'_k p'_k(u_j,u_k)-\psi_k(u_k).
\]

Assume that \(\Pi_j(\cdot,u_k)\) and \(\Pi'_k(u_j,\cdot)\) are quasiconcave in their own effort choices for every fixed rival effort. Then the effort game admits a Nash equilibrium
\[
u^*=(u_j^*,u_k^*).
\]

In the allocation stage, the organization applies the HAT allocation rule to the realized equilibrium costs $\widetilde C_{ij}(u_j^*)$ and $\widetilde C'_{ik}(u_k^*)$. Thus, in our two-agent allocation subproblem,
\[
q_{ij}^*=1,\quad q_{ik}^{\prime *}=0
\quad
\text{if}
\quad
\widetilde C_{ij}(u_j^*)<\widetilde C'_{ik}(u_k^*),
\]
\[
q_{ij}^*=0,\quad q_{ik}^{\prime *}=1
\quad
\text{if}
\quad
\widetilde C'_{ik}(u_k^*)<\widetilde C_{ij}(u_j^*).
\]

If $\widetilde C_{ij}(u_j^*)=\widetilde C'_{ik}(u_k^*)$, then both agents attain the same realized minimum risk-adjusted cost, and any allocation supported on the tied agents is optimal. If \(h_j,a_k,\phi_j,\psi_k\) are differentiable, with one-sided derivatives at the boundaries, then the human equilibrium effort satisfies one of the following necessary marginal conditions:
\[
u_j^*=0
\quad\Longrightarrow\quad
B_j\rho\,p_j(u^*)p'_k(u^*)\dot h_j(0)
\le
\dot\phi_j(0),
\]
\[
u_j^*=\bar u_j
\quad\Longrightarrow\quad
B_j\rho\,p_j(u^*)p'_k(u^*)\dot h_j(\bar u_j)
\ge
\dot\phi_j(\bar u_j),
\]
or
\[
u_j^*\in(0,\bar u_j)
\quad\Longrightarrow\quad
B_j\rho\,p_j(u^*)p'_k(u^*)\dot h_j(u_j^*)
=
\dot\phi_j(u_j^*).
\]
Similarly, the AI equilibrium investment satisfies one of
\[
u_k^*=0
\quad\Longrightarrow\quad
B'_k\rho\,p_j(u^*)p'_k(u^*)\dot a_k(0)
\le
\dot\psi_k(0),
\]
\[
u_k^*=\bar u_k
\quad\Longrightarrow\quad
B'_k\rho\,p_j(u^*)p'_k(u^*)\dot a_k(\bar u_k)
\ge
\dot\psi_k(\bar u_k),
\]
or
\[
u_k^*\in(0,\bar u_k)
\quad\Longrightarrow\quad
B'_k\rho\,p_j(u^*)p'_k(u^*)\dot a_k(u_k^*)
=
\dot\psi_k(u_k^*).
\]

Finally, for any fixed effort profile \((u_j,u_k)\),
\[
\lim_{\rho\to\infty}p_j(u_j,u_k)
=
\begin{cases}
1, & \widetilde C_{ij}(u_j)<\widetilde C'_{ik}(u_k),\\
0, & \widetilde C_{ij}(u_j)>\widetilde C'_{ik}(u_k),\\
\frac12, & \widetilde C_{ij}(u_j)=\widetilde C'_{ik}(u_k).
\end{cases}
\]
Thus, as \(\rho\to\infty\), the smooth effort-stage contest converges to the deterministic risk-adjusted cost comparison underlying the Human--AI Substitution Principle.
\end{styledtheorem}

\noindent
\textbf{Remark (on the maintained quasiconcavity assumption).}
The requirement that \(\Pi_j(\cdot,u_k)\) and \(\Pi'_k(u_j,\cdot)\) be quasiconcave in own strategy is a substantive joint restriction on $(\rho,B_j,B'_k,h_j,a_k,\phi_j,\psi_k)$. It is not implied by \(h_j,a_k\) being continuous and nondecreasing and \(\phi_j,\psi_k\) being continuous and convex. For example, let \(h_j(u_j)=u_j\), \(\phi_j(u_j)=\tfrac12 u_j^2\), \(\rho=5\), \(B_j=100\), with \(\widetilde C^0_{ij}=10\) and a fixed opponent cost \(\widetilde C'_{ik}(u_k)=5\). Then
\[
p_j(u_j,u_k)
=
\frac{1}{1+\exp[5(5-u_j)]},
\]
and
\[
\Pi_j(u_j,u_k)
=
100p_j(u_j,u_k)-\tfrac12u_j^2.
\]
This payoff is not quasiconcave on economically plausible effort intervals. For instance,
\(\Pi_j(0,u_k)\approx 0\), \(\Pi_j(4.5,u_k)\approx -2.54\), and \(\Pi_j(4.7,u_k)\approx 7.20\). Hence the upper contour set at threshold \(t=0\) contains \(u_j=0\), excludes points near \(u_j=4.5\), and contains points again near \(u_j=4.7\). The upper contour set is therefore disconnected on any interval \(U_j=[0,\bar u_j]\) with \(\bar u_j\ge 4.7\), so quasiconcavity fails. Thus the maintained quasiconcavity assumption is not innocuous.

The issue becomes more important when \(\rho\) is large. A larger \(\rho\) makes \(p_j\) a steeper sigmoid in the realized cost gap \(\widetilde C'_{ik}(u_k)-\widetilde C_{ij}(u_j)\), so the marginal benefit term \(\rho\,p_jp'_k\,\dot h_j(u_j)\) can become concentrated near the cost-crossing point. If this local spike exceeds marginal effort cost after effort was previously unattractive, the payoff can lose single-peakedness and quasiconcavity can fail.

This does not affect the pointwise limit \(\rho\to\infty\) in the theorem: the contest-success measure still converges to the deterministic risk-adjusted cost comparison underlying the Human--AI Substitution Principle for any fixed effort profile. What it affects is the stronger interpretation that equilibrium effort choices \(u^*(\rho)\) remain well behaved as \(\rho\) grows. That interpretation requires equilibrium existence along the path, which is precisely what the quasiconcavity assumption supplies in the pure-strategy version of the theorem. Existence of an equilibrium for each value of \(\rho\) does not, by itself, establish uniqueness, continuity, or convergence of the equilibrium effort profiles as \(\rho\to\infty\); the limit stated in the theorem concerns the contest-success measure at a fixed effort profile. Alternatively, one can drop quasiconcavity and invoke Glicksberg's (\citeyear{glicksberg1952}) mixed-strategy existence theorem for compact strategy spaces and continuous payoffs, at the cost of replacing a deterministic effort equilibrium with a mixed-strategy equilibrium.

\begin{proof}
\textbf{Existence.}
The strategy sets
\[
U_j=[0,\bar u_j],
\qquad
U_k=[0,\bar u_k]
\]
are nonempty, compact, and convex subsets of \(\mathbb R\). Since \(h_j,a_k,\phi_j,\psi_k\) are continuous, the realized risk-adjusted costs
\[
\widetilde C_{ij}(u_j)
=
\widetilde C^0_{ij}-h_j(u_j),
\qquad
\widetilde C'_{ik}(u_k)
=
\widetilde C^{\prime 0}_{ik}-a_k(u_k)
\]
are continuous in the effort choices. The exponential function is continuous and strictly positive, so
\[
\exp[-\rho \widetilde C_{ij}(u_j)]
+
\exp[-\rho \widetilde C'_{ik}(u_k)]
>
0
\]
for every \((u_j,u_k)\in U_j\times U_k\). Hence \(p_j\) and \(p'_k\) are continuous on \(U_j\times U_k\). Therefore the payoff functions
\[
\Pi_j(u_j,u_k)
=
B_j p_j(u_j,u_k)-\phi_j(u_j),
\]
and
\[
\Pi'_k(u_j,u_k)
=
B'_k p'_k(u_j,u_k)-\psi_k(u_k)
\]
are continuous. By assumption, each payoff is quasiconcave in the player's own strategy. By the Debreu--Glicksberg--Fan pure-strategy equilibrium existence theorem
\citep{debreu1952,glicksberg1952,fan1952}, the effort game admits at least one Nash equilibrium
\[
u^*=(u_j^*,u_k^*).
\]

\textbf{Allocation.}
At equilibrium, the realized risk-adjusted costs, $\widetilde C_{ij}(u_j^*)$ and $\widetilde C'_{ik}(u_k^*)$, are fixed numbers. The organization's two-agent allocation subproblem is therefore the HAT allocation problem evaluated at these realized costs. By Theorem~\ref{thm:Global Optimal Allocation}, the optimal allocation assigns the task to the agent with lower realized risk-adjusted cost. If the two realized costs are equal, both agents attain the same minimum cost, and any allocation supported on the tied agents is optimal.

\textbf{Marginal conditions.}
Assume now that \(h_j,a_k,\phi_j,\psi_k\) are differentiable, with one-sided derivatives at the boundaries. Write
\[
X(u_j)=-\rho\widetilde C_{ij}(u_j),
\qquad
Y(u_k)=-\rho\widetilde C'_{ik}(u_k).
\]
Then
\[
p_j=\frac{e^X}{e^X+e^Y}.
\]
A direct computation gives
\[
\frac{\partial p_j}{\partial X}
=
p_j(1-p_j)
=
p_jp'_k.
\]
Since
\[
\frac{d}{du_j}\widetilde C_{ij}(u_j)
=
-\dot h_j(u_j),
\]
we have
\[
\frac{dX}{du_j}
=
\rho\dot h_j(u_j).
\]
By the chain rule,
\[
\frac{\partial p_j}{\partial u_j}
=
\rho\,p_jp'_k\,\dot h_j(u_j).
\]
Therefore
\[
\frac{\partial \Pi_j}{\partial u_j}
=
B_j\rho\,p_jp'_k\,\dot h_j(u_j)-\dot\phi_j(u_j).
\]

If \(u_j^*\in(0,\bar u_j)\), optimality of \(u_j^*\) against \(u_k^*\) implies
\[
\frac{\partial \Pi_j}{\partial u_j}(u_j^*,u_k^*)=0,
\]
so
\[
B_j\rho\,p_j(u^*)p'_k(u^*)\dot h_j(u_j^*)
=
\dot\phi_j(u_j^*).
\]
If \(u_j^*=0\), a feasible increase in \(u_j\) cannot raise the payoff, so the right derivative must be nonpositive:
\[
B_j\rho\,p_j(u^*)p'_k(u^*)\dot h_j(0)
-
\dot\phi_j(0)
\le 0.
\]
Thus
\[
B_j\rho\,p_j(u^*)p'_k(u^*)\dot h_j(0)
\le
\dot\phi_j(0).
\]
If \(u_j^*=\bar u_j\), a feasible decrease in \(u_j\) cannot raise the payoff. Equivalently, the left derivative at \(\bar u_j\) must be nonnegative:
\[
B_j\rho\,p_j(u^*)p'_k(u^*)\dot h_j(\bar u_j)
-
\dot\phi_j(\bar u_j)
\ge 0.
\]
Thus
\[
B_j\rho\,p_j(u^*)p'_k(u^*)\dot h_j(\bar u_j)
\ge
\dot\phi_j(\bar u_j).
\]

The AI-side conditions are derived identically. Since \(p'_k=1-p_j\),
\[
\frac{\partial p'_k}{\partial u_k}
=
\rho\,p_jp'_k\,\dot a_k(u_k),
\]
and therefore
\[
\frac{\partial \Pi'_k}{\partial u_k}
=
B'_k\rho\,p_jp'_k\,\dot a_k(u_k)-\dot\psi_k(u_k).
\]
Applying the same interior and boundary optimality conditions on \(U_k=[0,\bar u_k]\) gives
\[
u_k^*=0
\quad\Longrightarrow\quad
B'_k\rho\,p_j(u^*)p'_k(u^*)\dot a_k(0)
\le
\dot\psi_k(0),
\]
\[
u_k^*=\bar u_k
\quad\Longrightarrow\quad
B'_k\rho\,p_j(u^*)p'_k(u^*)\dot a_k(\bar u_k)
\ge
\dot\psi_k(\bar u_k),
\]
and
\[
u_k^*\in(0,\bar u_k)
\quad\Longrightarrow\quad
B'_k\rho\,p_j(u^*)p'_k(u^*)\dot a_k(u_k^*)
=
\dot\psi_k(u_k^*).
\]

\textbf{Limit as \(\rho\to\infty\).}
Fix any effort profile \((u_j,u_k)\). Let
\[
\Delta(u_j,u_k)
=
\widetilde C'_{ik}(u_k)-\widetilde C_{ij}(u_j).
\]
Then
\[
p_j(u_j,u_k)
=
\frac{1}
{1+\exp[-\rho\Delta(u_j,u_k)]}.
\]
If \(\Delta(u_j,u_k)>0\), then \(\exp[-\rho\Delta(u_j,u_k)]\to 0\), so \(p_j\to 1\). If \(\Delta(u_j,u_k)<0\), then \(\exp[-\rho\Delta(u_j,u_k)]\to \infty\), so \(p_j\to 0\). If \(\Delta(u_j,u_k)=0\), then \(p_j=1/2\) for all \(\rho\). Hence
\[
\lim_{\rho\to\infty}p_j(u_j,u_k)
=
\begin{cases}
1, & \widetilde C_{ij}(u_j)<\widetilde C'_{ik}(u_k),\\
0, & \widetilde C_{ij}(u_j)>\widetilde C'_{ik}(u_k),\\
\frac12, & \widetilde C_{ij}(u_j)=\widetilde C'_{ik}(u_k).
\end{cases}
\]
Therefore the smooth effort-stage contest converges to the deterministic risk-adjusted cost comparison used in the HAT allocation rule and in the Human--AI Substitution Principle.
\end{proof}

\newpage
\section*{About the Authors}

\begin{minipage}{\textwidth}
\begin{wrapfigure}{l}{0.25\textwidth}
\vspace{-\baselineskip}
\includegraphics[width=\linewidth]{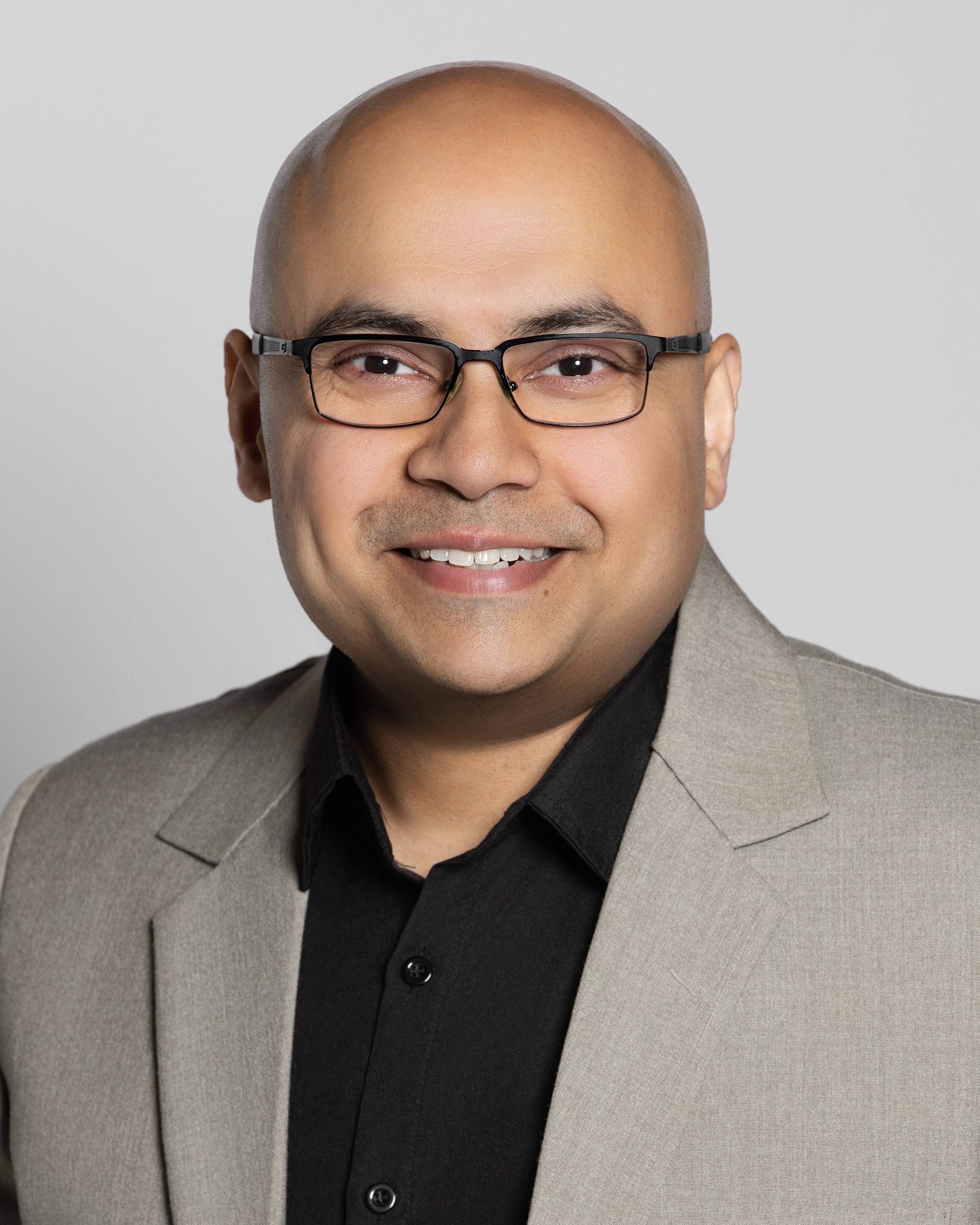} 
\end{wrapfigure}
\textbf{Bonny Banerjee} received M.S. in Electrical Engineering and Ph.D. in Computer Science (major AI, minor Cognitive Science) from the Ohio State University, USA. Just after graduating with Ph.D., he spent 3.5 years (2007-2011, through the Great Recession) leading the research at a startup, which resulted in a number of patents, substantial investor funding, and launch of a commercial product for the end-user which received wide media coverage. The intellectual property was acquired by the leading company in the field. For the last 15 years, he has served as a faculty at institutions of higher education, and as a consultant for small and large businesses in the U.S. and abroad. Dr. Banerjee holds 8 patents and has published over 75 peer-reviewed articles in reputed journals and conference proceedings on AI agents and related areas. He received the Best Paper Award at a NeurIPS 2023 Workshop. He has served as the Principal Investigator of research projects funded by the U.S. National Science Foundation, Department of Homeland Security, Army, St. Jude Children’s Research Hospital, and startup investors. He serves on the editorial boards of \textit{IEEE Transactions on Cybernetics} (IEEE), \textit{Neural Networks} (Elsevier), and \textit{Brain Theory and Intelligence Science} (Springer Nature) journals.\\
\end{minipage}

\noindent
\begin{minipage}{\textwidth}
\begin{wrapfigure}{l}{0.25\textwidth}
\vspace{-\baselineskip}
\includegraphics[width=\linewidth]{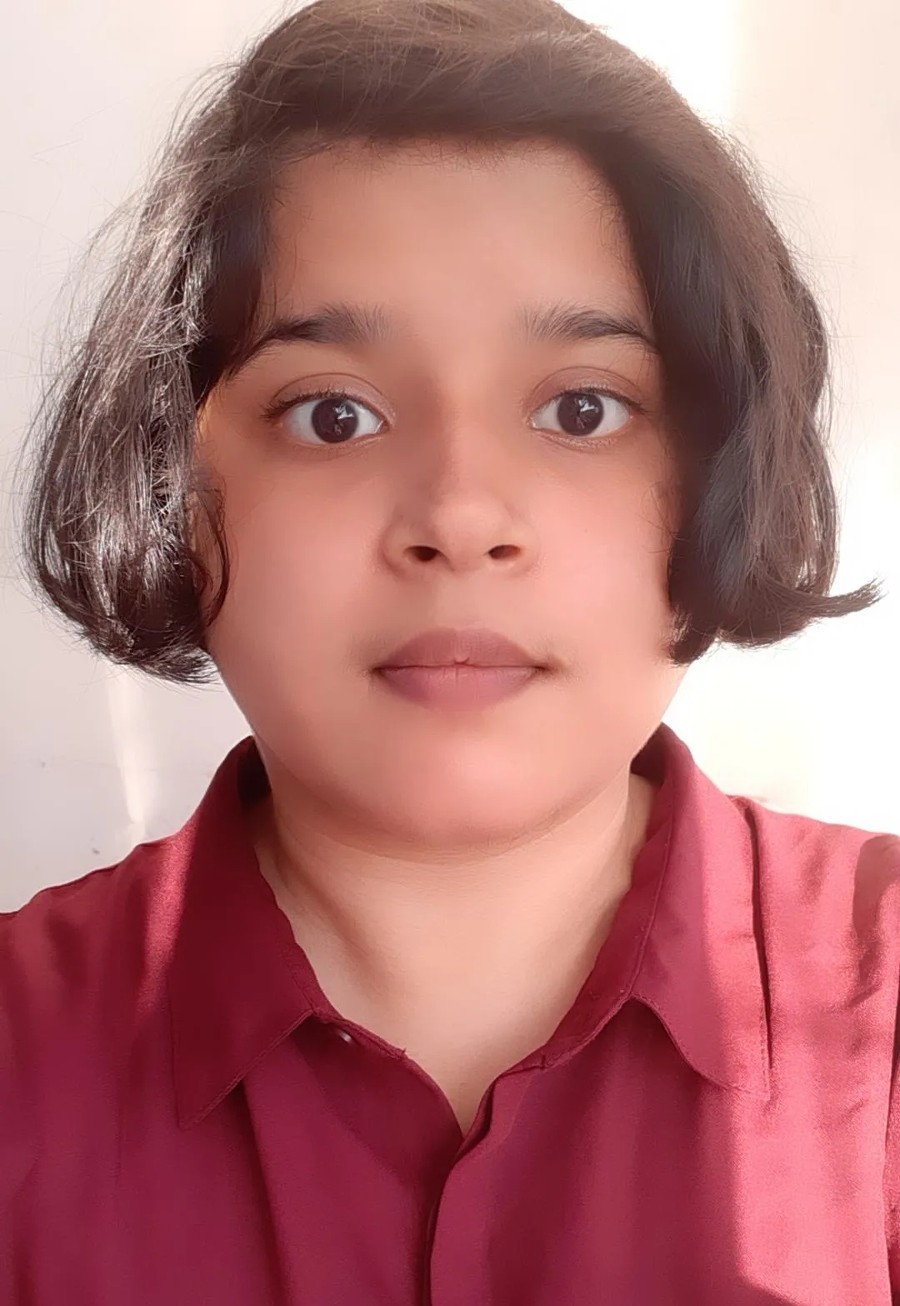}
\end{wrapfigure}
\noindent\textbf{Shreya Singh} earned a Master's in Data Science from Chennai Mathematical Institute, India, and a Master's in Mathematics from Central University of Rajasthan, India. Her research interests include number theory, automatic sequences, stochastic processes, quantitative research, game theory, mathematical modeling, optimization, topological data analysis, and machine learning. She has qualified CSIR NET JRF and GATE Data Science.
\vspace{34mm}
\end{minipage}

\noindent
\begin{minipage}{\textwidth}
\begin{wrapfigure}{l}{0.25\textwidth}
\vspace{-\baselineskip}
\includegraphics[width=\linewidth]{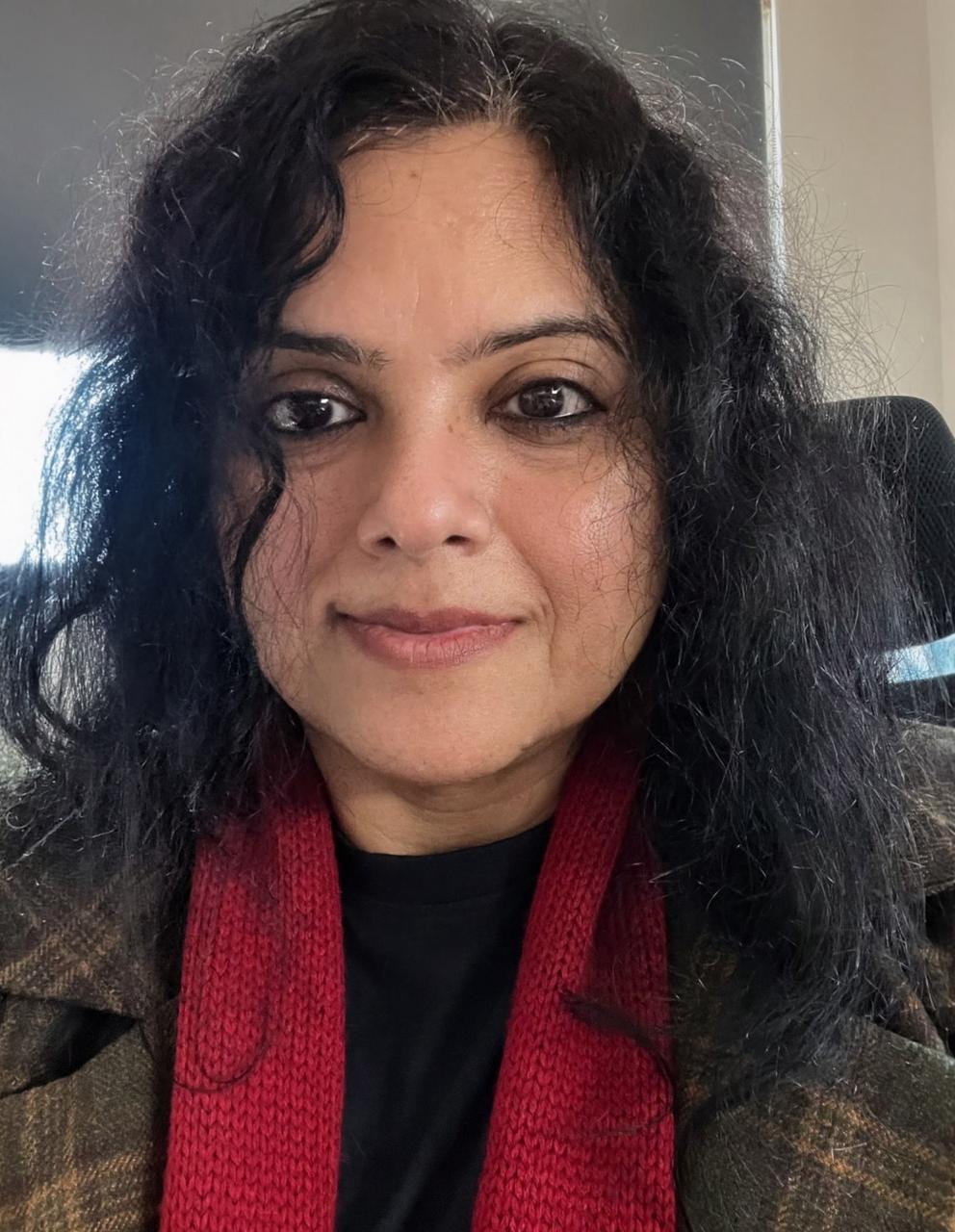} 
\end{wrapfigure}
\noindent\textbf{Simanti Bandyopadhyay} received M.Sc. in Economics from the University of Calcutta, India and Ph.D. in Economics from Jawaharlal Nehru University, New Delhi, India. She is Professor of Economics in the School of Management and Entrepreneurship, Shiv Nadar University, an Institution of Eminence located in Delhi-NCR, India. She was a Senior Economist with National Institute of Public Finance and Policy (an autonomous research institute under the Ministry of Finance, India), New Delhi for 11 years before joining the university. Her professional experience consists of research, consultancy and teaching for over 25 years. She works on development economics, in diverse areas related to development finance and public policy. She has visited many institutions of repute in India and abroad for assignments on research, teaching and research guidance. She received the MacArthur Foundation Award. A regular contributor of research articles and an external reviewer to several international journals on public policy, economics and operations research, Dr. Bandyopadhyay has led important projects and has worked as Advisor/Consultant to the World Bank, Asian Development Bank, United Nations Development Program (UNDP), Planning Commission of India (erstwhile), and multiple Ministries and Commissions in the Government of India and different state governments in India. She has worked extensively on different aspects of Development of South Asian Countries and has conducted international training programs on public finance and public policy for Government officials of South Asian countries.
\end{minipage}

\end{document}